\PassOptionsToPackage{capitalize}{cleveref}
\documentclass{bytedance_seed}

\usepackage{url}
\usepackage{latexsym}
\usepackage[T1]{fontenc}
\usepackage[utf8]{inputenc}
\usepackage{microtype}
\usepackage{graphicx}
\usepackage{array}
\usepackage{booktabs}
\usepackage{multirow}
\usepackage{placeins}
\usepackage{subcaption}
\usepackage{xcolor}
\usepackage{tikz}
\usetikzlibrary{
    arrows.meta,
    calc,
    positioning,
    fit,
    decorations.pathreplacing,
    shapes.geometric
}
\usepackage{etoc}

\usepackage{amsthm}
\usepackage{thmtools}%
\usepackage{thm-restate}
\usepackage{amsmath}
\usepackage{amssymb}
\usepackage{mathtools}
\usepackage{dsfont} %
\usepackage{bm}
\usepackage{mathrsfs}
\usepackage{xspace} %

\usepackage{xparse}
\usepackage{refcount} %

\usepackage[shortlabels,inline]{enumitem}
\setlist{nolistsep}

\usepackage{hyperref}
\usepackage[noabbrev,capitalize]{cleveref} %

\theoremstyle{plain}%
\newtheorem{theorem}{Theorem} %
\newtheorem*{theorem*}{Theorem} %
\newtheorem{lemma}[theorem]{Lemma} %
\newtheorem*{lemma*}{Lemma} %
\newtheorem{proposition}[theorem]{Proposition} %
\newtheorem*{proposition*}{Proposition} %
\theoremstyle{definition} %
\newtheorem{definition}{Definition} %
\newtheorem{assumption}{Assumption} %
\newtheorem{desideratum}{Desideratum} %
\newtheorem{claim}{Claim} %
\newtheorem{manualclaiminner}{Claim}
\NewDocumentEnvironment{manualclaim}{m o}{%
    \renewcommand{\themanualclaiminner}{#1}%
    \IfNoValueTF{#2}
        {\begin{manualclaiminner}}
        {\begin{manualclaiminner}[#2]}%
}{%
    \end{manualclaiminner}%
}
\newtheorem{remark}{Remark} %
\theoremstyle{remark} %

\renewcommand*{\proofname}{\normalfont\bfseries Proof}%

\crefname{equation}{Eq.}{Eqs.}
\crefname{assumption}{Assumption}{Assumptions}
\crefname{desideratum}{Desideratum}{Desiderata}
\crefname{claim}{Claim}{Claims}
\crefname{manualclaiminner}{Claim}{Claims}
\newcommand{\crefrangeconjunction}{--}

\makeatletter
\renewenvironment{proof}[1][\proofname]{\par
  \pushQED{\qed}%
  \normalfont \topsep6\p@\@plus6\p@\relax
  \trivlist
  \item[\hskip\labelsep
        \normalfont\bfseries %
    #1\@addpunct{.}]\ignorespaces
}{
  \popQED\endtrivlist\@endpefalse
}
\makeatother

\crefalias{enumi}{pureitem} %
\crefalias{enumii}{pureitem}
\crefalias{enumiii}{pureitem}
\crefalias{enumiv}{pureitem}
\crefformat{pureitem}
  {#2\textcolor{black}{#1}#3}
\crefrangeformat{pureitem}
  {#3\textcolor{black}{#1}#4\crefrangeconjunction#5\textcolor{black}{#2}#6}
\crefmultiformat{pureitem}
  {#2\textcolor{black}{#1}#3}%
  { and #2\textcolor{black}{#1}#3}%
  {, #2\textcolor{black}{#1}#3}%
  {, and #2\textcolor{black}{#1}#3}
\crefrangemultiformat{pureitem}
  {#3\textcolor{black}{#1}#4\crefrangeconjunction#5\textcolor{black}{#2}#6}%
  { and #3\textcolor{black}{#1}#4\crefrangeconjunction#5\textcolor{black}{#2}#6}%
  {, #3\textcolor{black}{#1}#4\crefrangeconjunction#5\textcolor{black}{#2}#6}%
  {, and #3\textcolor{black}{#1}#4\crefrangeconjunction#5\textcolor{black}{#2}#6}

\usepackage[mathscr]{eucal} %
\usepackage{cancel}
\usepackage{extarrows}

\newcommand{\vx}{\bm{x}}
\newcommand{\vy}{\bm{y}}
\newcommand{\vq}{\bm{q}}

\newcommand{\vv}{\bm{v}}

\newcommand{\vz}{\bm{z}}

\newcommand{\vw}{\bm{w}}
\newcommand{\vu}{\bm{u}}
\newcommand{\vp}{\bm{p}}

\newcommand{\vb}{\bm{b}}

\newcommand{\vh}{\bm{h}}
\newcommand{\vf}{\bm{f}}
\newcommand{\vr}{\bm{r}}
\newcommand{\ve}{\bm{e}}

\newcommand{\vzeta}{\bm{\zeta}}
\newcommand{\vone}{\bm{1}}

\newcommand{\mX}{\bm{X}}
\newcommand{\mU}{\bm{U}}
\newcommand{\mW}{\bm{W}}
\newcommand{\mQ}{\bm{Q}}
\newcommand{\mK}{\bm{K}}
\newcommand{\mV}{\bm{V}}

\newcommand{\mR}{\bm{R}}
\newcommand{\mM}{\bm{M}}
\newcommand{\mI}{\bm{I}}
\newcommand{\mG}{\bm{G}}
\newcommand{\mA}{\bm{A}}
\newcommand{\mB}{\bm{B}}
\newcommand{\mD}{\bm{D}}
\newcommand{\mE}{\bm{E}}
\newcommand{\mH}{\bm{H}}
\newcommand{\mJ}{\bm{J}}
\newcommand{\mP}{\bm{P}}
\newcommand{\mS}{\bm{S}}
\newcommand{\mT}{\bm{T}}
\newcommand{\mY}{\bm{Y}}
\newcommand{\mZ}{\bm{Z}}
\newcommand{\mSigma}{\bm{\Sigma}}

\newcommand{\fC}{\mathcal{C}}
\newcommand{\fH}{\mathcal{H}}
\newcommand{\fR}{\mathcal{R}}
\newcommand{\fF}{\mathcal{F}}

\newcommand{\fN}{\mathcal{N}}
\newcommand{\fK}{\mathcal{K}}
\newcommand{\fP}{\mathcal{P}}

\newcommand{\fE}{\mathcal{E}}

\newcommand{\fD}{\mathcal{D}}
\newcommand{\fV}{\mathcal{V}}

\newcommand{\fI}{\mathcal{I}}

\newcommand{\RR}{\mathbb{R}}

\newcommand{\CC}{\mathbb{C}}
\newcommand{\ii}{\mathrm{i}}

\DeclareMathOperator{\PP}{\mathbb{P}}
\DeclareMathOperator{\EE}{\mathbb{E}}
\DeclareMathOperator{\Var}{\mathrm{Var}}
\DeclareMathOperator{\Cov}{\mathrm{Cov}}

\DeclareMathOperator{\1}{\mathds{1}}

\DeclareMathOperator{\diag}{diag}

\DeclareMathOperator{\Tr}{Tr}
\DeclareMathOperator{\Law}{Law}

\DeclareMathOperator{\sech}{sech}
\DeclareMathOperator{\MS}{MS}
\DeclareMathOperator{\silu}{SiLU}
\DeclareMathOperator{\gelu}{GELU}

\newcommand{\set}[1]{\left\{#1\right\}}
\newcommand{\floor}[1]{\left\lfloor#1\right\rfloor}

\DeclarePairedDelimiter{\norm}{\lVert}{\rVert}
\DeclarePairedDelimiterXPP{\normop}[1]{}{\lVert}{\rVert}{_{\mathrm{op}}}{#1}
\DeclarePairedDelimiterXPP{\normfro}[1]{}{\lVert}{\rVert}{_{\mathrm{F}}}{#1}
\DeclarePairedDelimiterXPP{\normHS}[1]{}{\lVert}{\rVert}{_{\mathrm{HS}}}{#1}
\DeclarePairedDelimiterXPP{\normrms}[1]{}{\lVert}{\rVert}{_{\mathrm{rms}}}{#1}

\newcommand{\transpose}{{\mkern-1.0mu\mathsf{T}}}

\newcommand{\dif}{\mathop{}\!\mathrm{d}}
\newcommand{\given}{\;\middle|\;}
\newcommand{\bigO}{\mathcal{O}}
\newcommand{\littleo}{o}
\newcommand{\bigOmega}{\Omega}
\newcommand{\littleomega}{\omega}
\newcommand{\bigTheta}{\Theta}
\newcommand{\iid}{\mathrel{\overset{\mathrm{iid}}{\sim}}}
\newcommand{\lex}{\mathrm{lex}}

\newcommand{\pto}{\overset{\mathbb{P}}{\longrightarrow}}

\newcommand{\proofparagraph}[1]{\par\medskip\noindent\textbf{#1}\enspace}

\DeclareMathAlphabet{\mathsfit}{T1}{\sfdefault}{\mddefault}{\sldefault}
\SetMathAlphabet{\mathsfit}{bold}{T1}{\sfdefault}{\bfdefault}{\sldefault}
\newcommand{\opfont}[1]{\mathsfit{#1}}
\newcommand{\opA}{\opfont{A}}
\newcommand{\opB}{\opfont{B}}
\newcommand{\opC}{\opfont{C}}
\newcommand{\opD}{\opfont{D}}
\newcommand{\opE}{\opfont{E}}
\newcommand{\opF}{\opfont{F}}
\newcommand{\opG}{\opfont{G}}
\newcommand{\opH}{\opfont{H}}
\newcommand{\opI}{\opfont{I}}
\newcommand{\opJ}{\opfont{J}}
\newcommand{\opK}{\opfont{K}}
\newcommand{\opL}{\opfont{L}}
\newcommand{\opM}{\opfont{M}}
\newcommand{\opP}{\opfont{P}}
\newcommand{\opR}{\opfont{R}}
\newcommand{\opS}{\opfont{S}}
\newcommand{\opT}{\opfont{T}}
\newcommand{\opW}{\opfont{W}}
\newcommand{\opPhi}{\mathsf{\Phi}}
\newcommand{\opPi}{\mathsf{\Pi}}

\newcommand{\cre}[1]{\opL(#1)} %
\newcommand{\ann}[1]{\opL(#1)^*} %
\newcommand{\routine}{\fR} %

\newcommand{\acc}{\nu} %
\newcommand{\upd}{s} %
\newcommand{\resp}{\rho} %
\newcommand{\cyc}{b} %
\newcommand{\accN}{\acc_N}
\newcommand{\updN}{\upd_N}
\newcommand{\respN}{\resp_N}
\newcommand{\accNL}{\acc_{N,L}}
\newcommand{\updNL}{\upd_{N,L}}
\newcommand{\respNL}{\resp_{N,L}}
\newcommand{\cycNL}{\cyc_{N,L}}

\makeatletter
\DeclareRobustCommand{\muP}{%
  \texorpdfstring{%
    \if b\expandafter\@car\f@series\@nil
      \ensuremath{\bm{\mu}}%
    \else
      \ensuremath{\mu}%
    \fi
    P}%
  {muP}\xspace}
\makeatother
\newcommand{\LoopmuP}{Loop-\muP}
\newcommand{\DepthmuP}{Depth-\muP}

\definecolor{widthterm}{HTML}{34816A}
\definecolor{depthterm}{HTML}{384CBE}
\definecolor{loopterm}{HTML}{DC702A}
\newcommand{\wterm}[1]{{\color{widthterm}#1}}
\newcommand{\dterm}[1]{{\color{depthterm}#1}}
\newcommand{\lterm}[1]{{\color{loopterm}#1}}

\title{\LoopmuP: Residual Scaling and Hyperparameter Transfer in Looped Transformers}

\author[1,2,3]{Shaowen Wang}
\author[1,2,3]{Bingrui Li}
\author[1]{Tansheng Zhu}
\author[2,3]{Ge Zhang}
\author[2]{\\Wenhao Huang}
\author[2]{Shen Yan}
\author[1]{Jian Li}

\affiliation[1]{Tsinghua University}
\affiliation[2]{ByteDance Seed}
\affiliation[3]{M-A-P}

\abstract{%
Looped Transformers have drawn growing attention because looping improves multi-step reasoning and increases the executed depth of a model without adding parameters.
They achieve this by reusing a stack of $L$ layers for $N$ loops.
However, increasing the loop count $N$ not only destabilizes training but also shifts the best hyperparameters, so they must be searched again for every loop count.
Inspired by \DepthmuP~\citep{yang2024tensor}, which multiplies each residual branch by $L^{-1/2}$ for depth $L$ so that a learning rate tuned on a shallow untied network transfers to a deep one, we study how the residual branches and the learning rate should scale with the loop count $N$.
We find that loops cannot simply be treated as extra depth, because the loops share weights.
At initialization, distinct layers add nearly uncorrelated increments to the residual stream, so their sum grows as $\sqrt{L}$, whereas a shared layer adds positively correlated increments at every loop, so their sum grows linearly in $N$.
For a single layer looped $N$ times, we show in the infinite-width limit that multiplying the residual branch by $1/N$ and keeping the learning rate independent of $N$ is the only parameterization with both a stable forward pass and maximal residual-stream updates.
We then generalize the analysis to a stack of $L$ distinct layers looped $N$ times to obtain a parameterization we call \LoopmuP, in which the loop factor $1/N$ combines with the depth factor $L^{-1/2}$.
Empirically, we find in language modeling that the base learning rate tuned on the smallest model stays best under \LoopmuP for loop counts up to $8$ and executed depths up to $384$ layers.
These results make \LoopmuP a practical recipe for scaling looped Transformers.\looseness=-1

}

\date{September 14, 2026}
\correspondence{\email{wangsw23@mails.tsinghua.edu.cn}, \email{zhangge.eli@bytedance.com}}

\begin{document}

\maketitle
\raggedbottom

\etocdepthtag.toc{mainmatter}
\section{Introduction}

Looped Transformers, which apply the same stack of $L$ layers $N$ times~\citep{dehghani2018universal,lan2019albert}, have drawn growing attention because they reach an executed depth of $NL$ layers while storing the parameters of only $L$ layers.
Looped Transformers can express and learn iterative algorithms by running one step per loop~\citep{saunshi2025latent,gatmiry2024multistep}, and on synthetic tasks they match standard Transformers with far fewer parameters and generalize to longer inputs~\citep{yang2023looped,fan2024looped}.
In language modeling, \citet{saunshi2025latent} and \citet{zhu2025ouro} observe that downstream accuracy rises as the loop count $N$ grows, most strongly on reasoning tasks.
Increasing the loop count $N$ is therefore a natural way to scale a looped model.

\begin{figure}[t]
  \centering
  \setlength{\abovecaptionskip}{3pt}
  \resizebox{\linewidth}{!}{%
\begin{tikzpicture}[
  x=1cm, y=0.90cm,
  font=\sffamily,
  >={Latex[length=2mm,width=1.4mm]},
  paneldeep/.style ={rounded corners=4pt, draw=none, fill=deepblue!4},
  paneloop/.style  ={rounded corners=4pt, draw=none, fill=looporange!5},
  panelphase/.style={rounded corners=4pt, draw=none, fill=ink!3},
  blk/.style={rounded corners=2.5pt, draw=#1!72!black, fill=#1!12,
    line width=0.7pt, minimum width=0.80cm, minimum height=0.62cm,
    align=center, inner sep=1.5pt, font=\sffamily\small},
  flow/.style     ={->, draw=ink!45, line width=0.55pt,
                    shorten >=1.2pt, shorten <=1.2pt},
  stepd/.style    ={-{Latex[length=1.7mm,width=1.3mm]},
                    draw=deepblue!72!black,  line width=0.85pt},
  stepl/.style    ={-{Latex[length=1.7mm,width=1.3mm]},
                    draw=looporange!75!black, line width=0.85pt,
                    shorten >=0.6pt, shorten <=0.6pt},
  sumarr/.style   ={-{Latex[length=2.4mm,width=1.9mm]},
                    draw=ink!92, line width=1.4pt},
  origin/.style   ={circle, fill=ink!90, inner sep=0pt, minimum size=2.6pt},
  boundary/.style ={ink!60, dashed, line width=0.6pt},
  axis/.style     ={draw=ink!80, line width=0.6pt},
  tick/.style     ={font=\sffamily\footnotesize, text=ink!90},
  region/.style   ={align=center, font=\sffamily\fontsize{7.3}{8.4}\selectfont, text=ink!85},
  critical/.style ={star, star points=5, star point ratio=2.1,
                    fill=phaseAccent, draw=black, inner sep=0pt, minimum size=8pt},
  ttl/.style    ={font=\sffamily\bfseries\large, text=ink},
  tag/.style    ={font=\sffamily\itshape\small,  text=ink!60},
  math/.style   ={font=\sffamily\small,           text=ink!90},
  rscale/.style ={font=\sffamily\footnotesize,    text=rescale!85!black},
  verdict/.style={font=\sffamily\bfseries\small,  text=#1},
]

\definecolor{deepblue}  {RGB}{ 56, 76,190}
\definecolor{looporange}{RGB}{220,112, 42}
\definecolor{ink}       {RGB}{ 32, 38, 52}
\definecolor{rescale}   {RGB}{170, 80, 30}
\definecolor{phaseStable}  {RGB}{216,237,231}
\definecolor{phaseUnstable}{RGB}{251,232,199}
\definecolor{phaseAccent}  {RGB}{173,102,  0}

\def\pwa{5.95}       %
\def\pwb{5.40}       %
\def\pwc{6.23}       %
\def\gap{0.30}       %
\def\ph{5.15}        %
\def\px{0.30}        %
\def\sx{1.72}        %
\def\yTitle{4.75}
\def\yTag{4.30}
\def\yArch{3.60}
\def\yEqA{2.62}
\def\yEqB{2.17}
\def\yNorm{1.52}
\def\yVerd{0.98}
\def\yRule{0.52}

\pgfmathsetmacro{\xb}{\pwa+\gap}
\pgfmathsetmacro{\xc}{\pwa+\pwb+2*\gap}

\fill[paneldeep]  (0,0) rectangle (\pwa,\ph);
\fill[paneloop]   (\xb,0) rectangle (\xb+\pwb,\ph);
\fill[panelphase] (\xc,0) rectangle (\xc+\pwc,\ph);

\begin{scope}
  \node[ttl, anchor=west] at (\px, \yTitle) {(a)~Untied residual network};
  \node[tag, anchor=west] at (\px, \yTag)   {$L$ distinct weight matrices};

  \node[blk=deepblue] (b0) at (0.85, \yArch) {$W_0$};
  \node[blk=deepblue, right=0.24cm of b0] (b1) {$W_1$};
  \node[blk=deepblue, right=0.24cm of b1] (b2) {$W_2$};
  \node[font=\sffamily\bfseries, text=ink!45, inner sep=1pt,
        minimum width=0.40cm, right=0.24cm of b2] (bd) {$\cdots$};
  \node[blk=deepblue, right=0.24cm of bd] (b3) {$W_{L-1}$};
  \draw[flow] ($(b0.west)+(-0.30,0)$) -- (b0.west);
  \draw[flow] (b0.east) -- (b1.west);
  \draw[flow] (b1.east) -- (b2.west);
  \draw[flow] (b2.east) -- (bd.west);
  \draw[flow] (bd.east) -- (b3.west);
  \draw[flow] (b3.east) -- ++(0.30,0);

  \node[math, anchor=west] at (\px, \yEqA)
    {$x_{\ell} = x_{\ell-1} + {\color{rescale}c}\, \mathcal{F}_\ell(h_\ell)$};
  \node[rscale, anchor=west] at (\px, \yEqB)
    {${\color{rescale}c}$: branch multiplier,~$h_\ell = W_\ell\, x_{\ell-1}$};

  \begin{scope}[shift={(0.36, 0.44)}, scale=0.92]
    \coordinate (s0) at (0.00, 0.00);
    \coordinate (s1) at (0.45, 0.45);
    \coordinate (s2) at (0.12, 0.92);
    \coordinate (s3) at (0.72, 1.10);
    \coordinate (s4) at (1.15, 0.78);
    \coordinate (s5) at (1.02, 0.28);
    \draw[sumarr, shorten >=1.2pt] (s0) -- (s5);
    \draw[stepd] (s0) -- (s1);
    \draw[stepd] (s1) -- (s2);
    \draw[stepd] (s2) -- (s3);
    \draw[stepd] (s3) -- (s4);
    \draw[stepd] (s4) -- (s5);
    \node[origin] at (s0) {};
  \end{scope}

  \node[math, anchor=west] at (\sx, \yNorm)
    {$\displaystyle\bigl\|\textstyle\sum_{\ell} \mathcal{F}_\ell(h_\ell)\bigr\| = \Theta(\sqrt{L})$};
  \node[verdict=deepblue!80!black, anchor=west] at (\sx, \yVerd)
    {Random-walk growth};
  \node[math, anchor=west] at (\sx, \yRule)
    {${\color{rescale}c}=a/\sqrt{L}$ bounds the sum};
\end{scope}

\begin{scope}[shift={(\xb,0)}]
  \node[ttl, anchor=west] at (\px, \yTitle) {(b)~Looped network};
  \node[tag, anchor=west] at (\px, \yTag)   {single shared $W$, reused $N$ times};

  \node[blk=looporange, minimum width=1.30cm] (lw) at (\pwb/2, \yArch) {$W$};
  \draw[flow] ($(lw.west)+(-0.65,0)$) -- (lw.west);
  \draw[flow] (lw.east) -- ++(0.65,0);
  \draw[-{Latex[length=1.8mm,width=1.4mm]},
        draw=looporange!85!black, line width=0.85pt]
    ($(lw.south east)+(-0.20,-0.03)$) .. controls +(0.10,-0.34) and +(-0.10,-0.34) ..
    ($(lw.south west)+(0.20,-0.03)$);
  \node[font=\sffamily\bfseries\footnotesize, text=looporange!90!black, anchor=west]
    at ($(lw.south east)+(-0.10,-0.20)$) {$\times N$};

  \node[math, anchor=west] at (\px, \yEqA)
    {$x_{r} = x_{r-1} + {\color{rescale}c}\, \mathcal{F}(h_r)$};
  \node[rscale, anchor=west] at (\px, \yEqB)
    {${\color{rescale}c}$: branch multiplier,~$h_r = W x_{r-1}$};

  \begin{scope}[shift={(0.36, 1.16)}]
    \coordinate (c0) at (0.00, 0.00);
    \coordinate (c1) at (0.24, 0.03);
    \coordinate (c2) at (0.48, 0.02);
    \coordinate (c3) at (0.72, 0.06);
    \coordinate (c4) at (0.96, 0.05);
    \coordinate (c5) at (1.20, 0.09);
    \draw[stepl] (c0) -- (c1);
    \draw[stepl] (c1) -- (c2);
    \draw[stepl] (c2) -- (c3);
    \draw[stepl] (c3) -- (c4);
    \draw[stepl] (c4) -- (c5);
    \node[origin] at (c0) {};
    \draw[sumarr] ($(c0)+(0,-0.40)$) -- ($(c5)+(0,-0.40)$);
    \node[origin] at ($(c0)+(0,-0.40)$) {};
  \end{scope}

  \node[math, anchor=west] at (\sx, \yNorm)
    {$\displaystyle\bigl\|\textstyle\sum_r \mathcal{F}(h_r)\bigr\| = \Theta(N)$};
  \node[verdict=looporange!85!black, anchor=west] at (\sx, \yVerd)
    {Linear growth};
  \node[math, anchor=west] at (\sx, \yRule)
    {${\color{rescale}c}=a/N$ bounds the sum};
\end{scope}

\begin{scope}[shift={(\xc,0)}]
  \node[ttl, anchor=west] at (\px, \yTitle) {(c)~Scaling with $N$};
  \node[tag, anchor=west] at (\px, \yTag)   {which rules are stable and keep learning};

  \node[font=\sffamily\small, text=ink, rotate=90] at (0.52, 2.45) {Learning rate};
  \node[font=\sffamily\small, text=ink] at (3.72, 0.22)
    {Branch multiplier ${\color{rescale}c}$};

  \begin{scope}[shift={(1.565cm, 2.593*0.9cm)}, x=2.15cm, y=1.111cm]
    \def\amin{-0.10}\def\amax{2.10}
    \def\gmin{-1.25}\def\gmax{1.10}
    \fill[gray!13]        (\amin,\gmin) rectangle (1,\gmax);
    \fill[phaseStable]    (1,0) rectangle (\amax,\gmax);
    \fill[phaseUnstable]  (1,\gmin) rectangle (\amax,0);
    \draw[boundary] (1,\gmin) -- (1,\gmax);
    \draw[boundary] (1,0) -- (\amax,0);
    \draw[axis] (\amin,\gmax) -- (\amin,\gmin) -- (\amax,\gmin);

    \foreach \a/\lab in {0/{$a$}, 0.5/{$a/\sqrt{N}$}, 1/{$a/N$}, 1.5/{$a/N^{3/2}$}}
      \draw[axis] (\a,\gmin) -- ++(0,-0.07) node[tick, below, inner sep=1.5pt] {\lab};
    \foreach \g/\lab in {-1/{$\eta N$}, 0/{$\eta$}, 1/{$\eta/N$}}
      \draw[axis] (\amin,\g) -- ++(-0.035,0) node[tick, left, inner sep=1.5pt] {\lab};

    \node[region] at (0.42,-0.62) {Initialization\\unstable};
    \node[region] at (1.55,0.55) {Initialization\\stable, but\\updates vanish};
    \node[region] at (1.55,-0.62) {Initialization\\stable, but\\training unstable};

    \node[critical] at (1,0) {};
    \node[font=\sffamily\bfseries\small, text=phaseAccent!80!black, anchor=south east,
          inner sep=1.5pt] at (0.97,0.06) {\LoopmuP};
  \end{scope}
\end{scope}

\end{tikzpicture}
}
  \caption{\textbf{Weight sharing changes how residual contributions add up, leaving one stable scaling with $N$.}
  \textbf{(a)}~In an untied network, the $L$ layers contribute nearly independent directions, so residual stream grows as $\bigTheta(\sqrt{L})$ and the branch multiplier $c=a/\sqrt{L}$ with a constant $a>0$ keeps it bounded.
  \textbf{(b)}~In a looped network, one block is reused $N$ times and its contributions are positively correlated, so residual stream grows as $\bigTheta(N)$ and $c$ must scale as $a/N$.
  \textbf{(c)}~Each parameterization sets $c=aN^{-\alpha_N}$ and the Adam learning rate $\eta N^{-\gamma_N}$.
  A branch multiplier that decays more slowly than $1/N$ ($\alpha_N<1$) lets the residual stream grow with $N$ at initialization (left).
  For $\alpha_N>1$, $\gamma_N\ge0$ makes the updates vanish (upper right), and $\gamma_N<0$ makes training unstable (lower right).
  At $\alpha_N=1$, the same holds for every $\gamma_N>0$ and $\gamma_N<0$.
  The only choice that keeps the residual stream bounded and the updates of order one is therefore $\alpha_N=1$ and $\gamma_N=0$, which is \LoopmuP.}
  \label{fig:main-schematic}
\vspace{-0.42cm}
\end{figure}

However, increasing the loop count $N$ not only destabilizes training but also changes the best learning rate, requiring retuning at every $N$~\citep{geiping2025scaling,zhu2025ouro,prairie2026parcae}.
For example, \citet{zhu2025ouro} reduce the loop count after observing loss spikes, and they adopt a conservative learning rate because their compute budget does not allow a learning-rate sweep.

Since a looped model executes $L \times N$ layers, a natural fix is to treat the loops as extra depth.
Deep networks solve this problem by introducing depthwise parameterizations, which shrink each residual branch as the depth $L$ grows~\citep{yang2024tensor,dey2025completep}.
Each branch $f$ updates the residual stream as $\vx\gets\vx+c\,f(\vx)$, and because a deeper network adds more such updates, these parameterizations set the branch multiplier to $c\propto L^{-\alpha_L}$.
Theory supports the whole range $\alpha_L\in[1/2,1]$, from \DepthmuP ($\alpha_L=1/2$)~\citep{yang2024tensor,bordelon2024depthwise} to CompleteP ($\alpha_L=1$)~\citep{dey2025completep}, and experiments find that both ends transfer the learning rate~\citep{mlodozeniec2025completed,hagele2026decoupling}.
We ask whether the loop count $N$ admits the same range of exponents.

We answer this question by analyzing a single residual block looped $N$ times.
At initialization, the loops share the same weights, so their residual increments $c\,f(\vx_{r-1})$ are positively correlated, and their sum $\sum_{r=1}^{N} c\,f(\vx_{r-1})$ grows linearly with $N$, whereas the increments $c\,f_\ell(\vx_{\ell-1})$ of $L$ distinct layers are independent and their sum $\sum_{\ell=1}^{L} c\,f_\ell(\vx_{\ell-1})$ grows only as $\sqrt{L}$ (\cref{fig:main-schematic}).
\cref{fig:energy} validates this analysis in a Transformer. While a branch multiplier proportional to $L^{-1/2}$ keeps the residual stream of an untied model of order one, only a branch multiplier proportional to $N^{-1}$ does so for a looped model.
Similarly, during training, all $N$ loops receive the same weight update, so the changes it causes in their residual increments are also positively correlated, and the change in the residual stream after one training step grows linearly with $N$.
Therefore, only a branch multiplier of $1/N$ keeps the changes in the residual stream of order one (\cref{fig:lr-step-update}, right).
This differs from depth, which admits every exponent in $[1/2,1]$~\citep{yang2024tensor,bordelon2024depthwise,dey2025completep}.

We formalize this argument in the spirit of the Tensor Programs series~\citep{yang2021tensor,yang2024tensor}.
We describe a parameterization by how its branch multiplier and learning rate scale with $N$, and we classify all such parameterizations as the width and then $N$ grow to infinity.
Only the branch multiplier $1/N$ with an $N$-independent learning rate keeps the forward pass stable and the residual-stream updates of order one (\cref{sec:one-block}, \cref{fig:main-schematic}c).
We prove this classification for linear networks trained with SGD and extend it to blocks with multiple weight matrices, such as the attention and MLP blocks of a Transformer.
For nonlinear activations and Adam, we justify it under additional assumptions (\cref{app:loop-classification,app:fixed-k-proofs,app:nonlinear-justifications}).

Furthermore, we generalize the analysis from looping a single block to looping a stack of $L$ independent layers, which yields a parameterization we call \textbf{\LoopmuP}.
\LoopmuP multiplies each residual branch by a factor proportional to $1/(N\sqrt{L})$ and keeps the Adam learning-rate factor $1/\sqrt{L}$, with no factor of $N$ (\cref{app:loop-depth,app:loop-depth-proofs}).
\cref{tab:loop-mup-recipe} summarizes \LoopmuP as a recipe for prenorm Transformers. (\cref{sec:loop-mup-recipe,app:loop-mup-rows}).

Empirically, we validate \LoopmuP by training looped Transformers for language modeling on FineWeb-Edu~\citep{penedo2024fineweb}.
We find that under \LoopmuP, the learning rate tuned on the smallest model stays best for loop counts up to $8$ and executed depths up to $384$ layers.
We also compare \LoopmuP with three other loop factors at $12$ layers.
The factor $1$ is the default implementation without a branch multiplier, $N^{-1/2}$ treats the loops as extra depth under \DepthmuP, and $N^{-3/2}$ decays faster than $1/N$.
Under all three, the best learning rate shifts with $N$, and the best loss over the learning-rate sweep is higher than under \LoopmuP at every loop count above one (\cref{sec:experiments}).

Together, our theory and experiments offer a simple recipe for scaling looped Transformers, thus enabling the loop count to join width and depth as a third axis for scaling language models.

\section{Related Work}
\label{sec:related-work}

\paragraph{Looped Transformers.}
Looped Transformers apply the same stack of layers several times, which increases the executed depth without adding parameters.
The idea goes back to recurrent depth in Universal Transformers~\citep{dehghani2018universal} and cross-layer parameter sharing in ALBERT~\citep{lan2019albert}.
In theory, looping lets fixed weights carry out an iterative computation.
\citet{giannou2023looped} construct a looped Transformer with fewer than 13 layers that emulates a one-instruction computer and can run general programs.
\citet{gatmiry2024multistep} prove that for in-context linear regression, the best looped linear Transformer implements multi-step preconditioned gradient descent.
\citet{saunshi2025latent} show that $T$ loops can simulate $T$ steps of chain-of-thought.
In practice, looped models match standard Transformers with far fewer parameters when learning function classes in context~\citep{yang2023looped}.
They also generalize to longer inputs on algorithmic tasks when the loop count grows with the input length~\citep{fan2024looped}.
In language modeling, looping part of a model improves loss at a fixed parameter count~\citep{ng2024loopnn}.
\citet{saunshi2025latent} find that looping helps reasoning tasks more than memorization tasks such as closed-book question answering.
\citet{zhu2025ouro} report that accuracy rises with the loop count up to the count used in training, and \citet{geiping2025scaling} improve reasoning benchmarks by running more loops at test time.

\paragraph{Training stability of deep and looped networks.}
As depth grows, the contributions of successive layers accumulate, so the forward signal and its updates can explode or vanish.
For deep networks with independent layers, one group of fixes changes the architecture.
ReZero gates each residual branch with a learnable scalar that starts at zero and removes layer normalization~\citep{bachlechner2020rezero}.
DeepNet up-weights the skip connection before each layer normalization and scales down the initialization~\citep{wang2022deepnet}.
Another group keeps the architecture and changes only the scaling.
Fixup rescales the branch weights at initialization according to depth~\citep{zhang2019fixup}.
\citet{marion2022scalingresnets} characterize which residual multipliers keep the signal nontrivial as depth grows.
They find that weights varying smoothly across layers require the residual multiplier $1/L$, whose limit is a neural ODE~\citep{chen2018neural}.
Looped models are harder to train, and most existing fixes change the architecture.
\citet{geiping2025scaling} observe that the hidden states in their first large run collapse, and they succeed only after changing the normalization and cutting the learning rate by a factor of ten.
\citet{prairie2026parcae} treat the loop as a dynamical system and constrain the matrix that carries the hidden state from one loop to the next, so that the state cannot explode.
Concurrently, DeepLoop~\citep{li2026deeploop} and FPRM~\citep{movahedi2026fixed} adjust the residual scaling to stabilize looped Transformers, and \citet{chen2026training} damp each residual update by $1/N$ when looping frozen pretrained blocks.
None of these derive how hyperparameters should scale with $N$.

\paragraph{Hyperparameter transfer and parameterization.}
Tensor Programs derive the maximal-update parameterization, under which hyperparameters tuned on a narrow model transfer to wider ones~\citep{yang2021tensor,yang2022tensorprogramsV}, and extend it to adaptive optimizers~\citep{yang2023tensor}.
\DepthmuP scales each residual branch by $L^{-1/2}$ to maximize feature diversity across layers~\citep{yang2024tensor}, and \citet{bordelon2024depthwise} and \citet{hayou2023commute} study the same scaling through its training dynamics and its order of limits.
CompleteP chooses $L^{-1}$ so that every layer stays non-lazy and reports more reliable depth transfer than with $L^{-1/2}$~\citep{dey2025completep}.
Later work reports learning-rate transfer under both exponents, with slightly lower loss for $L^{-1/2}$ in two settings~\citep{mlodozeniec2025completed,hagele2026decoupling}, and extends the $L^{-1}$ rule to matrix-preconditioned optimizers~\citep{qiu2025hyperparameter}.
All of these works scale depth with independent weights, where either exponent in $[1/2,1]$ transfers the learning rate.
We show that the loop axis differs.
Because a shared layer adds positively correlated increments at every loop, only the branch multiplier $1/N$ with an $N$-independent learning rate keeps the forward pass stable and the updates maximal.
\LoopmuP combines this loop factor with either depth exponent (\cref{tab:loop-mup-recipe}).

\section{Loop Scaling with One Shared Residual Block}\label{sec:one-block}

\subsection{Setup}\label{sec:one-block-setup}

We start with a single residual block applied $N$ times with shared weights.
Our goal is to determine how the branch multiplier and the learning rate should scale with $N$.
For a Transformer, this corresponds to a single attention or MLP block, and the corresponding recipe is given in \cref{sec:loop-mup-recipe}.

\begin{figure}[t]
  \centering
  \includegraphics[width=\linewidth]{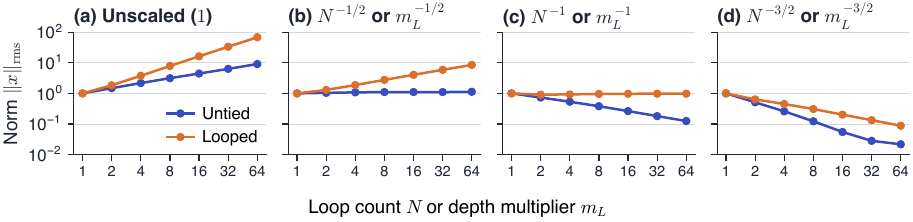}
  \caption{\textbf{Residual-stream norm at initialization under different branch multipliers for untied and looped models.}
  We measure the normalized residual-stream norm $\normrms{\vx}$ of Transformers at initialization and plot it (log scale) against the loop count $N$ of a looped model (orange) or the depth multiplier $m_L=L/12$ of an untied model with $L$ layers (blue).
  The looped model repeats one 12-layer pre-norm Transformer stack $N$ times, and the untied model chains $m_L$ independent copies of that stack.
  The norm of the untied model stays of order one only under the branch multiplier $m_L^{-1/2}$, while the norm of the looped model stays of order one only under $N^{-1}$.}
  \label{fig:energy}
\vspace{-0.3cm}
\end{figure}

\paragraph{Model.}
Following \DepthmuP~\citep{yang2024tensor}, we use an MLP of width $n$ with a scalar input $\zeta$ and a scalar output.
The trainable hidden matrix is $\mW_t\in\RR^{n\times n}$, with $t$ indexing training steps; the input and readout matrices $\mU\in\RR^{n\times1}$ and $\mV\in\RR^{n\times1}$ stay fixed.
Each loop applies a coordinatewise activation $\phi$ and adds the branch output to the hidden state, scaled by the branch multiplier $c_N=aN^{-\alpha_N}$ for a constant $a>0$.
As in \DepthmuP, we subtract the mean of the branch output over its coordinates, $\MS(\vz)=\vz-\vone_n\vone_n^\transpose\vz/n$, where $\vone_n$ is the all-ones vector.
Set $\vx_t^0(\zeta)=\mU\zeta$. The network computes
\begin{equation}
    \vh_t^r(\zeta)=\mW_t\vx_t^{r-1}(\zeta), \qquad
    \vx_t^r(\zeta)=\vx_t^{r-1}(\zeta)+c_N\MS(\phi(\vh_t^r(\zeta))), \qquad
    f_t(\zeta)=\mV^\transpose \vx_t^N(\zeta),
\end{equation}
for $1\le r\le N$, where $\vx_t^r$ is the hidden state after $r$ loops and $\vh_t^r$ is the preactivation in loop $r$.
An untied $L$-layer network has the same form, except that each layer has its own hidden matrix.
The exponent $\alpha_N$ controls how fast the branch multiplier decreases with $N$.
For an untied $L$-layer network, \DepthmuP sets the branch multiplier to $aL^{-1/2}$, corresponding to the unique optimal depth exponents $(\alpha_L,\gamma_L)=(1/2,1/2)$.
We will see that sharing the hidden matrix aligns residual contributions across loops, requiring a loop exponent of $1$ rather than $1/2$.

\paragraph{Initialization and limits.}
We initialize the weights with \muP~\citep{yang2021tensor}, $\mU_i\iid\fN(0,1)$, $\mV_i\iid\fN(0,n^{-2})$, and $\mW_{0,ij}\iid\fN(0,n^{-1})$.
Under this initialization, each coordinate of the hidden state and preactivation is of order one as $n\to\infty$ at any fixed loop index.
We analyze the network in the limit of large width and many loops.
As in \DepthmuP, we take $n\to\infty$ before $N\to\infty$.

\paragraph{Training.}
We train only the shared hidden matrix, and keep $\mU$ and $\mV$ fixed.
At step $t$, we draw a training input $\xi_t$ and a differentiable loss $\ell_t$, and write $\mG_t=\nabla_{\mW}\ell_t(f_t(\xi_t))$ for the gradient of $\mW$.
Following Tensor Programs~IVb and~VI~\citep{yang2023tensor,yang2024tensor}, we rescale $\mG_t$ by $nN^{\alpha_N-1}$ so that its entries stay of order one (\cref{sec:one-block-desiderata}).
With a base learning rate $\eta>0$, an exponent $\gamma_N$, and a map $Q_t$ that turns the rescaled gradients into an update direction, the update is
\begin{equation}\label{eqn:one-block-optimizer-update}
    \widetilde{\mG}_t=nN^{\alpha_N-1}\mG_t,
    \qquad
    \mW_{t+1}=\mW_t-\eta n^{-1}N^{-\gamma_N} Q_t(\widetilde{\mG}_0,\dots,\widetilde{\mG}_t).
\end{equation}
For SGD, $Q_t$ returns the last rescaled gradient, whereas for Adam it applies entrywise the bias-corrected ratio $\widehat{m}_t/(\sqrt{\widehat{v}_t}+\epsilon_{\mathrm{opt}})$ of the first- and second-moment estimates of the rescaled-gradient history.
The question of this section is which pair $(\alpha_N,\gamma_N)$ to use.

\subsection{Desiderata for hyperparameter transfer}\label{sec:one-block-desiderata}

We require a stable forward pass at initialization, and training updates that keep the hidden-state update bounded while producing an order-one output response.
To read off the exponents, we take the identity activation, drop mean subtraction, and study one nondegenerate SGD step.
All scales below take $n\to\infty$ before $N\to\infty$, at fixed inputs, and we use the RMS norm $\normrms{\vz}=\norm{\vz}/\sqrt{n}$.

\paragraph{Forward pass at initialization.}
The first requirement keeps the network's representations and output bounded as the loop count grows.
\begin{desideratum}[Stable initialization]\label{des:one-block-stable-initialization}
    For every fixed bounded input $\zeta$, the initialized network should satisfy that all three of $\normrms{\vx_0^r(\zeta)}$, $\normrms{\vh_0^r(\zeta)}$, and $|f_0(\zeta)|$ are $\bigO(1)$, uniformly over $1\le r\le N$.
\end{desideratum}

Take the input $\zeta=1$.
With a shared matrix, the forward map becomes $\vx_0^N(1)=(\mI+c_N\mW_0)^N\mU$, and every way of picking $m$ residual terms produces the \textit{same} vector $\mW_0^m\mU$.
At fixed $N$, these vectors become orthogonal in normalized inner product as $n\to\infty$, so
\[
    \vx_0^N(1)
    =\sum_{m=0}^N\binom{N}{m}c_N^m\mW_0^m\mU,
    \qquad
    \normrms{\vx_0^N(1)}^2
    \pto
    \sum_{m=0}^N\binom{N}{m}^2c_N^{2m}.
\]
If the layers are instead independent, distinct products remain orthogonal and the squared norm tends to $\sum_{m=0}^N\binom{N}{m}c_N^{2m}=(1+c_N^2)^N$.
The linear increment ($m=1$) thus has size $N c_N$ under shared loops and only $\sqrt{N} c_N$ under untied layers.

With $\accN=N^{1-\alpha_N}$, the linear residual contribution has scale $Nc_N=a\accN$.
Thus the initialized residual stream grows with $N$ when $\alpha_N<1$, stays of order one when $\alpha_N=1$, and vanishes toward the input $\mU$ when $\alpha_N>1$.
\cref{des:one-block-stable-initialization} therefore needs $\alpha_N\ge1$.
In our experiments on the pre-norm Transformer (\cref{fig:energy,fig:corr}), the same accumulation mechanism is at work: residual increments across loops are more aligned than those across independent 12-layer stacks, producing linear rather than square-root norm growth without a branch multiplier (\cref{sec:init-scaling}).

\begin{figure}[t]
  \centering
  \begin{subfigure}[b]{0.605\linewidth}
    \centering
    \includegraphics[height=1.45in]{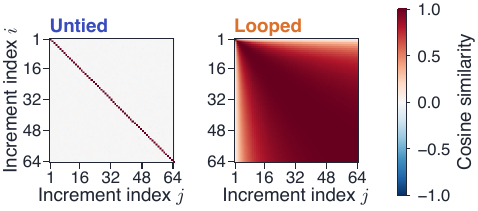}
  \end{subfigure}\hfill
  \begin{subfigure}[b]{0.39\linewidth}
    \centering
    \includegraphics[height=1.45in]{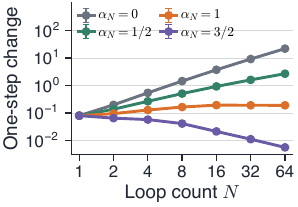}
  \end{subfigure}
  \caption{\textbf{Left:} we measure in Transformers at initialization the cosine similarity between the residual-stream increments $\vx^r-\vx^{r-1}$ of successive 12-layer stacks, without residual scaling.
  In an untied model of $64$ independent stacks the increments are nearly orthogonal, while in a looped model with $N{=}64$ they are all positively correlated.
  \textbf{Right:} we measure the RMS change of the final residual stream after one AdamW step at a learning rate shared across $N$.
  Setup as in \cref{sec:init-scaling}.}
  \label{fig:corr}\label{fig:lr-step-update}
\vspace{-0.3cm}
\end{figure}

\paragraph{One gradient step.}
Write $\Delta\vx_0^r(\zeta)=\vx_1^r(\zeta)-\vx_0^r(\zeta)$, and define $\Delta\vh_0^r$ and $\Delta f_0$ likewise.
The second requirement asks for bounded updates to the residual stream and order-one output responses.
\begin{desideratum}[Stable training and nontrivial response]\label{des:one-block-stable-training-and-nontrivial-response}
    For bounded inputs and loss derivatives, a single gradient update should keep $\normrms{\Delta\vx_0^r(\zeta)}$ and $\normrms{\Delta\vh_0^r(\zeta)}$ uniformly $\bigO(1)$ over $1\le r\le N$.
    For a nonzero input and a nonzero loss derivative, it should additionally satisfy $|\Delta f_0(\zeta)|=\bigTheta(1)$.
\end{desideratum}

We now work in the initialization-stable regime $\alpha_N\ge1$.
For a nonzero input $\xi_0$ and a nonzero initial loss derivative $\chi_0=\ell_0'(f_0(\xi_0))$, the shared gradient collects contributions from all loops:
\[
    \mG_0=c_N\sum_{r=1}^N
    \frac{\partial\ell_0}{\partial\vx_0^r}(\xi_0)\,
    \vx_0^{r-1}(\xi_0)^\transpose,
    \qquad
    \widetilde{\mG}_0=nN^{\alpha_N-1}\mG_0.
\]
Each backward signal inherits the readout scale $1/n$, and the $N$ aligned terms bring an extra factor $Nc_N=a\accN$, so the entries of $\mG_0$ are of order $\accN/n$.
Multiplying by $nN^{\alpha_N-1}$ then yields $\widetilde{\mG}_0 = \bigTheta(1)$.
An SGD step therefore changes each entry of $\mW$ by order $\updN/n$, where $\updN=N^{-\gamma_N}$.

\begin{table*}[t]
    \caption{
        \textbf{Recipe of \LoopmuP.}
        The left column is the untied recipe of \citet{dey2025completep}, and the right column shows its looped counterpart.
        Width, depth, and loop terms are shown in \wterm{green}, \dterm{blue}, and \lterm{orange}.
        \LoopmuP uses $\alpha_L=1/2$.
        Rows marked $\ddagger$ use the loop-depth theory (\cref{app:loop-mup-rows}), and rows marked $\dagger$ retain the untied recipe.
        The base values $\sigma$, $\eta$, $\omega$, and $\epsilon_{\mathrm{opt}}$ are tuned on the base model, where both columns coincide.
    }
    \label{tab:loop-mup-recipe}
    \centering
    \footnotesize
    \setlength{\tabcolsep}{10pt}
    \renewcommand{\arraystretch}{0.88}
    \begin{tabular}{@{}lll@{}}
        \toprule
        & Untied, $L$ layers & Looped, $L$ layers $\times$ $N$ loops \\
        \midrule
        Emb.\ init.\ var.$^{\dagger}$ & $\sigma^2$ & $\sigma^2$ \\
        Emb.\ LR$^{\dagger}$ & $\eta$ & $\eta$ \\
        Pre-LN init.\ var.$^{\dagger}$ & $\sigma^2$ & $\sigma^2$ \\
        Pre-LN LR$^{\dagger}$ & $\eta\,\dterm{m_L^{\alpha_L-1}}$ & $\eta\,\dterm{m_L^{\alpha_L-1}}$ \\
        Hidden init.\ var.$^{\ddagger}$ & $\sigma^2\,\wterm{m_n^{-1}}$ & $\sigma^2\,\wterm{m_n^{-1}}$ \\
        Hidden LR$^{\ddagger}$ & $\eta\,\wterm{m_n^{-1}}\,\dterm{m_L^{\alpha_L-1}}$ & $\eta\,\wterm{m_n^{-1}}\,\dterm{m_L^{\alpha_L-1}}$ \\
        Hidden bias LR$^{\dagger}$ & $\eta\,\dterm{m_L^{\alpha_L-1}}$ & $\eta\,\dterm{m_L^{\alpha_L-1}}$ \\
        Hidden WD$^{\dagger}$ & $\omega\,\wterm{m_n}$ & $\omega\,\wterm{m_n}$ \\
        MHA residual$^{\ddagger}$ & $\mX+\dterm{m_L^{-\alpha_L}}\,\mathrm{MHA}(\mathrm{LN}(\mX))$ & $\mX+\lterm{N^{-1}}\,\dterm{m_L^{-\alpha_L}}\,\mathrm{MHA}(\mathrm{LN}(\mX))$ \\
        MLP residual$^{\ddagger}$ & $\mZ+\dterm{m_L^{-\alpha_L}}\,\mathrm{MLP}(\mathrm{LN}(\mZ))$ & $\mZ+\lterm{N^{-1}}\,\dterm{m_L^{-\alpha_L}}\,\mathrm{MLP}(\mathrm{LN}(\mZ))$ \\
        Final-LN init.\ var.$^{\dagger}$ & $\sigma^2$ & $\sigma^2$ \\
        Final-LN LR$^{\dagger}$ & $\eta$ & $\eta$ \\
        Unemb.\ init.\ var.$^{\dagger}$ & $\sigma^2$ & $\sigma^2$ \\
        Unemb.\ LR$^{\dagger}$ & $\eta$ & $\eta$ \\
        Unemb.\ forward$^{\dagger}$ & $\mX\mW_{\mathrm{unemb}}^{\transpose}\,\wterm{m_n^{-1}}$ & $\mX\mW_{\mathrm{unemb}}^{\transpose}\,\wterm{m_n^{-1}}$ \\
        AdamW $\epsilon_{\mathrm{opt}}$, residual blocks$^{\ddagger}$ & $\epsilon_{\mathrm{opt}}\,\wterm{m_n^{-1}}\,\dterm{m_L^{-\alpha_L}}$ & $\epsilon_{\mathrm{opt}}\,\wterm{m_n^{-1}}\,\dterm{m_L^{-\alpha_L}}$ \\
        AdamW $\epsilon_{\mathrm{opt}}$, emb.\ and unemb.$^{\dagger}$ & $\epsilon_{\mathrm{opt}}\,\wterm{m_n^{-1}}$ & $\epsilon_{\mathrm{opt}}\,\wterm{m_n^{-1}}$ \\
        \bottomrule
    \end{tabular}
\end{table*}

The first loop receives the raw input $\mU\zeta$, giving the exact direct response $\Delta\vh_0^1(\zeta)=\Delta\mW_0\mU\zeta$.
The gradient is aligned with $\mU$, so this response has RMS scale $\updN$.
Boundedness of this first loop therefore requires $\gamma_N\ge0$.
The same $\Delta\mW_0$ is then reused in every loop and multiplied by $c_N$ each time, so the aggregate response $|\Delta f_0(\zeta)|$ has scale $Nc_N\updN=a\respN$, where $\respN=\accN\updN=N^{1-\alpha_N-\gamma_N}$.
The first-step output response grows with $N$ if $\alpha_N+\gamma_N<1$, vanishes if $\alpha_N+\gamma_N>1$, and stays of order one if and only if $\alpha_N+\gamma_N=1$.
The first-step measurements in \cref{fig:corr} illustrate how the loop factor $N^{-\alpha_N}$ controls the response (\cref{sec:init-scaling}).

Together, \cref{des:one-block-stable-initialization,des:one-block-stable-training-and-nontrivial-response} give $\alpha_N\ge1$, $\gamma_N\ge0$, and $\alpha_N+\gamma_N=1$, whose only solution is 
$(\lterm{\alpha_N},\lterm{\gamma_N})=(\lterm{1},\lterm{0})$.
The branch multiplier is therefore $a/N$.
Substituting into \cref{eqn:one-block-optimizer-update}, the raw SGD rate $\eta N^{\alpha_N-1-\gamma_N}$ collapses to $\eta$; Adam keeps order-one coordinates inside $Q_t$ and uses the external rate $\eta n^{-1}N^{-\gamma_N}=\eta/n$.
Both rates are independent of $N$.
With this choice, one training step with sufficiently small $\eta>0$ produces order-one changes in the hidden states and preactivations at every fixed positive loop fraction: for any $\lambda\in(0,1]$, both $\normrms{\Delta\vx_0^{\floor{\lambda N}}(\zeta)}$ and $\normrms{\Delta\vh_0^{\floor{\lambda N}}(\zeta)}$ are $\bigTheta(1)$ on a fixed nonzero input.
We defer the formal classification of parameterizations and the corresponding feature-learning statements to \cref{app:loop-classification}.

\subsection{Extension to looped Transformers}\label{sec:loop-mup-recipe}

\paragraph{Theory roadmap.}
We derive $(\alpha_N,\gamma_N)=(1,0)$ as the unique choice of loop exponents by analyzing one-step SGD training of a single-layer residual network in \cref{sec:one-block-desiderata}.
In \cref{app:loop-classification}, we generalize this result to any finite training horizon by classifying the $(\alpha_N,\gamma_N)$ parameterization space (\cref{clm:one-block-initialization-stability,clm:one-block-training-and-nontriviality,clm:one-block-loop-mup-and-diversity}).
In \cref{app:fixed-k}, we then extend the analysis to multi-layer residual branches (e.g., two-layer MLPs) (\cref{clm:fixed-k-training-and-nontriviality,clm:fixed-k-loop-mup-and-feature-behavior}).
When the looped object is an untied $L$-layer residual network (\cref{app:loop-depth}), we introduce additional depth exponents $\alpha_L$ and $\gamma_L$ as in \DepthmuP, taking $L\to\infty$ before $N\to\infty$.
\cref{clm:loop-depth-initialization-stability,clm:loop-depth-training-and-nontriviality,clm:loop-depth-distinguished-point} then single out the loop-depth parameterization, which we call \LoopmuP:
$(\lterm{\alpha_N},\dterm{\alpha_L};\lterm{\gamma_N},\dterm{\gamma_L}) = \left(\lterm{1},\dterm{\frac12};\lterm{0},\dterm{\frac12}\right)$.
\cref{fig:loop-phase-diagrams,fig:loop-depth-phase-diagram} give the corresponding phase diagrams.
We translate this choice into the Transformer recipe in \cref{app:loop-mup-rows}.
The linear-network arguments under SGD are proved in \cref{app:mathematical-tools,app:fixed-k-proofs,app:loop-depth-proofs}.
\cref{app:nonlinear-justifications} sketches how the same exponents persist under additional assumptions on nonlinear activations and Adam.

\begin{figure}[t]
  \centering
  \includegraphics[width=\linewidth]{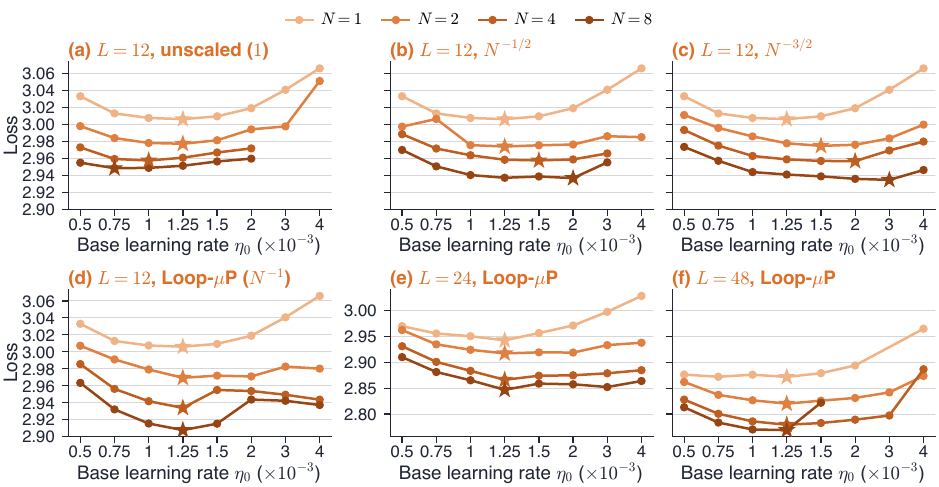}
  \caption{\textbf{Only \LoopmuP keeps one base learning rate best across loop counts and depths.}
  Final validation loss against the base learning rate $\eta_0$ for loop counts $N\in\{1,2,4,8\}$; stars mark the best learning rate for each $N$.
  Panels~(a)--(c) use the loop factors $1$, $N^{-1/2}$, and $N^{-3/2}$ at $L{=}12$, and panels~(d)--(f) apply \LoopmuP ($N^{-1}$) at $L{=}12$, $24$, and $48$.
  The hidden learning rate is $\eta_0 m_L^{-1/2}$ with $m_L=L/12$, so it equals $\eta_0$ in panels~(a)--(d) and is smaller than $\eta_0$ in panels~(e) and~(f).
  The best $\eta_0$ stays at $1.25\times10^{-3}$ for every $N$ under \LoopmuP (panels~(d)--(f)), whereas it moves with $N$ under the other loop factors (panels~(a)--(c)).
  Diverged runs (final validation loss $>4$) are omitted.}
  \label{fig:lr-transfer}
\vspace{-0.3cm}
\end{figure}

\paragraph{Recipe for looped Transformers.}
Following \citet{dey2025completep}, we work with a pre-LN Transformer of $L$ distinct layers, each consisting of one attention branch and one MLP branch, and we reuse the stack for $N$ loops.
The resulting $2L$ residual branches are the architectural counterparts of the blocks analyzed in \cref{app:loop-depth}.
We write $m_n=n/n_{\mathrm{base}}$ and $m_L=L/L_{\mathrm{base}}$ for the width and depth multipliers relative to a base model.
The depth exponents $(\alpha_L,\gamma_L)$ control how the branch multiplier and the learning rate scale with $m_L$, while $(\alpha_N,\gamma_N)$ do the same for $N$.
Because $m_L$ is proportional to $L$, the depth factor $m_L^{-\alpha_L}$ matches the factor $L^{-\alpha_L}$ of \cref{app:loop-depth} up to the constant $a$.
As in \citet{dey2025completep}, we allow $\alpha_L\in[1/2,1]$, and \LoopmuP takes $\alpha_L=1/2$.

\cref{tab:loop-mup-recipe} compares the untied recipe (left) with its looped version (right).
Motivated by \cref{sec:one-block-desiderata}, we multiply each shared residual branch by an extra $N^{-1}$ on top of the untied depth factor, which gives the branch multiplier $a N^{-1}m_L^{-\alpha_L}$.
Under \LoopmuP, the branch multiplier is $aN^{-1}m_L^{-1/2}$ and the Adam learning-rate factor on hidden weights is $m_n^{-1}m_L^{-1/2}$ (\cref{app:loop-mup-rows}).
Treating the $N$ loops as ordinary extra depth would instead scale the branch as $N^{-1/2}m_L^{-1/2}$, and would allow the hidden state to grow with $N$ (\cref{fig:energy}, panel~(b)).
Apart from the branch multiplier, the two columns share the same initialization, learning-rate, weight-decay, and AdamW $\epsilon_{\mathrm{opt}}$ rules, so a model tuned at $N=1$ can be looped by rescaling only its residual branches.

\section{Experiments}
\label{sec:experiments}

We now test whether the \LoopmuP recipe in \cref{tab:loop-mup-recipe} transfers hyperparameters across loop counts when training looped Transformer language models.
All models are decoder-only Transformers with a Llama-style pre-norm architecture~\citep{vaswani2017attention,touvron2023llama}.
We keep the token embedding and unembedding head outside the loop and repeat a stack of $L$ distinct layers for $N$ loops, and we call each loop a pass.
Each layer has two residual branches (attention and MLP), and we scale each branch by the loop factor $N^{-\alpha_N}$.
Each branch plays the role of the residual block in \cref{sec:one-block}, so the two branches per layer only change the constant $a$.
We first check the predicted residual scaling in the forward pass and the first optimizer step (\cref{sec:init-scaling}), then train language models to test learning-rate transfer across loop counts and depths (\cref{sec:lm-transfer}), and finally check that the mechanism still holds after training (\cref{sec:post-training}).

\subsection{Residual scaling in the forward pass and after one optimizer step}
\label{sec:init-scaling}

\begin{figure}[t]
  \centering
  \includegraphics[width=0.9\linewidth]{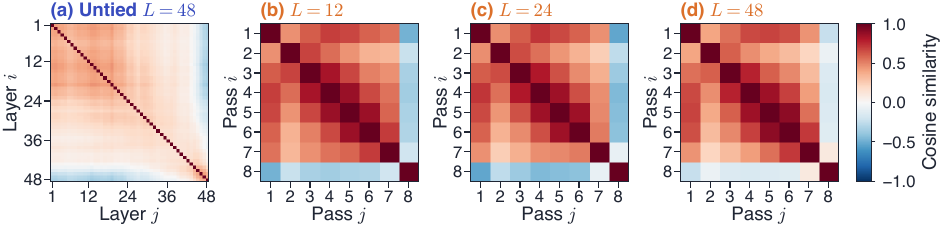}
  \caption{\textbf{After training, the pass increments of looped models stay aligned, while the per-layer increments of an untied model align mainly with nearby layers.}
  Cosine similarity between residual-stream increments after full training ($20{,}000$ steps), for the $48$ layers of an untied $L{=}48$ model (a) and for the $8$ passes $\vx^r-\vx^{r-1}$ of looped models with $N{=}8$ and $L\in\{12,24,48\}$ (b--d).
  In the looped models, every pair of passes that excludes the final pass is positively correlated, whereas in the untied model the similarity fades with layer distance.}
  \label{fig:update-cosine}
\vspace{-0.3cm}
\end{figure}

\paragraph{Protocol.}
We run a controlled experiment that tests the two requirements of \cref{sec:one-block-desiderata}, a bounded forward pass at initialization and an order-one change after one optimizer step.
We vary the loop count $N\in\{1,2,4,8,16,32,64\}$ and compare four loop factors, $1$ ($\alpha_N{=}0$), $N^{-1/2}$ ($\alpha_N{=}1/2$), $N^{-1}$ ($\alpha_N{=}1$), and $N^{-3/2}$ ($\alpha_N{=}3/2$).
The looped model repeats one 12-layer stack.
As a reference, we build untied models of $L$ layers by chaining $m_L=L/12$ independent 12-layer stacks, and we scale their branches by the depth factor $m_L^{-\alpha_L}$.
We write $\vx^r$ for the residual stream after the $r$-th 12-layer stack, which in a looped model is the $r$-th pass.
All runs use batch size $1$, sequence length $128$, and fixed random token inputs, and we average over $10$ seeds.
For the forward pass, we measure at initialization the RMS norm $\normrms{\vx^N}$ after the last pass and the pairwise cosine similarities between the increments $\vx^r-\vx^{r-1}$.
For the first step, we take one AdamW step with a base learning rate shared across all $N$ and measure the RMS change of $\vx^N$.

\paragraph{Forward pass.}
\emph{A looped model needs the loop factor $N^{-1}$ to keep the forward pass bounded, whereas an untied model needs only the depth factor $m_L^{-1/2}$.}
Without residual scaling, the looped norm grows linearly in $N$, while the untied norm grows only as $\sqrt{m_L}$ (\cref{fig:energy}, panel~(a)).
The untied increments are nearly orthogonal and add in quadrature, whereas the shared weights make the looped increments positively correlated, so they add along one direction (\cref{fig:corr}, left), as \cref{sec:one-block-desiderata} predicts.
Hence $N^{-1/2}$ still lets the looped norm grow with $N$ (panel~(b)), and only $N^{-1}$ keeps it approximately invariant across $N$ (panel~(c)).
The factor $N^{-3/2}$ also bounds the norm but drives the looped model toward the identity map (panel~(d)), so the forward pass rules out $\alpha_N<1$ (\cref{clm:one-block-initialization-stability}) without separating $\alpha_N{=}1$ from $\alpha_N{=}3/2$.

\paragraph{One optimizer step.}
\emph{With a learning rate shared across $N$, only $N^{-1}$ gives a one-step change that neither grows nor vanishes with $N$.}
The one-step change grows with $N$ under $\alpha_N{=}0$ and $\alpha_N{=}1/2$, shrinks under $\alpha_N{=}3/2$, and stays within a small range under $\alpha_N{=}1$ (\cref{fig:lr-step-update}, right), and the first ten steps show the same pattern (\cref{app:lr-multistep}).
As in the forward pass, weight sharing explains this pattern.
The same weight update enters all $N$ passes and its effects add along one direction, while AdamW sets the size of the update by the learning rate alone, so the change scales as $\eta N^{1-\alpha_N}$, which is of order one only at $\alpha_N{=}1$ (\cref{clm:one-block-training-and-nontriviality}).
Together with the forward pass, this leaves $N^{-1}$ as the only loop factor under which a learning rate independent of $N$ changes a bounded residual stream by order one, as learning-rate transfer across loop counts requires.

\subsection{Learning-rate transfer across loop counts and depths}
\label{sec:lm-transfer}

\paragraph{Protocol.}
We train looped Transformers with $L\in\{12,24,48\}$ distinct layers on FineWeb-Edu~\citep{penedo2024fineweb} for $10$B tokens with loop counts $N\in\{1,2,4,8\}$, and for each $(L,N)$ we sweep learning rates around a base value tuned at $L{=}12$, $N{=}1$.
At $L{=}12$, we compare four loop factors, $1$, $N^{-1/2}$, $N^{-1}$ (\LoopmuP), and $N^{-3/2}$.
The factor $1$ leaves the branches unscaled, and the other three decay more slowly than, as fast as, and faster than the theory prescribes.
At $L{=}24$ and $L{=}48$, we use \LoopmuP and, following the prediction $\eta\propto L^{-1/2}$ (\cref{tab:loop-mup-recipe}), scale the learning rate by $m_L^{-1/2}=(L/12)^{-1/2}$ (\cref{app:completep-loop-table}).
The optimizer, schedule, and full model configuration are in \cref{tab:lm-model-config}, and we report the validation loss on the FineWeb-Edu held-out split at the end of training.

\paragraph{Transfer across loop counts.}
\emph{Only $N^{-1}$ keeps the best learning rate fixed across $N$, and it also reaches the lowest loss.}
Under $N^{-1}$, the optimal learning rate stays at $1.25\times10^{-3}$ for every $N$ at $L{=}12$ (\cref{fig:lr-transfer}, panel~(d)), and it also stays fixed across $N$ at $L{=}24$ and $L{=}48$ (panels~(e) and~(f)).
Without a loop factor, the optimum falls from $1.25\times10^{-3}$ at $N\le2$ to $7.5\times10^{-4}$ at $N{=}8$, and learning rates of $3\times10^{-3}$ and above diverge at $N\ge4$ (panel~(a)), consistent with a one-step change that grows with $N$ (\cref{sec:init-scaling}).
Under $N^{-3/2}$, the optimum rises from $1.25\times10^{-3}$ at $N{=}1$ to $3\times10^{-3}$ at $N{=}8$ (panel~(c)), so each loop count needs its own sweep, and under $N^{-1/2}$ it also moves at $N\ge4$ (panel~(b)).
The theory predicts the rise under $N^{-3/2}$, because at a fixed learning rate the one-step response scales as $\respN=N^{1-\alpha_N}=N^{-1/2}$ (\cref{sec:one-block-desiderata}), so an order-one response needs a learning rate growing as $N^{1/2}$, which reaches about $3.5\times10^{-3}$ at $N{=}8$.
Beyond transfer, $N^{-1}$ reaches a lower best loss than all three alternatives at every $N>1$, and its best loss falls the most as $N$ grows (panels~(a)--(d)).
This matches the forward pass, where $1$ and $N^{-1/2}$ leave the residual stream unbounded and $N^{-3/2}$ drives the looped model toward the identity map.

\paragraph{Transfer across depths.}
\emph{With the learning rate scaled by $m_L^{-1/2}$, the same base learning rate is also optimal at every depth.}
After the depth correction, the optimal base learning rate is the same at $L{=}12$, $24$, and $48$ (\cref{fig:lr-transfer}, panels~(d)--(f)).
Together with the fixed optimum across $N$, one base learning rate is optimal for every $(L,N)$ we train, which shows \LoopmuP's effectiveness.
\begin{figure}[t]
  \centering
  \includegraphics[width=\linewidth]{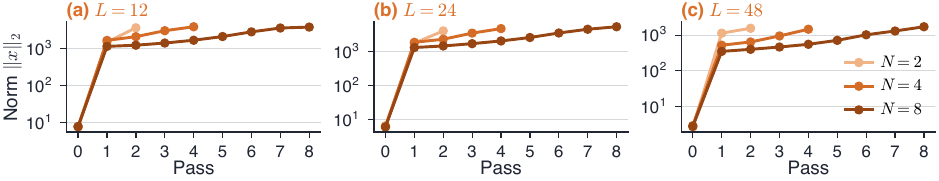}
  \caption{\textbf{Under \LoopmuP the trained residual stream ends at a similar norm for every loop count.}
  Each line traces the L2 norm of the residual stream from pass~0 (after the embedding) through the last pass, for $N\in\{2,4,8\}$ after full training.}
  \label{fig:hidden-norm}
\vspace{-0.3cm}
\end{figure}

 \subsection{Residual stream beyond initialization}
\label{sec:post-training}

\paragraph{Protocol.}
We check whether the two predictions that we verified at initialization in \cref{sec:init-scaling} still hold after full training ($20{,}000$ steps) of the $L\in\{12,24,48\}$ models from \cref{sec:lm-transfer}.
For $N{=}8$, we compute the cosine similarity between the increments of every pair of passes, and for comparison we compute the same matrix between the per-layer increments of the trained untied $L{=}48$ model (\cref{fig:update-cosine}).
For $N\in\{2,4,8\}$, we trace the norm of the residual stream after every pass, from the embedding output ($r{=}0$) through the last pass (\cref{fig:hidden-norm}).

\paragraph{Increment alignment.}
\emph{After training, the increments of different passes stay aligned.}
We find that the cosine similarities between pass increments are lower than at initialization (\cref{fig:corr}) but remain positive for every pair of passes that excludes the final pass, at all three depths (\cref{fig:update-cosine}, panels~(b)--(d)) and also at $N{=}4$ (\cref{app:cosine-small-N}).
Only the increment of the final pass is weakly or negatively correlated with the others.
In contrast, we find that the per-layer increments of the trained untied model have low similarity overall and align only with nearby layers (panel~(a)), as in the untied models with $L{=}12$ and $L{=}24$ (\cref{app:cosine-untied}).

\paragraph{Final norm.}
\emph{Under \LoopmuP, the final residual norm is consistent across loop counts.}
We find that the per-pass trajectories differ across $N$, but within each depth the norm after the last pass varies by at most a factor of $1.33$, as it stays approximately invariant at initialization (\cref{fig:hidden-norm}).
\raggedbottom
\section{Conclusion}

We study how to train looped Transformers stably as the number of loops grows.
By analyzing a residual block looped $N$ times in the infinite-width limit, we find that looping differs from adding depth, because the loops share weights and their updates to the residual stream therefore add along the same direction rather than in independent directions.
This property leads to \LoopmuP, which multiplies each residual branch by $1/N$ on top of an existing depth parameterization and keeps the learning rate independent of $N$.
In our language-modeling experiments, \LoopmuP trains looped Transformers stably, transfers the learning rate across loop counts up to $8$ and executed depths up to $384$ layers, and reaches a lower loss than loop factors that decay more slowly or faster than $1/N$.
We hope that \LoopmuP offers the looped-model community a simple recipe for training better looped models at a lower tuning cost.
We discuss the limitations of this work in \cref{app:limitations}.

\subsection*{AI use statement}

We used generative AI tools to edit the paper for readability and to write and edit parts of the plotting and analysis code.
We also used these tools to help verify the mathematical proofs.
The authors reviewed all AI-edited text, tested the AI-assisted code, and checked every AI-assisted figure against the underlying experimental data.
We take responsibility for the final content of this work, including text, code, and figures produced with the aid of generative AI.

\subsection*{Ethics statement}

This work studies how to parameterize looped Transformers and involves no human subjects or personal data.
We train only on the publicly released FineWeb-Edu corpus~\citep{penedo2024fineweb} and release no new datasets or model weights, so we do not foresee ethical concerns beyond those common to language-model research.

\subsection*{Reproducibility statement}

\cref{tab:loop-mup-recipe} gives the \LoopmuP recipe, and \cref{app:experimental-details} lists all model configurations and hyperparameters.

\section*{Acknowledgments}
The authors are supported in part by the National Natural Science Foundation of China Grant 04130200126.

\clearpage
\flushbottom
\bibliographystyle{plainnat}
\newpage
\appendix
\raggedbottom
\etocdepthtag.toc{appendix}
\begingroup
    \hypersetup{linkcolor=black}
    \etocsettagdepth{mainmatter}{none}
    \etocsettagdepth{appendix}{subsection}
    \renewcommand{\contentsname}{Appendix Contents}
    \tableofcontents
\endgroup

\crefalias{section}{appendix}
\crefalias{subsection}{appendix}
\crefname{appendix}{Appendix}{Appendices}

\numberwithin{equation}{section}
\renewcommand{\theequation}{\thesection\arabic{equation}}
\renewcommand{\theHsection}{\Alph{section}}
\renewcommand{\theHequation}{\theHsection\arabic{equation}}

\paragraph{Guide to the appendix.}
\cref{app:loop-classification} provides a rigorous statement of the derivations in \cref{sec:one-block}.
\cref{app:fixed-k,app:loop-depth} extend the theory of \cref{sec:one-block} to the models summarized in \cref{sec:loop-mup-recipe}.
\cref{app:fixed-k} allows each residual branch to contain several matrices, and \cref{app:loop-depth} stacks distinct layers and derives \LoopmuP.
\cref{app:additional-results} reports additional experimental results for \cref{sec:experiments}, \cref{app:limitations} discusses the limitations of this work, and \cref{app:experimental-details} describes the experimental setup.
\cref{app:mathematical-tools,app:fixed-k-proofs,app:loop-depth-proofs,app:nonlinear-justifications} contain the mathematical tools, the proofs of the Claims for linear networks trained with SGD, and a justification of these Claims for nonlinear networks and Adam.

\section{Formal Classification of Loopwise Parameterizations}\label{app:loop-classification}

We now formalize the scaling analysis of \cref{sec:one-block}.
We first introduce the definitions shared by a single-matrix residual branch and by residual branches of fixed finite block depth, and then state the corresponding classification claims.
The strict linear-SGD proofs are given in \cref{app:fixed-k-proofs}; extensions to nonlinear networks and entrywise optimizers are discussed in \cref{app:nonlinear-justifications}.

\subsection{Definitions}\label{app:loop-definitions}

We use the following definitions for the model with a single matrix in its residual branch (\cref{sec:one-block-setup}), and for the model with a residual branch of fixed finite block depth (\cref{app:fixed-k}).
In the latter setting, all preactivation conditions apply at every internal position.

We formulate the definitions for finite input batches.
At step $t$, we train on a batch $\Xi_t=(\xi_{t,b})_{b=1}^B$ of scalar inputs; the singleton case $B=1$ recovers the input $\xi_t$ used in the setup.
For evaluation, we fix a finite batch $\vzeta$ of scalar inputs, which need not coincide with a training batch, and stack the corresponding hidden states, preactivations, and outputs columnwise into $\mX_t^r(\vzeta)$, $\mH_t^r(\vzeta)$, and $\vf_t(\vzeta)$.
For a batch matrix, we use the RMS norm $\normrms{\mX}=n^{-1/2}\normfro{\mX}$, which reduces to the vector RMS norm for a singleton batch.
For the fixed-size output batch $\vf_t(\vzeta)$, asymptotic statements are understood coordinatewise.
We write $\Delta F_t(\vzeta)=F_{t+1}(\vzeta)-F_t(\vzeta)$ for the one-step change of any of these quantities.
Unless stated otherwise, all asymptotic notation and convergence below refer to the ordered limit $n\to\infty$ followed by $N\to\infty$; the conventions are summarized in \cref{app:mathematical-tools}.

We encode each finite training run by the following data.
\begin{definition}[Training routine]\label{def:training-routine}
    A training routine $\routine=\left(T,\eta,\set{Q_t,\Xi_t,\ell_t}_{t<T}\right)$ consists of a finite horizon $T$, a base learning rate $\eta>0$, and, at each step $t<T$, an entrywise update map $Q_t$, a training batch $\Xi_t$, and a differentiable loss $\ell_t$ including the labels.
    The routine does not depend on $n$ or $N$.
\end{definition}

We call a parameterization stable when its initialization and every one-step update remain bounded.
\begin{definition}[Stability]\label{def:one-block-stability}
    We say that a parameterization is
    \begin{enumerate}[label=\upshape(\roman*)]
        \item\label{it1:one-block-stability} 
        \textit{stable at initialization} if, for every evaluation batch $\vzeta$, $\mH_0^r(\vzeta),\mX_0^r(\vzeta)=\bigO(1)$ for $1\le r\le N$ and $\vf_0(\vzeta)=\bigO(1)$;
        \item\label{it2:one-block-stability} 
        \textit{stable during training} if, for every routine $\routine$, step $0\le t<T$, evaluation batch $\vzeta$, and $1\le r\le N$, $\Delta\mH_t^r(\vzeta),\Delta\mX_t^r(\vzeta)=\bigO(1)$ and $\Delta\vf_t(\vzeta)=\bigO(1)$.
    \end{enumerate}
    We call it \textit{stable} when both conditions hold.
\end{definition}

Nontriviality asks whether at least one training routine changes the output in the limit.
\begin{definition}[Nontriviality]\label{def:one-block-nontriviality}
    We call a parameterization \textit{trivial} if $\vf_t(\vzeta)-\vf_0(\vzeta)\to0$ for every routine $\routine$, step $1\le t\le T$, and evaluation batch $\vzeta$. 
    It is \textit{nontrivial} otherwise.
\end{definition}

We use three further properties to describe the features.
Faithfulness requires the preactivations, which are the inputs to $\phi$, to remain of order one.
To exclude the degenerate zero-input case, we call an evaluation batch nondegenerate when all of its inputs are nonzero.
\begin{definition}[Faithfulness]\label{def:one-block-faithfulness}
    A parameterization is \textit{faithful} on a nondegenerate evaluation batch $\vzeta$ if $\mH_t^r(\vzeta)=\bigTheta(1)$ for every routine $\routine$, step $0\le t\le T$, and loop $1\le r\le N$, uniformly over $r$.
\end{definition}

Feature learning asks that one training step change both the residual states and the preactivations by order one at every fixed positive loop fraction.
We represent the loop position $\floor{\lambda N}$ by its fraction $\lambda\in(0,1]$.
\begin{definition}[Feature learning]\label{def:one-block-feature-learning}
    A parameterization induces \textit{feature learning} if, on a nondegenerate evaluation batch $\vzeta$, there exist a routine $\routine$ and a step $0\le t<T$ such that both
    $\Delta\mX_t^{\floor{\lambda N}}(\vzeta)$ and
    $\Delta\mH_t^{\floor{\lambda N}}(\vzeta)$
    are $\bigTheta(1)$ for every $\lambda\in(0,1]$.
\end{definition}

Unlike in the \DepthmuP analysis~\citep{yang2024tensor}, stability and nontriviality alone determine the critical point for a single-matrix residual branch (\cref{des:one-block-stable-initialization,des:one-block-stable-training-and-nontrivial-response}).
The claims below select $(\alpha_N,\gamma_N)=(1,0)$ without using feature diversity, which we introduce here to describe the loop trajectory at this point (\cref{app:one-block-classification}).
In \cref{app:loop-depth-desiderata,app:loop-depth-classification}, feature diversity across depth also helps select the parameterization.
Feature diversity measures how rapidly the hidden states vary between nearby loop fractions.
Following \DepthmuP, we quantify this variation through the scaling of the increment between $\lambda$ and $\lambda+\varepsilon$ as $\varepsilon\downarrow0$.
\begin{definition}[Feature-diversity exponent]\label{def:one-block-feature-diversity}
    On a nondegenerate evaluation batch $\vzeta$, let $\fK$ be the set of all $\kappa_N\ge0$ such that
    \[
        \normrms*{\mX_t^{\floor{(\lambda+\varepsilon)N}}(\vzeta)-\mX_t^{\floor{\lambda N}}(\vzeta)}
        =\bigOmega(\varepsilon^{1-\kappa_N})
    \]
    for every routine $\routine$, step $0\le t\le T$, $\lambda\in[0,1)$, and sufficiently small $0<\varepsilon\le1-\lambda$, where the limit is taken first and then $\varepsilon\downarrow0$.
    When $\fK\ne\varnothing$, we define the feature-diversity exponent to be $\sup\fK$.
    If $\fK=\varnothing$, for example because the hidden state does not change across loops, we call the diversity degenerate and leave the exponent undefined.
\end{definition}

When matched by a corresponding upper bound, $\kappa_N=0$ yields the smooth-path scaling $\bigTheta(\varepsilon)$, whereas larger exponents correspond to stronger local variation.

\subsection{Classification at block depth 1}\label{app:one-block-classification}

We now formalize the scaling analysis in \cref{sec:one-block-desiderata} for a residual branch containing a single matrix.
We use the model of \cref{sec:one-block-setup} and the definitions in \cref{app:loop-definitions}.
The first two claims characterize initialization stability, training stability, and nontriviality separately.

\begin{claim}\label{clm:one-block-initialization-stability}
    A parameterization is stable at initialization if and only if $\alpha_N\ge1$.
\end{claim}

\begin{claim}\label{clm:one-block-training-and-nontriviality}
    Consider a parameterization that is stable at initialization. 
    Training stability and nontriviality are characterized separately as follows:
    \begin{enumerate}[label=\upshape(\roman*)]
        \item\label{it1:one-block-training-and-nontriviality} It is stable during training if and only if $\gamma_N\ge0$ and $\alpha_N+\gamma_N\ge1$.
        \item\label{it2:one-block-training-and-nontriviality} It is nontrivial if and only if $\alpha_N+\gamma_N\le1$.
    \end{enumerate}
    In particular, when $\alpha_N+\gamma_N>1$, no training routine changes the output in the limit.
\end{claim}

Combining initialization stability, training stability, and nontriviality, we find that $(\alpha_N,\gamma_N)=(1,0)$ is the only stable and nontrivial point.
\cref{fig:loop-phase-diagrams}(a) shows the corresponding phase classification.
The final claim describes the feature properties at this point.

\begin{claim}\label{clm:one-block-loop-mup-and-diversity}
    At the unique stable and nontrivial point $(\alpha_N,\gamma_N)=(1,0)$, the parameterization is faithful and induces feature learning.
    Its feature-diversity exponent is $\kappa_N=0$.
\end{claim}

\begin{figure}[!htbp]
    \centering
    \definecolor{phaseStable}{RGB}{216,237,231}
    \definecolor{phaseTrivial}{RGB}{251,232,199}
    \definecolor{phaseNontrivial}{RGB}{219,231,245}
    \definecolor{phaseUnknown}{RGB}{234,228,240}
    \definecolor{phaseAccent}{RGB}{173,102,0}
    \resizebox{\textwidth}{!}{%
    \begin{tikzpicture}[
        x=1.12cm, y=0.88cm, >=Latex, font=\small,
        critical/.style={star, star points=5, star point ratio=2.1,
            fill=phaseAccent, draw=black, inner sep=0pt, minimum size=8pt},
        region/.style={align=center, font=\small},
        boundary/.style={black!65, dashed, line width=0.65pt}
    ]
        \begin{scope}
            \fill[gray!13] (-0.6,-3.6) rectangle (1,1.8);
            \fill[phaseStable] (1,0) rectangle (4.6,1.8);
            \fill[phaseNontrivial] (1,0) -- (4.6,-3.6) -- (1,-3.6) -- cycle;
            \fill[phaseTrivial] (1,0) -- (4.6,0) -- (4.6,-3.6) -- cycle;
            \draw[boundary] (1,-3.6) -- (1,1.8);
            \draw[boundary] (1,0) -- (4.6,0);
            \draw[black!75, line width=0.8pt] (1,0) -- (4.6,-3.6)
                node[pos=0.57, above, sloped, font=\small, fill=white, inner sep=1.5pt]
                {$\alpha_N+\gamma_N=1$};
            \draw[->, thick] (-0.6,-3.6) -- (4.92,-3.6) node[right] {$\alpha_N$};
            \draw[->, thick] (-0.6,-3.6) -- (-0.6,2.05) node[above] {$\gamma_N$};
            \foreach \tick in {0,1,2,3,4}
                \draw (\tick,-3.6) -- (\tick,-3.69) node[below, font=\small] {$\tick$};
            \foreach \tick in {-3,-2,-1,0,1}
                \draw (-0.6,\tick) -- (-0.68,\tick) node[left, font=\small] {$\tick$};
            \node[region, rotate=90] at (0.16,-0.9) {Initialization unstable};
            \node[region] at (2.9,1.02) {Training stable\\Trivial};
            \node[region] at (3.25,-0.62) {Training unstable\\Trivial};
            \node[region] at (2.3,-2.92) {Training unstable\\Nontrivial};
            \node[critical] at (1,0) {};
            \node[anchor=south west, font=\small] at (1.06,0.09) {$(1,0)$};
            \node[anchor=north] at (2,-4.18) {(a) Block depth 1};
        \end{scope}
        \begin{scope}[xshift=7.6cm]
            \fill[gray!13] (-0.6,-3.6) rectangle (1,1.8);
            \fill[phaseStable] (1,0) rectangle (4.6,1.8);
            \fill[phaseUnknown] (1,-3.6) rectangle (4.6,0);
            \draw[boundary] (1,-3.6) -- (1,1.8);
            \draw[boundary] (1,0) -- (4.6,0);
            \draw[black!40, densely dotted, line width=0.8pt] (1,0) -- (4.6,-3.6);
            \node[font=\small, text=black!65, rotate=-40]
                at (1.96,-1.12) {$\alpha_N+\gamma_N=1$};
            \draw[->, thick] (-0.6,-3.6) -- (4.92,-3.6) node[right] {$\alpha_N$};
            \draw[->, thick] (-0.6,-3.6) -- (-0.6,2.05) node[above] {$\gamma_N$};
            \foreach \tick in {0,1,2,3,4}
                \draw (\tick,-3.6) -- (\tick,-3.69) node[below, font=\small] {$\tick$};
            \foreach \tick in {-3,-2,-1,0,1}
                \draw (-0.6,\tick) -- (-0.68,\tick) node[left, font=\small] {$\tick$};
            \node[region, rotate=90] at (0.16,-0.9) {Initialization unstable};
            \node[region] at (2.9,1.02) {Training stable\\Trivial};
            \node[region, fill=phaseUnknown, inner sep=2pt] at (2.8,-2.8)
                {Training unstable};
            \node[critical] at (1,0) {};
            \node[anchor=south west, font=\small] at (1.06,0.09) {$(1,0)$};
            \node[diamond, fill=red!65!black, draw=black, inner sep=0pt,
                minimum size=6pt] at (4,-1) {};
            \draw[->, black!65] (3.35,-0.5) -- (3.96,-0.92);
            \node[align=center, font=\small, fill=phaseUnknown, inner sep=2pt] at (3.05,-0.35)
                {\cref{app:fixed-k-counterexample}};
            \node[anchor=north] at (2,-4.18) {(b) Fixed block depth $k\ge2$};
        \end{scope}
            \node[critical] at (2.92,-5.18) {};
            \node[anchor=west] at (3.1,-5.18) {Unique stable and nontrivial point};
    \end{tikzpicture}%
    }
    \caption{
        Loopwise classification from \cref{clm:one-block-initialization-stability,clm:one-block-training-and-nontriviality,clm:fixed-k-training-and-nontriviality}.
        Both panels use the same axes and scales.
        Panel~(a) has a single matrix in the residual branch, while panel~(b) has a residual branch of fixed block depth $k\ge2$.
        In both panels, initialization is stable exactly when $\alpha_N\ge1$; within this region, training is stable exactly when $\gamma_N\ge0$.
        Thus $(1,0)$ is the only point that is both stable and nontrivial in either setting.
        In panel~(a), the line $\alpha_N+\gamma_N=1$ is the boundary of the nontrivial region.
        In panel~(b), the dotted line $\alpha_N+\gamma_N=1$ is shown only for comparison with panel~(a).
        The diamond marks the counterexample in \cref{app:fixed-k-counterexample}.
        For residual branches of block depth $k\ge2$, the nontriviality equivalence is restricted to the training-stable region, so the shaded training-unstable region is not fully classified by that criterion.
    }
    \label{fig:loop-phase-diagrams}
\end{figure}

\subsection{Classification at fixed block depth}\label{app:fixed-k}

We have so far considered a looped network whose residual branch contains a single matrix (\cref{sec:one-block}).
We now consider a residual branch of fixed finite block depth $k$, with $k$ distinct matrices shared across the same $N$ loops.
This is the setting relevant to the multi-layer branches in a Transformer (\cref{sec:loop-mup-recipe}).
We first specify the model and update rule, and then explain how the classification changes when the residual branch has depth $k$.

\paragraph{Model.}
For each loop $r$, we pass the residual state through these $k$ shared matrices in order.
We set $\vx_t^0(\zeta)=\mU\zeta$ and $\vz_t^{r,0}(\zeta)=\vx_t^{r-1}(\zeta)$.
With activations $\phi_1,\dots,\phi_k$, the internal preactivations $\vh_t^{r,j}(\zeta)$, internal features $\vz_t^{r,j}(\zeta)$, and residual states $\vx_t^r(\zeta)$ satisfy
\begin{equation}\label{eqn:fixed-k-model}
    \begin{aligned}
        \vh_t^{r,j}&=\mW_t^j\vz_t^{r,j-1}, \\
        \vz_t^{r,j}&=\MS(\phi_j(\vh_t^{r,j})), \\
        \vx_t^r&=\vx_t^{r-1}+aN^{-\alpha_N}\vz_t^{r,k},
    \end{aligned}
\end{equation}
where $1\le r\le N$ and $1\le j\le k$.
The output is $f_t(\zeta)=\mV^\transpose\vx_t^N(\zeta)$.
The $k$ matrices are initialized independently with the same Gaussian law as $\mW_0$ and independently of $\mU$ and $\mV$.
For each matrix $\mW^j$, the gradient $\mG_{j,t}=\nabla_{\mW^j}\ell_t(f_t(\xi_t))$ sums contributions from all $N$ visits.
As in \cref{eqn:one-block-optimizer-update}, we pass $\widetilde{\mG}_{j,t}=nN^{\alpha_N-1}\mG_{j,t}$ to $Q_t$.
All matrices are updated at the same time by
\begin{equation}\label{eqn:fixed-k-optimizer-update}
    \mW_{t+1}^j=\mW_t^j-\eta n^{-1}N^{-\gamma_N}Q_t(\widetilde{\mG}_{j,0},\dots,\widetilde{\mG}_{j,t}).
\end{equation}
We apply the definitions in \cref{app:loop-definitions} with every preactivation condition imposed at each internal position $1\le j\le k$.
Feature learning concerns both the residual state and all internal preactivations within the same training routine.
Nontriviality still concerns the output, while feature diversity concerns residual-state increments across loops.

Because $k$ is fixed, the additional matrix products change constants but introduce no new powers of $N$.
Thus \cref{clm:one-block-initialization-stability} remains valid: initialization is stable if and only if $\alpha_N\ge1$.
For residual branches of depth $k\ge2$, however, we characterize nontriviality only within the training-stable region, unlike in \cref{clm:one-block-training-and-nontriviality}.

\begin{manualclaim}{\getrefnumber{clm:one-block-training-and-nontriviality}$'$}\label{clm:fixed-k-training-and-nontriviality}
    Consider a parameterization that is stable at initialization.
    Training stability is characterized first, and nontriviality is then characterized within that stable-training region:
    \begin{enumerate}[label=\upshape(\roman*)]
        \item\label{it1:fixed-k-training-and-nontriviality} It is stable during training if and only if $\gamma_N\ge0$.
        \item\label{it2:fixed-k-training-and-nontriviality} If it is also stable during training, it is nontrivial if and only if $\alpha_N+\gamma_N\le1$.
    \end{enumerate}
\end{manualclaim}

Combining these three conditions gives $\alpha_N\ge1$, $\gamma_N\ge0$, and $\alpha_N+\gamma_N\le1$.
Thus the only point that is both stable and nontrivial is $(\alpha_N,\gamma_N)=(1,0)$.
\cref{fig:loop-phase-diagrams}(b) shows the corresponding phase classification.

\begin{manualclaim}{\getrefnumber{clm:one-block-loop-mup-and-diversity}$'$}\label{clm:fixed-k-loop-mup-and-feature-behavior}
    At the unique stable and nontrivial point $(\alpha_N,\gamma_N)=(1,0)$, the parameterization is faithful at every internal preactivation and induces feature learning in both the residual state and every internal preactivation.
    Its feature-diversity exponent is $\kappa_N=0$.
\end{manualclaim}

\begin{remark}\label{rmk:fixed-k-nontriviality-scope}
    The qualification in \cref{clm:fixed-k-training-and-nontriviality}~\cref{it2:fixed-k-training-and-nontriviality} is essential.
    The counterexample in \cref{app:fixed-k-counterexample} gives a $k=2$ model that is stable at initialization, has $\gamma_N<0$ and $\alpha_N+\gamma_N>1$, yet changes its output by a nonzero amount after two SGD steps.
    Thus, unlike for a single-matrix residual branch, $\alpha_N+\gamma_N\le1$ is not necessary for nontriviality outside the training-stable region.
\end{remark}

\section{\LoopmuP for Stacks of Distinct Layers}\label{app:loop-depth}

The recipe in \cref{sec:loop-mup-recipe} states \LoopmuP for a looped Transformer, which repeats a stack of $L$ distinct layers over $N$ loops, and this appendix derives it.
We extend the single-block model of \cref{sec:one-block} to such a stack. 
The setup, intuitive analysis, classification, and recipe derivation are in \cref{app:loop-depth-setup,app:loop-depth-desiderata,app:loop-depth-classification,app:loop-mup-rows}, respectively.
The linear-SGD proof and nonlinear/Adam discussion are given in \cref{app:loop-depth-proofs,app:nonlinear-heuristic-justifications}, respectively.

\subsection{Setup}\label{app:loop-depth-setup}

We now replace the single shared block by an ordered stack of $L$ distinct residual blocks and reuse the whole stack over $N$ loops.
Our goal is to determine how the branch multiplier and the learning rate should scale with both $L$ and $N$.
We call one pass through the $L$ blocks a \textit{cycle}.

\paragraph{Model.}
As in \cref{sec:one-block-setup}, we use an MLP of width $n$ with a scalar input $\zeta$ and a scalar output.
The trainable hidden matrices are $\mW_t^l\in\RR^{n\times n}$ for $1\le l\le L$; the input and readout matrices $\mU$ and $\mV$ stay fixed.
Each block applies a coordinatewise activation $\phi_l$ and adds a mean-subtracted branch output to the hidden state, scaled by the branch multiplier $c=aN^{-\alpha_N}L^{-\alpha_L}$ for a constant $a>0$.
Set $\vx_t^{1,0}(\zeta)=\mU\zeta$ and $\vx_t^{r,0}(\zeta)=\vx_t^{r-1,L}(\zeta)$ for $r\ge2$.
The network computes
\begin{equation}
    \vh_t^{r,l}(\zeta)=\mW_t^l\vx_t^{r,l-1}(\zeta),\qquad
    \vx_t^{r,l}(\zeta)=\vx_t^{r,l-1}(\zeta)+c\,\MS\bigl(\phi_l(\vh_t^{r,l}(\zeta))\bigr),\qquad
    f_t(\zeta)=\mV^\transpose \vx_t^{N,L}(\zeta),
\end{equation}
for $1\le r\le N$ and $1\le l\le L$.
The exponents $\alpha_N$ and $\alpha_L$ control how fast the branch multiplier decreases with the number of loops and with the number of blocks in a cycle.

\paragraph{Initialization and limits.}
We initialize every hidden matrix with \muP as in \cref{sec:one-block-setup}, independently across $l$.
We analyze the network in the limit of large width, many blocks, and many loops, taking $n\to\infty$ before $L\to\infty$ before $N\to\infty$.
Under this initialization, each coordinate of the hidden state and preactivation is of order one as $n\to\infty$ at any fixed pair $(r,l)$.

\paragraph{Training.}
We train only the hidden matrices and keep $\mU$ and $\mV$ fixed.
At step $t$, we draw a training input $\xi_t$ and a differentiable loss $\ell_t$, and write $\mG_t^l=\nabla_{\mW^l}\ell_t(f_t(\xi_t))$ for the gradient of block $l$ through all $N$ loops.
Following Tensor Programs~IVb and~VI~\citep{yang2023tensor,yang2024tensor} and \cref{eqn:one-block-optimizer-update}, we rescale $\mG_t^l$ by $nN^{\alpha_N-1}L^{\alpha_L}$ so that its entries stay of order one.
With a base learning rate $\eta>0$, exponents $\gamma_N$ and $\gamma_L$, and a map $Q_t$ that turns the rescaled gradients into an update direction, the update is
\begin{equation}\label{eqn:loop-depth-optimizer-update}
    \widetilde{\mG}_t^l=nN^{\alpha_N-1}L^{\alpha_L}\mG_t^l,
    \qquad
    \mW_{t+1}^l=\mW_t^l-\eta n^{-1}N^{-\gamma_N}L^{-\gamma_L}
    Q_t(\widetilde{\mG}_0^l,\dots,\widetilde{\mG}_t^l).
\end{equation}
The exponents $\gamma_N$ and $\gamma_L$ control how the learning rate depends on $N$ and $L$ after this rescaling.
The question of this appendix is which tuple $(\alpha_N,\alpha_L;\gamma_N,\gamma_L)$ to use.

\subsection{Desiderata for loop-depth hyperparameter transfer}\label{app:loop-depth-desiderata}

We extend \cref{des:one-block-stable-initialization,des:one-block-stable-training-and-nontrivial-response} of \cref{sec:one-block-desiderata} from one shared block to a stack of $L$ independent blocks reused over $N$ loops.
As there, we take the identity activation, drop mean subtraction, and study one nondegenerate SGD step.
All limits below are taken in the order $n\to\infty$, then $L\to\infty$, then $N\to\infty$, with inputs held fixed.

\paragraph{Forward pass at initialization.}
The first requirement keeps the network's representations and output bounded as $L$ and $N$ grow, which we read off from \cref{des:one-block-stable-initialization} applied to every pair $(r,l)$.

Take the input $\zeta=1$.
Within one cycle the $L$ hidden matrices are independent, so distinct first-order residual directions become orthogonal in normalized inner product as $n\to\infty$, and the first-degree contribution of a single cycle has RMS scale $\sqrt{L}\,c$.
Reusing the same stack over $N$ loops adds this contribution coherently, and the assembled linear increment is $Nc\sum_{l=1}^L\mW_0^l\mU$.
Hence
\[
    \normrms*{Nc\sum_{l=1}^L\mW_0^l\mU}^2
    \pto N^2Lc^2
    =a^2N^{2-2\alpha_N}L^{1-2\alpha_L}.
\]
Writing $\accNL=N^{1-\alpha_N}L^{1/2-\alpha_L}$, the linear residual contribution has scale $N\sqrt{L}\,c=a\accNL$.
Thus the initialized residual stream grows with $L$ when $\alpha_L<1/2$, independently of $N$.
When $\alpha_L=1/2$ it grows with $N$ if $\alpha_N<1$, stays of order one if $\alpha_N=1$, and vanishes toward the input $\mU$ if $\alpha_N>1$.
When $\alpha_L>1/2$ it vanishes as $L\to\infty$ for every $\alpha_N$.
Since the limit $L\to\infty$ is taken before $N\to\infty$, \cref{des:one-block-stable-initialization} therefore needs 
(i) $\alpha_L>\frac12$,
or 
(ii) $\alpha_L=\frac12$ and $\alpha_N\ge1$.
This keeps the distance between the hidden state and the input bounded at initialization but does not require it to remain nonzero.

\paragraph{One gradient step.}
Write $\Delta\vx_0^{r,l}(\zeta) = \vx_1^{r,l}(\zeta) - \vx_0^{r,l}(\zeta)$, and define $\Delta\vh_0^{r,l}$ and $\Delta f_0$ analogously, always at the same input.
We now work in the initialization-stable regime above.

An SGD step changes each entry of $\mW^l$ by order $\updNL/n$, where $\updNL=N^{-\gamma_N}L^{-\gamma_L}$.
The first block in the first loop receives the raw input $\mU\zeta$, giving the direct response $\Delta\vh_0^{1,1}(\zeta)=\Delta\mW_0^1\mU\zeta$.
This response has RMS scale $\updNL$.
Because we take $L\to\infty$ before $N\to\infty$, boundedness of the first-block preactivation therefore requires
(i) $\gamma_L>0$,
or
(ii) $\gamma_L=0$ and $\gamma_N\ge0$.

Each updated block is then applied $N$ times and multiplied by $c$ at every application.
The gradient of every block is aligned with the readout, so these $N\times L$ contributions add coherently in the scalar output rather than as orthogonal hidden-state directions.
The aggregate response $|\Delta f_0(\zeta)|$ therefore has scale $NLc\,\updNL=a\respNL$, where $\respNL=N^{1-\alpha_N-\gamma_N}L^{1-\alpha_L-\gamma_L}$.
The first-step output response blows up with $L$ if $1-\alpha_L-\gamma_L>0$, vanishes with $L$ if $1-\alpha_L-\gamma_L<0$, and, after the $L$-limit, blows up or vanishes with $N$ according to the sign of $1-\alpha_N-\gamma_N$.
\cref{des:one-block-stable-training-and-nontrivial-response} asks for a uniformly bounded hidden-state update and an order-one output response, so both exponents must satisfy $\gamma_L=1-\alpha_L$ and $\gamma_N=1-\alpha_N$, respectively.

Substituting $\gamma_L=1-\alpha_L$ and $\gamma_N=1-\alpha_N$ back into the local-update bound, and keeping the initialization-stable regime, yields the stable and nontrivial region (the shaded area in \cref{fig:loop-depth-phase-diagram}):
\begin{equation}\label{eqn:loop-depth-stable-and-nontrivial-region}
    \gamma_L=1-\alpha_L,
    \qquad
    \gamma_N=1-\alpha_N,
    \qquad
    \begin{cases}
        \alpha_L=\frac12,\quad \alpha_N\ge1,\\
        \frac12<\alpha_L<1,\quad \alpha_N\in\RR,\\
        \alpha_L=1,\quad \alpha_N\le1.
    \end{cases}
\end{equation}
\cref{des:one-block-stable-initialization,des:one-block-stable-training-and-nontrivial-response} therefore leave a region rather than a unique point.

To select a point inside this region, we additionally ask for the largest depthwise variation that independent block directions can support, as in \DepthmuP~\citep{yang2024tensor}.
This is the layer analogue of the order-one hidden-state change at a fixed loop fraction in \cref{sec:one-block-desiderata}.
\begin{desideratum}[Maximal feature diversity]\label{des:loop-depth-maximal-feature-diversity}
    At initialization and on a fixed nonzero input, for every fixed $\lambda\in[0,1)$ and sufficiently small $0<\varepsilon\le1-\lambda$, the rescaled hidden-state change across an $\varepsilon$-fraction of the first stack should satisfy
    $N\normrms*{\vx_0^{1,\floor{(\lambda+\varepsilon)L}} - \vx_0^{1,\floor{\lambda L}}} = \bigTheta(\sqrt{\varepsilon})$
    after the ordered network limits and then as $\varepsilon\downarrow0$.
\end{desideratum}

Across $\varepsilon L$ independent blocks the hidden-state increment has RMS scale $c\sqrt{\varepsilon L}$.
Multiplying by the outer factor $N$ in \cref{des:loop-depth-maximal-feature-diversity} therefore gives $Nc\sqrt{\varepsilon L} = a\accNL\sqrt{\varepsilon}$.
The desideratum asks for this quantity to be $\bigTheta(\sqrt{\varepsilon})$ as $\varepsilon\downarrow0$, hence $\accNL=\bigTheta(1)$.
In the ordered limits, this forces $\alpha_L=1/2$ and then $\alpha_N=1$.
The stable-and-nontrivial relations then give $\gamma_L=1/2$ and $\gamma_N=0$, and select $(\alpha_N,\alpha_L;\gamma_N,\gamma_L) = \left(1,\frac12;0,\frac12\right)$, which we call \LoopmuP.
The formal classification, feature-learning result, and feature diversity properties are stated in \cref{app:loop-depth-classification}.

\begin{figure}[!htbp]
    \centering
    \definecolor{depthPhaseStable}{RGB}{216,237,231}
    \definecolor{depthPhaseBoundary}{RGB}{0,112,93}
    \definecolor{depthPhaseAccent}{RGB}{173,102,0}
    \resizebox{\textwidth}{!}{%
    \begin{tikzpicture}[
        x=2.75cm, y=3.6cm, >=Latex, font=\small,
        formula/.style={draw=black!45, fill=depthPhaseStable!30,
            line width=0.5pt, inner xsep=5pt, inner ysep=4pt},
        included/.style={depthPhaseBoundary, line width=1.3pt},
        excluded/.style={black!60, dashed, line width=0.8pt},
        critical/.style={star, star points=5, star point ratio=2.1,
            fill=depthPhaseAccent, draw=black, inner sep=0pt, minimum size=9pt}
    ]
        \fill[depthPhaseStable] (-1.5,0.5) rectangle (3.5,1);
        \draw[black!25, dotted] (1,0) -- (1,1.34);
        \draw[excluded] (-1.5,0.5) -- (1,0.5);
        \draw[included, ->] (1,0.5) -- (3.5,0.5);
        \draw[included, <-] (-1.5,1) -- (1,1);
        \draw[excluded] (1,1) -- (3.5,1);
        \fill[depthPhaseBoundary] (1,1) circle (2.4pt);
        \node[critical] at (1,0.5) {};
        \draw[->, thick] (-1.5,0) -- (3.75,0) node[right] {$\alpha_N$};
        \draw[->, thick] (-1.5,0) -- (-1.5,1.35) node[above] {$\alpha_L$};
        \foreach \tick in {-1,0,1,2,3}
            \draw (\tick,0) -- (\tick,-0.022) node[below] {$\tick$};
        \draw (-1.5,0.5) -- (-1.555,0.5) node[left] {$\frac12$};
        \draw (-1.5,1) -- (-1.555,1) node[left] {$1$};
        \node[align=center] at (1,0.77)
            {Stable and nontrivial: $\frac12<\alpha_L<1,\quad\alpha_N\in\RR$};
        \draw[->, depthPhaseBoundary] (-1.03,0.76) -- (-1.39,0.76);
        \draw[->, depthPhaseBoundary] (3.02,0.76) -- (3.39,0.76);
        \node[text=black!70] at (0.9,1.19) {Training unstable for $\alpha_L>1$};
        \node[text=black!70, align=center] at (-0.30,0.30)
            {Initialization unstable for $\alpha_L<\frac12$};
        \node[anchor=north west] (criticalName) at (1.05,0.47)
            {\LoopmuP: $(1,\frac12)$};
        \node[formula, below=1pt of criticalName.south west, anchor=north west]
            {$(\alpha_N,\alpha_L;\gamma_N,\gamma_L)=\left(1,\frac12;0,\frac12\right)$};
        \node[anchor=south west] at (-1.5,1.02) {Included for $\alpha_N\le1$};
        \node[anchor=south east] at (3.5,0.52) {Included for $\alpha_N\ge1$};
        \node[formula] at (1,1.42) {$\gamma_N=1-\alpha_N,\qquad\gamma_L=1-\alpha_L$};
        \draw[included] (-1.03,-0.25) -- (-0.62,-0.25);
        \node[anchor=west] at (-0.56,-0.25) {Included boundary};
        \draw[excluded] (1.34,-0.25) -- (1.75,-0.25);
        \node[anchor=west] at (1.81,-0.25) {Excluded boundary};
    \end{tikzpicture}%
    }
    \caption{
        Loop-depth classification on the two-dimensional slice $\gamma_N=1-\alpha_N$, $\gamma_L=1-\alpha_L$, from
        \cref{clm:loop-depth-initialization-stability,clm:loop-depth-training-and-nontriviality,clm:loop-depth-distinguished-point}.
        The shaded strip shows the interior of the stable and nontrivial region; solid boundaries are included and dashed boundaries are excluded.
        The star marks \LoopmuP, distinguished by its nondegenerate initialization loop trajectory and attained rescaled depth-diversity exponent $\kappa_L=1/2$.
    }
    \label{fig:loop-depth-phase-diagram}
\end{figure}

\subsection{Classification of loop-depth parameterizations}\label{app:loop-depth-classification}

We now parameterize the loop-depth model by the four exponents $(\alpha_N,\alpha_L;\gamma_N,\gamma_L)$.
Because the limit $L\to\infty$ is taken before $N\to\infty$, a power $N^{u_N}L^{u_L}$ is bounded in the ordered limit exactly when $u_L<0$, or when $u_L=0$ and $u_N\le0$.
We denote this lexicographic condition by $(u_L,u_N)\preceq_{\lex}(0,0)$, comparing the depth component first and using the loop component only to break ties.
We use $\prec_{\lex}$ for the strict order and $\succeq_{\lex},\succ_{\lex}$ for the corresponding reverse orders.
We use the input-batch and training-routine conventions of \cref{app:loop-definitions}, with each routine also independent of $L$.
We first characterize initialization stability, training stability, and nontriviality.

\begin{claim}\label{clm:loop-depth-initialization-stability}
    A parameterization is stable at initialization if and only if $(1/2-\alpha_L,\,1-\alpha_N)\preceq_{\lex}(0,0)$.
\end{claim}

Equivalently, $\alpha_L>1/2$ is sufficient for every $\alpha_N$.
On the boundary $\alpha_L=1/2$, the remaining condition is $\alpha_N\ge1$, matching the single-matrix residual-branch result in \cref{clm:one-block-initialization-stability}.

\begin{claim}\label{clm:loop-depth-training-and-nontriviality}
    Consider a parameterization that is stable at initialization. 
    We characterize training stability and nontriviality separately:
    \begin{enumerate}[label=\upshape(\roman*)]
        \item\label{it1:loop-depth-training-and-nontriviality} It is stable during training if and only if both
        \[
            \begin{aligned}
                (-\gamma_L,-\gamma_N)&\preceq_{\lex}(0,0),\\
                (1-\alpha_L-\gamma_L,\,1-\alpha_N-\gamma_N)&\preceq_{\lex}(0,0).
            \end{aligned}
        \]
        \item\label{it2:loop-depth-training-and-nontriviality} It is nontrivial if and only if $(1-\alpha_L-\gamma_L,\,1-\alpha_N-\gamma_N)\succeq_{\lex}(0,0)$.
    \end{enumerate}
    In particular, when $(1-\alpha_L-\gamma_L,\,1-\alpha_N-\gamma_N)\prec_{\lex}(0,0)$, no training routine changes the output in the limit.
\end{claim}

The first lexicographic condition controls the local-update scale $\updNL$, while the second controls the aggregate-response scale $\respNL$.
Combining the two claims gives the formal stable-and-nontrivial region: $\gamma_L=1-\alpha_L$ and $\gamma_N=1-\alpha_N$, with $(\alpha_L,\alpha_N)$ lying in
\[
    \left(\frac12,1\right) \preceq_{\lex} (\alpha_L,\alpha_N) \preceq_{\lex} (1,1).
\]
The resulting region contains infinitely many points, as shown in \cref{fig:loop-depth-phase-diagram}.
We therefore need an additional criterion that distinguishes feature variation along the loop and depth directions.

\cref{def:one-block-feature-diversity} measures variation along a single loop coordinate, whereas the repeated stack involves both loop and depth coordinates.
We multiply within-cycle depth increments by the number of cycles $N$ to compare them with variation over the full loop trajectory, and apply the same exponent convention to both directions.

\begin{definition}[Loop-depth feature increments]\label{def:loop-depth-feature-increments}
    We write $\vy_t^0=\mU\zeta$ and $\vy_t^r=\vx_t^{r,L}$ for the hidden state after $r$ complete cycles.
    We also write $r_N(\rho)=1+\floor{\rho(N-1)}$ for the loop at fraction $\rho$.
    For $\rho,\lambda\in[0,1)$ and sufficiently small $\varepsilon>0$, we define the loop increment and the rescaled depth increment by
    \[
        \begin{aligned}
            D_{N,t}(\rho,\varepsilon)
            &=\normrms*{
            \vy_t^{\floor{(\rho+\varepsilon)N}}-\vy_t^{\floor{\rho N}}},\\
            D_{L,t}(\rho,\lambda,\varepsilon)
            &=N\normrms*{
            \vx_t^{r_N(\rho),\floor{(\lambda+\varepsilon)L}}
            -\vx_t^{r_N(\rho),\floor{\lambda L}}}.
        \end{aligned}
    \]
    On a nondegenerate evaluation batch, we define $\kappa_N$ and $\kappa_L$ by applying the exponent convention in \cref{def:one-block-feature-diversity} to $D_{N,t}$ and $D_{L,t}$, respectively, taking the network limits before $\varepsilon\downarrow0$.
\end{definition}

\begin{claim}[\LoopmuP]\label{clm:loop-depth-distinguished-point}
    Within the stable and nontrivial region, the only parameterization for which the initialization loop increment is nondegenerate and the rescaled depth-diversity exponent attains $\kappa_L=1/2$ is
    \[
        (\alpha_N,\alpha_L;\gamma_N,\gamma_L)
        =\left(1,\frac12;0,\frac12\right).
    \]
    We call this parameterization \textbf{\LoopmuP} (marked in \cref{fig:loop-depth-phase-diagram}).
    It is faithful and induces feature learning.
    For every SGD routine, its loop and rescaled depth increments satisfy
    \[
        \begin{aligned}
            D_{N,t}(\rho,\varepsilon)&=\bigTheta(\varepsilon),
            &\kappa_N&=0,\\
            D_{L,t}(\rho,\lambda,\varepsilon)&=\bigTheta(\sqrt{\varepsilon}),
            &\kappa_L&=\frac12.
        \end{aligned}
    \]
\end{claim}

The two directions therefore have different path regularity.
The loop trajectory has smooth $\bigTheta(\varepsilon)$ increments, whereas the $N$-rescaled within-cycle depth profile has Brownian-scale $\bigTheta(\sqrt{\varepsilon})$ increments and is rough in this increment-scaling sense.
At \LoopmuP, the factor $N$ cancels the loop factor $1/N$ in the branch multiplier; the unrescaled depth increment is $D_{L,t}/N$ and vanishes in the loop limit.

\subsection{From \LoopmuP to the Transformer recipe}\label{app:loop-mup-rows}

We now translate the exponents selected in \cref{clm:loop-depth-distinguished-point} into the Transformer recipe of \cref{tab:loop-mup-recipe}.
At $(\alpha_N,\alpha_L;\gamma_N,\gamma_L)=\left(1,\frac12;0,\frac12\right)$, the branch multiplier becomes
\[
    c=aN^{-\alpha_N}L^{-\alpha_L}
    =\frac{a}{N\sqrt{L}}.
\]
For either optimizer, let $\Delta\mW_t^l=\mW_{t+1}^l-\mW_t^l$.

\paragraph{SGD.}
For SGD, $Q_t$ returns $\widetilde{\mG}_t^l$. 
Substituting the gradient rescaling from \cref{eqn:loop-depth-optimizer-update}, we obtain
\[
    \Delta \mW_t^l 
    = -\eta n^{-1}N^{-\gamma_N}L^{-\gamma_L} \widetilde{\mG}_t^l 
    = -\eta N^{\alpha_N-1-\gamma_N}L^{\alpha_L-\gamma_L} \mG_t^l
    = -\eta \mG_t^l.
\]
Thus the learning rate on the raw SGD gradient is $\eta$.

\paragraph{Adam.}
For Adam, $Q_t$ acts entrywise on the normalized gradient history.
With fixed $0\le\beta_1,\beta_2<1$ and
$\epsilon_{\mathrm{opt}}>0$, its bias-corrected update direction is
\[
    Q_t(\widetilde{\mG}_0^l,\ldots,\widetilde{\mG}_t^l)
    =
    \frac{
        \displaystyle
        \frac{1-\beta_1}{1-\beta_1^{t+1}}
        \sum_{s=0}^{t}\beta_1^{t-s}\widetilde{\mG}_s^l
    }{
        \displaystyle
        \sqrt{
            \frac{1-\beta_2}{1-\beta_2^{t+1}}
            \sum_{s=0}^{t}\beta_2^{t-s}
            (\widetilde{\mG}_s^l)^{\odot 2}
        }
        +\epsilon_{\mathrm{opt}}
    },
\]
where the square, square root, and quotient act entrywise.
Substituting $(\gamma_N,\gamma_L)=(0,\frac12)$ into \cref{eqn:loop-depth-optimizer-update}, we obtain
\[
    \Delta\mW_t^l
    = -\eta n^{-1}N^{-\gamma_N}L^{-\gamma_L} Q_t(\widetilde{\mG}_0^l,\ldots,\widetilde{\mG}_t^l)
    = -\frac{\eta}{n\sqrt{L}} Q_t(\widetilde{\mG}_0^l,\ldots,\widetilde{\mG}_t^l).
\]
Thus Adam's \textit{effective} learning rate (the coefficient multiplying $Q_t$) is $\eta/(n\sqrt{L})$.

As in \cref{sec:loop-mup-recipe}, we write $m_n=n/n_{\mathrm{base}}$ and $m_L=L/L_{\mathrm{base}}$ for the width and depth multipliers relative to the base model.
Relative to its tuned coefficients, the branch multiplier has factor $N^{-1}m_L^{-1/2}$, while the hidden-matrix Adam effective learning rate has factor $m_n^{-1}m_L^{-1/2}$, as illustrated in \cref{tab:loop-mup-recipe}.

\section{Additional Experimental Results}
\label{app:additional-results}

This appendix collects the results that support the experiments of \cref{sec:experiments}, in the order in which the main text uses them.

\subsection{One-step response over the first ten steps}
\label{app:lr-multistep}

\Cref{sec:init-scaling} measures the change of the residual stream after the first optimizer step.
\Cref{fig:lr-step-update-appendix} extends this measurement to the first ten optimizer steps, with one panel per exponent $\alpha_N$ and one curve per step.
The pattern of the first step holds at every step.
Under \LoopmuP the update stays within a factor of about four over $N=1$--$64$ and levels off near $N=32$, while the unscaled and square-root updates keep growing with $N$.

\begin{figure}[htbp]
  \centering
  \includegraphics[width=\linewidth]{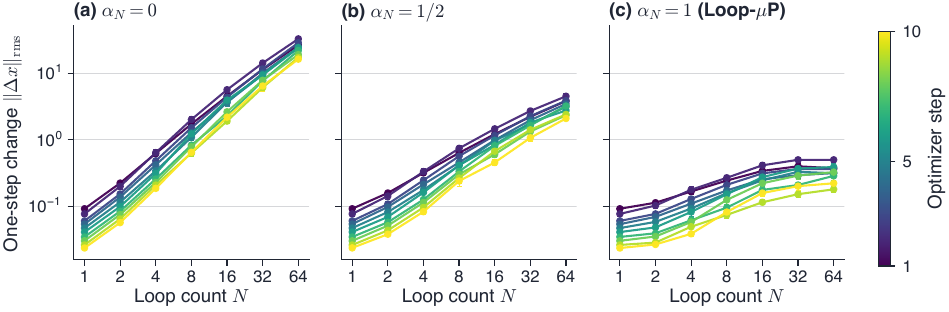}
  \caption{\textbf{The one-step pattern of \cref{fig:lr-step-update} (right) holds over the first 10 optimizer steps.}
  Each panel fixes one exponent $\alpha_N$ and plots the RMS change of the final residual stream against the loop count $N$, with one curve per optimizer step $t{=}1,\dots,10$ (viridis ramp, mean $\pm$ standard error over 10 seeds).
  All settings share the base learning rate $10^{-4}$.
  The unscaled and square-root updates grow with $N$ at every step (panels~(a) and~(b)), while the \LoopmuP update stays within a factor of about four over $N{=}1$--$64$ and levels off near $N{=}32$ (panel~(c)).}
  \label{fig:lr-step-update-appendix}
\vspace{-0.3cm}
\end{figure}

\FloatBarrier

\subsection{Post-training cosine similarity at a smaller loop count}
\label{app:cosine-small-N}

\Cref{sec:post-training} shows that the increments of different passes stay aligned after training at $N=8$.
\Cref{fig:update-cosine-appendix} repeats the measurement for the $L\in\{12,24,48\}$ models trained at $N=4$.
Adjacent increments in the first three passes have cosine similarities between $0.60$ and $0.70$ at all three depths.
The increment of the final pass is again weakly or negatively correlated with the others, most strongly at $L=12$.

\begin{figure}[htbp]
  \centering
  \includegraphics[width=\linewidth]{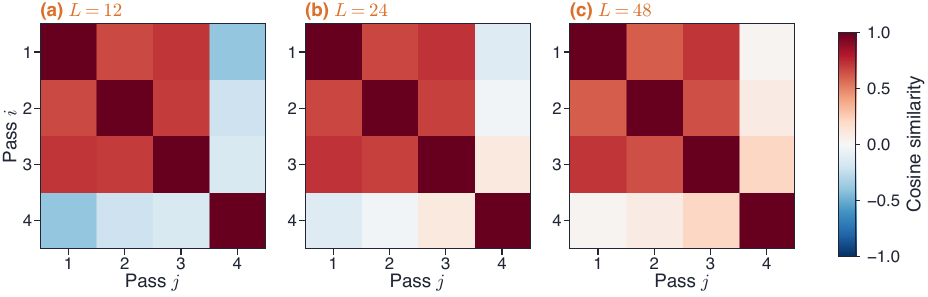}
  \caption{\textbf{Cosine similarity between pass increments after full training} ($N{=}4$).
  Companion to Figure~\ref{fig:update-cosine} ($N{=}8$).
  Early and middle passes remain positively correlated across all three depths. Pairs with the final pass are weaker and can be near zero or negative, especially at $L{=}12$, matching the pattern in Figure~\ref{fig:update-cosine}.}
  \label{fig:update-cosine-appendix}
\vspace{-0.3cm}
\end{figure}

\FloatBarrier

\subsection{Post-training layer increments of shallower untied models}
\label{app:cosine-untied}

\Cref{sec:post-training} uses the per-layer increments of a trained untied $L=48$ model as a control for the pass increments of looped models.
\Cref{fig:update-cosine-untied-appendix} repeats this measurement for the trained untied $L=12$ and $L=24$ models, plotting the cosine similarity between every pair of per-layer increments, averaged over many sequences.
The pattern matches $L=48$.
Early layers are positively correlated with one another, but the similarity fades with layer distance and turns negative for the last layers.

\begin{figure}[htbp]
  \centering
  \includegraphics[width=\linewidth]{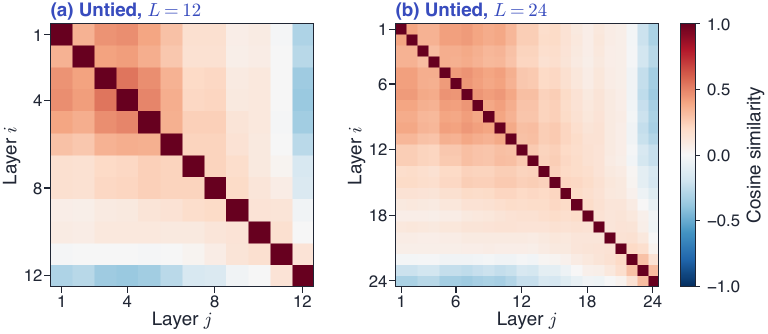}
  \caption{\textbf{Layer increments of trained untied models align mainly with nearby layers.}
  Cosine similarity between the per-layer residual-stream increments of trained untied models with $L{=}12$ (a) and $L{=}24$ (b), after full training.
  Companion to the untied $L{=}48$ model in panel~(a) of \cref{fig:update-cosine}.}
  \label{fig:update-cosine-untied-appendix}
\vspace{-0.3cm}
\end{figure}

\FloatBarrier

\section{Limitations}
\label{app:limitations}

Following \citet{yang2024tensor}, our theory analyzes a mean-centered MLP residual block and proves its results for linear networks, and its extension to nonlinear activations and Adam relies on the assumptions of Appendix~\ref{app:nonlinear-justifications}.
Future work could remove these assumptions and extend the analysis to attention and normalization layers, which would bring the theory closer to the Transformers used in practice.
Limited by our compute budget, our largest model has $438$M parameters and is trained on $10$B tokens, and we leave the study of \LoopmuP at larger scales to future work.

\section{Experimental Setup and Implementation Details}
\label{app:experimental-details}

This appendix lists the configurations behind the experiments of \cref{sec:experiments} and the instance of \cref{tab:loop-mup-recipe} that they use.

\subsection{Language modeling experiment configuration}
\label{app:model-config}

\cref{tab:lm-model-config} lists the model configurations and training hyperparameters for the language-modeling experiments of \cref{sec:lm-transfer}.

\begin{table*}[h]
    \caption{
        \textbf{Model configurations for the language-modeling experiments.}
        ``KV'' denotes key/value attention heads (each row uses 12 attention heads and 12 key/value heads, i.e.\ no GQA grouping).
        All models are decoder-only Llama-style pre-norm Transformers with RMSNorm and the Llama~3 tokenizer vocabulary~\citep{vaswani2017attention,touvron2023llama,dubey2024llama,xiong2020layer,zhang2019rmsnorm}.
        The token embedding and tied LM head are outside the loop~\citep{press2017using}, while the whole Transformer stack is repeated $N$ times.
        Training uses FineWeb-Edu for 20{,}000 optimizer steps (10B tokens, 0.5M-token global batch), AdamW~\citep{loshchilov2019decoupled} with $(\beta_1,\beta_2)=(0.9,0.95)$ and weight decay $\omega_0$ (see \cref{app:hparam-values}), and a warmup-stable-linear-decay schedule~\citep{wen2024understanding} with 500 warmup steps and a final 1{,}000-step decay.
    }
    \label{tab:lm-model-config}
    \centering
    \footnotesize
    \setlength{\tabcolsep}{5pt}
    \begin{tabular}{@{}cccccccc@{}}
        \toprule
        Layers $L$ & Loop counts $N$ & $d_{\mathrm{model}}$ & Attn/KV heads & Head dim & MLP dim & Vocab & Params (tied) \\
        \midrule
        12 & $\{1,2,4,8\}$ & 768 & 12/12 & 64 & 2048 & 128256 & 183.5M \\
        24 & $\{1,2,4,8\}$ & 768 & 12/12 & 64 & 2048 & 128256 & 268.4M \\
        48 & $\{1,2,4,8\}$ & 768 & 12/12 & 64 & 2048 & 128256 & 438.3M \\
        \bottomrule
    \end{tabular}
\end{table*}

\subsection{\LoopmuP parameterization used in the experiments}
\label{app:completep-loop-table}

\Cref{tab:completep-loop-param} instantiates \cref{tab:loop-mup-recipe} for our experiments.
We keep the width fixed ($m_n=1$) and use $L_{\mathrm{base}}=12$, so the depth multiplier is $m_L=L/12$.
We write $\eta_0$ for the base learning rate at $L=12$, $\sigma_0^2$ for the base initialization variance, $\omega_0$ for the base AdamW weight decay, and $\epsilon_{\mathrm{opt}}$ for the base AdamW numerical epsilon.

\begin{table*}[h]
    \caption{
        \textbf{\LoopmuP as used in our experiments.}
        The loop factor $N^{-1}$ appears only in the attention and MLP residual branches, while the factor $m_L^{-1/2}$ with $m_L=L/12$ scales both these branches and the learning rates of the Transformer layers.
        Learning rates depend on $L$ but not on $N$.
    }
    \label{tab:completep-loop-param}
    \centering
    \footnotesize
    \setlength{\tabcolsep}{10pt}
    \begin{tabular}{@{}>{\raggedright\arraybackslash}p{0.40\textwidth}>{\raggedright\arraybackslash}p{0.54\textwidth}@{}}
        \toprule
        Quantity & Scaling rule \\
        \midrule
        Token embedding init.\ variance & $\sigma_0^2$ \\
        Token embedding LR & $\eta_0$ \\
        \midrule
        Pre-LN init.\ variance & $\sigma_0^2$ \\
        Pre-LN LR & $\eta_0\,m_L^{-1/2}$ \\
        \midrule
        Transformer hidden init.\ variance & $\sigma_0^2$ \\
        Transformer hidden LR & $\eta_0\,m_L^{-1/2}$ \\
        Transformer hidden bias LR & $\eta_0\,m_L^{-1/2}$ \\
        Transformer hidden AdamW weight decay & $\omega_0$ \\
        \midrule
        Attention residual branch & $\mX+a\,N^{-1}m_L^{-1/2}\cdot\mathrm{MHA}(\mathrm{LN}(\mX))$ \\
        MLP residual branch & $\mZ+a\,N^{-1}m_L^{-1/2}\cdot\mathrm{MLP}(\mathrm{LN}(\mZ))$ \\
        \midrule
        Final-LN init.\ variance & $\sigma_0^2$ \\
        Final-LN LR & $\eta_0$ \\
        \midrule
        Unembedding init.\ variance & $\sigma_0^2$ \\
        Unembedding LR & $\eta_0$ \\
        Unembedding forward map & $\mX\mW_{\mathrm{unemb}}^{\transpose}$ \\
        \midrule
        AdamW $\epsilon_{\mathrm{opt}}$ for residual blocks & $\epsilon_{\mathrm{opt}}\,m_L^{-1/2}$ \\
        AdamW $\epsilon_{\mathrm{opt}}$ for embedding and unembedding & $\epsilon_{\mathrm{opt}}$ \\
        \bottomrule
    \end{tabular}
\end{table*}

\subsection{Hyperparameter values and sweep configuration}
\label{app:hparam-values}

\cref{tab:repro} lists the numerical values used to instantiate the parameterization of \cref{app:completep-loop-table} for the language-modeling experiments (\cref{sec:lm-transfer}) and the learning-rate sweep behind \cref{fig:lr-transfer}.

\begin{table*}[h]
    \caption{
        Numerical values used in the language-modeling experiments. Symbols refer to the parameterization in \cref{tab:completep-loop-param}. Initialization uses a truncated normal with standard deviation $\sigma_0$.
    }
    \label{tab:repro}
    \centering
    \footnotesize
    \setlength{\tabcolsep}{10pt}
    \begin{tabular}{@{}lcl@{}}
        \toprule
        Quantity & Symbol & Value \\
        \midrule
        Residual branch constant & $a$ & $1$ \\
        Base learning rate at $L{=}12$ & $\eta_0$ & $1.25\times10^{-3}$ \\
        Initialization standard deviation & $\sigma_0$ & $0.02$ \\
        AdamW weight decay & $\omega_0$ & $0.1$ \\
        AdamW numerical epsilon & $\epsilon_{\mathrm{opt}}$ & $10^{-8}$ \\
        \midrule
        LR sweep grid (base $\eta_0$) & --- & $\{5,\,7.5,\,10,\,12.5,\,15,\,20,\,30,\,40\}\!\times\!10^{-4}$ \\
        Divergence criterion & --- & final validation loss $>4$ \\
        Random seed & --- & $0$ (single run) \\
        Validation split & --- & FineWeb-Edu held-out \\
        \bottomrule
    \end{tabular}
\end{table*}

\FloatBarrier

\section{Notation and Mathematical Tools}\label{app:mathematical-tools}

We write $\norm{\cdot}$ for the Euclidean or Hilbert norm and use $\normrms{\vz}=n^{-1/2}\norm{\vz}$ and $\normrms{\mX}=n^{-1/2}\normfro{\mX}$.
Subscripts $\mathrm{op}$, $\mathrm{F}$, and $\mathrm{HS}$ indicate operator, Frobenius, and Hilbert--Schmidt norms.
We use $\bigO(\cdot)$, $\littleo(\cdot)$, $\bigOmega(\cdot)$, $\littleomega(\cdot)$, and $\bigTheta(\cdot)$ for standard asymptotic notation.
For random variables, these notations are understood in the sense of weak convergence (convergence in distribution). When applied to vectors or matrices, they are understood coordinatewise.

\begin{lemma}[Isserlis's Theorem \citep{isserlis1918formula}]
    \label{lem:aux-isserlis}
    Let $\vx=(x_1,\dots,x_m)$ be a centered Gaussian random vector. If $m$ is odd then $\EE[x_1\cdots x_m]=0$.
    If $m=2\ell$ then
    \[
        \EE[x_1\cdots x_{2\ell}]
        =\sum_{\pi\in\fP_2(2\ell)}\prod_{\set{a,b}\in\pi}\EE[x_ax_b],
    \]
    where $\fP_2(2\ell)$ is the set of pairings of $\set{1,\dots,2\ell}$.
\end{lemma}

We use standard properties of Haar rotations \citep{diaconis1987dozen} to recover coordinate laws from normalized Gram matrices.
We write
\[
    \mathrm{O}(n)=\set{\mQ\in\RR^{n\times n}:\mQ^\transpose\mQ=\mI_n}
\]
for the orthogonal group. Its Haar distribution is the uniform probability law, characterized by invariance under left multiplication:
\[
    \Law(\mR\mQ)=\Law(\mQ)
    \qquad\text{for every fixed }\mR\in\mathrm{O}(n).
\]
We write $\mQ\sim\mathrm{Haar}(\mathrm{O}(n))$ for this law and call left multiplication by an independent such matrix a Haar rotation.

For $1\le k\le n$, an orthonormal $k$-frame is a matrix $\mE\in\RR^{n\times k}$ with $\mE^\transpose\mE=\mI_k$. A Haar $k$-frame consists of the first $k$ columns of a Haar orthogonal matrix: it is a uniformly distributed frame, whose law is invariant under left multiplication by every fixed orthogonal matrix.

\begin{lemma}[Haar rotations]
    \label{lem:aux-haar}
    Let $k,r\ge1$ be fixed and let $\mQ\sim\mathrm{Haar}(\mathrm{O}(n))$, with $n\ge\max(k,r)$.
    \begin{enumerate}[label=\upshape(\roman*)]
        \item\label{it1:aux-haar} Let $\mY\in\RR^{n\times k}$ have i.i.d. standard normal entries.
        Then $\mE=\mY(\mY^\transpose\mY)^{-1/2}$ is a Haar $k$-frame: a uniformly distributed $n\times k$ matrix with orthonormal columns.
        As $n\to\infty$, the first $r$ rows of $\sqrt{n}\,\mE$ converge in distribution to an $r\times k$ matrix of i.i.d. standard normals.
        \item\label{it2:aux-haar} For each $n$, let $\mZ\in\RR^{n\times k}$ be deterministic and set $\mSigma=\mZ^\transpose\mZ/n$.
        As $n\to\infty$, the law of the first $r$ rows of $\mQ\mZ$ approaches that of $r$ independent $\fN(0,\mSigma)$ rows in bounded-Lipschitz distance, uniformly over $\normop{\mSigma}\le M$ for every fixed $M>0$, including singular $\mSigma$.
        If $\mSigma\to\mSigma_\infty$, these rows converge in distribution to $r$ independent $\fN(0,\mSigma_\infty)$ rows.
        \item\label{it3:aux-haar} For every $\vz\in\RR^n$ and every coordinate $i$,
        \[
            \EE_{\mQ}\left[|(\mQ \vz)_i|^2 \right]=\normrms*{\vz}^2,\qquad
            \EE_{\mQ}\left[|(\mQ \vz)_i|^4 \right]=\frac{3n}{n+2}\,\normrms*{\vz}^4.
        \]
    \end{enumerate}
\end{lemma}

Here bounded-Lipschitz distance compares probability laws by the supremum of absolute expectation differences over functions bounded in absolute value by one and $1$-Lipschitz with respect to the Frobenius norm.

For a family $\set{Z_\theta:\theta\in\fI}$, uniform tightness together with
\begin{equation}\label{eqn:coordinate-anti-concentration}
    \lim_{\varepsilon\downarrow0}
    \sup_{\theta\in\fI}\PP\left(|Z_\theta|\le\varepsilon\right)=0
\end{equation}
gives the uniform weak $\bigTheta(1)$ scale.

Let $\mathscr{G}$ be a sigma-algebra. If $Z\mid\mathscr{G}\sim\fN(0,\sigma^2)$ and $\sigma^2\ge m>0$ almost surely, then, for every $\varepsilon>0$,
\begin{equation}\label{eqn:conditional-gaussian-anti-concentration}
    \PP\left(|Z|\le\varepsilon\given\mathscr{G}\right)
    \le\sqrt{\frac{2}{\pi}}\,\frac{\varepsilon}{\sqrt{m}}
    \quad\text{almost surely}.
\end{equation}
For conditionally centered Gaussian variables $Z_\theta$ with conditional variances $\sigma_\theta^2$ and any $m,\varepsilon>0$, this gives
\[
    \sup_{\theta\in\fI}\PP(|Z_\theta|\le\varepsilon)
    \le\sup_{\theta\in\fI}\PP(\sigma_\theta^2<m)
    +\sqrt{\frac{2}{\pi}}\frac{\varepsilon}{\sqrt{m}}.
\]
Hence $\lim_{m\downarrow0}\sup_{\theta\in\fI}\PP(\sigma_\theta^2<m)=0$ implies \cref{eqn:coordinate-anti-concentration}.
For a single variable, $\sigma^2>0$ almost surely suffices.

\section{Proofs for the Infinite-Loop Limit in the Linear Case}\label{app:fixed-k-proofs}

We prove the claims in \cref{app:one-block-classification,app:fixed-k} for linear networks trained with SGD.
We first treat residual branches of fixed finite block depth in \cref{app:fixed-k-dynamics,app:fixed-k-classification}, then sharpen the nontriviality classification for a single-matrix residual branch in \cref{app:one-block-linear-classification}.
The exact first-step response in \cref{pro:fixed-k-canonical-response} establishes feature learning, while \cref{lem:fixed-k-prefix-noncancellation} gives the all-routine lower bounds for faithfulness and loop diversity.
The counterexample in \cref{app:fixed-k-counterexample} shows why the sharper nontriviality classification cannot hold for all fixed block depths.

\begin{definition}[Linear model at fixed block depth]\label{def:fixed-k-linear-setup}
    We consider a linear residual branch of fixed depth $k\ge1$, with $k$ trainable matrices shared across $N$ loops.
    The branch constant $a>0$, base learning rate $\eta>0$, and finite training horizon $T$ are fixed, while $\alpha_N,\gamma_N\in\RR$.
    We take $n\to\infty$ at fixed $N$, followed by $N\to\infty$, with $k$ fixed.
    \begin{itemize}[leftmargin=*]
        \item \textbf{Model.}
        We use identity activations without mean subtraction. The input and readout matrices $\mU,\mV\in\RR^{n\times1}$ are frozen, while $\mW_t^j\in\RR^{n\times n}$, $1\le j\le k$, are trainable.
        All initial matrices are mutually independent, with $\mU_i\iid\fN(0,1)$, $\mV_i\iid\fN(0,n^{-2})$, and $(\mW_0^j)_{i i'}\iid\fN(0,n^{-1})$.
        For a scalar input $\zeta$, the residual states, internal preactivations, and output satisfy
        \[
            \begin{aligned}
                \vx_t^0(\zeta)&=\mU\zeta,\\
                \vh_t^{r,j}(\zeta)&=\mW_t^j\cdots\mW_t^1\vx_t^{r-1}(\zeta),\qquad 1\le j\le k,\\
                \vx_t^r(\zeta)&=\vx_t^{r-1}(\zeta)+c_N\vh_t^{r,k}(\zeta),\qquad 1\le r\le N,\\
                f_t(\zeta)&=\mV^\transpose\vx_t^N(\zeta).
            \end{aligned}
        \]
        \item \textbf{Training.}
        For a training batch $\Xi_t$ with loss $\ell_t$, each matrix uses its complete gradient through all $N$ loops:
        \[
            \mG_{j,t}=\nabla_{\mW^j}\ell_t(\vf_t(\Xi_t)),
            \qquad
            \widetilde{\mG}_{j,t}=nN^{\alpha_N-1}\mG_{j,t}.
        \]
        All hidden matrices are updated simultaneously by
        \[
            \mW_{t+1}^j
            =\mW_t^j-\eta n^{-1}N^{-\gamma_N}\widetilde{\mG}_{j,t}
            =\mW_t^j-\eta N^{\alpha_N-1-\gamma_N}\mG_{j,t}.
        \]
        \item \textbf{Scales.}
        We write
        \[
            c_N=aN^{-\alpha_N},\qquad
            \accN=N^{1-\alpha_N},\qquad
            \updN=N^{-\gamma_N},\qquad
            \respN=\accN\updN.
        \]
        Thus $Nc_N=a\accN$, where $c_N$ is the branch multiplier.
        We use $\accN$ for initialization accumulation, $\updN$ for the local parameter-update scale, and $\respN$ for the aggregate state/output response scale.
        \item \textbf{Conventions.}
        The routine data and all finite batch sizes are independent of network size.
        We write $\Xi_t$ for a training batch, $\vzeta$ for an evaluation batch, and $\vf_t(\Xi_t)$ for the corresponding training outputs.
        For a singleton batch, we write $\Xi_t=(\xi_t)$ and $\chi_t=\ell_t'(f_t(\xi_t))$.
        Changes compare the same input at the same network position before and after a training step.
    \end{itemize}
\end{definition}

\subsection{SGD training dynamics}\label{app:fixed-k-dynamics}

We set $\vu=\mU$ and $\vv=n\mV$ and use the matrix products below to express the finite computation.
At finite width, we write $\mR_{0,t}=\mI$, $\mR_{j,t}=\mW_t^j\cdots\mW_t^1$, $\mS_{j,t}=\mR_{j-1,t}$, $\mT_{j,t}=\mW_t^k\cdots\mW_t^{j+1}$, $\mP_t=\mR_{k,t}$, and $\mA_t=\mI+c_N\mP_t$, with empty products interpreted as identities.
By linearity in the input, we encode a finite training batch $\Xi_t=(\xi_{t,b})_{b=1}^B$ by
\[
    d_t=\eta\sum_{b=1}^B\xi_{t,b}\,\partial_b\ell_t,
\]
where the derivatives are evaluated at the batch outputs; for a singleton batch, $d_t=\eta\chi_t\xi_t$.
The finite-step upper bounds use bounded inputs and loss-derivative controls. We first work on compact control sets and then transfer the estimates to tight realized controls by truncation and subsequences.

\begin{lemma}\label{lem:fixed-k-identities-and-bounds}
    For every $n,N,t$ and scalar input $\zeta$, the state identity holds for $0\le r\le N$, and the preactivation identity for $1\le r\le N$ and $1\le j\le k$:
    \begin{equation}\label{eqn:fixed-k-identities-and-bounds-network-state-identities}
        \begin{aligned}
            \vx_t^r(\zeta)&=\zeta \mA_t^r \vu,
            \qquad \vh_t^{r,j}(\zeta)=\zeta \mR_{j,t}\mA_t^{r-1}\vu, \\
            f_t(\zeta)&=\zeta\,n^{-1}\vv^\transpose \mA_t^N \vu.
        \end{aligned}
    \end{equation}
    For $1\le j\le k$ and integer $0\le s<N$, define
    \begin{equation}
        \vz_{j,t,s}\coloneqq \mS_{j,t}\mA_t^s \vu,
        \qquad
        \vb_{j,t,s}\coloneqq \mT_{j,t}^\transpose(\mA_t^\transpose)^{N-1-s}\vv.
    \end{equation}
    Here $\vz_{j,t,s}$ is the input to matrix $j$ in loop $s+1$ at input one: a residual state for $j=1$ and the preceding internal feature for $j>1$.
    In this linear specialization without mean subtraction, $\vz_{j,t,s}=\vz_t^{s+1,j-1}(1)$ in the notation of \cref{eqn:fixed-k-model}.
    Then the exact shared-gradient identities are
    \begin{equation}\label{eqn:fixed-k-identities-and-bounds-matrix-gradient-and-update-identities}
        \begin{aligned}
            \mG_{j,t}&=\frac{c_Nd_t}{\eta n}\sum_{s=0}^{N-1}\vb_{j,t,s}\vz_{j,t,s}^\transpose, \\
            \mW_{t+1}^j-\mW_t^j &=-\frac{ad_tN^{-\gamma_N}}{N} \sum_{s=0}^{N-1}\frac{\vb_{j,t,s}\vz_{j,t,s}^\transpose}{n}.
        \end{aligned}
    \end{equation}
    For $\alpha_N\ge1$, $\gamma_N\ge0$, each fixed number of training steps, and controls in a compact set, consider the initialization event
    \[
        \max_{1\le j\le k}\normop{\mW_0^j}\le M_0,
        \qquad \max\set{\normrms{\vu},\normrms{\vv}}\le D,
    \]
    with fixed $M_0,D$. On this event, all trained matrix, prefix, suffix, and residual-power operator norms, and all associated normalized vector norms, have bounds independent of $n,N$, uniformly over the allowed loop and matrix indices. In particular,
    \begin{equation}\label{eqn:fixed-k-identities-and-bounds-training-change-bounds}
        \begin{aligned}
            \normop*{\mW_{t+1}^j-\mW_t^j}&\le CN^{-\gamma_N}, \\
            \normop*{\mP_{t+1}-\mP_t}+
            \normop*{\mR_{j,t+1}-\mR_{j,t}}&\le CN^{-\gamma_N}, \\
            \normrms*{\Delta \vx_t^r(\zeta)}+|\Delta f_t(\zeta)|&\le CN^{1-\alpha_N-\gamma_N}, \\
            \normrms*{\Delta \vh_t^{r,j}(\zeta)}&\le CN^{-\gamma_N}.
        \end{aligned}
    \end{equation}
    Constants may depend on the fixed $k,a,t,\eta,M_0,D$, the compact control set, and input bounds, but not on $n,N,r$.
\end{lemma}

\begin{proof}
    We obtain \cref{eqn:fixed-k-identities-and-bounds-network-state-identities} by iterating the residual recurrence. At input one, a perturbation $\delta\mW^j$ changes the residual addition in loop $s+1$ by
    \[
        c_N\mT_{j,t}\,\delta \mW^j\mS_{j,t}\mA_t^s \vu.
    \]
    We propagate this addition through the remaining $N-1-s$ loops and contract with the readout $\vv/n$. Linearity in the batch inputs gives $\sum_b\xi_{t,b}\partial_b\ell_t=d_t/\eta$, so
    \[
        \delta\ell_t=\frac{c_Nd_t}{\eta n}\sum_{s=0}^{N-1}
        \vb_{j,t,s}^\transpose\,\delta\mW^j\,\vz_{j,t,s}.
    \]
    Reading off the matrix gradient gives the first identity in \cref{eqn:fixed-k-identities-and-bounds-matrix-gradient-and-update-identities}; substituting it into the SGD rule gives the second.
    For the bounds, suppose every matrix norm at time $t$ is at most $M_t\ge1$, the initial normalized vector norms are at most $D$, and $Nc_N\le a$. Then
    \[
        \begin{aligned}
            \normop*{\mA_t^r}&\le e^{aM_t^k}, \\
            \normrms*{\vz_{j,t,s}}&\le DM_t^{j-1}e^{aM_t^k}, \\
            \normrms*{\vb_{j,t,s}}&\le DM_t^{k-j}e^{aM_t^k}.
        \end{aligned}
    \]
    Using $\normop{\vp \vq^\transpose/n}=\normrms{\vp}\normrms{\vq}$ in \cref{eqn:fixed-k-identities-and-bounds-matrix-gradient-and-update-identities} gives
    \[
        M_{t+1}\le M_t+a|d_t|D^2M_t^{k-1}e^{2aM_t^k}.
    \]
    This finite induction proves the matrix bounds and the first line of \cref{eqn:fixed-k-identities-and-bounds-training-change-bounds}; telescoping the fixed $k$ matrices proves its second line. Finally,
    \begin{equation}\label{eqn:fixed-k-identities-and-bounds-residual-power-telescoping}
        \mA_{t+1}^r-\mA_t^r
        =\sum_{s=0}^{r-1}\mA_{t+1}^{r-1-s}c_N(\mP_{t+1}-\mP_t)\mA_t^s
    \end{equation}
    has uniformly bounded residual-power factors. Since $\normop{\mP_{t+1}-\mP_t}\le C\updN$,
    \[
        \normop{\mA_{t+1}^r-\mA_t^r}
        \le Cr c_N\updN\le C\accN\updN=C\respN.
    \]
    Multiplication by $\zeta\vu$ gives the residual-state bound, and pairing with $\vv/n$ gives the output bound. For the internal preactivation, we use
    \[
        \mR_{j,t+1}\mA_{t+1}^{r-1}-\mR_{j,t}\mA_t^{r-1}
        =(\mR_{j,t+1}-\mR_{j,t})\mA_{t+1}^{r-1}
        +\mR_{j,t}(\mA_{t+1}^{r-1}-\mA_t^{r-1})
    \]
    The first term has operator norm at most $C\updN$, and the second at most $C\respN$. Since $\accN\le1$, their sum is bounded by $C\updN$, proving the internal-preactivation change bound.
    A $1/4$-net of the unit sphere has at most $9^n$ points, and
    $\normop{\mW_0^j}\le2\max_{\vx,\vy\text{ in the net}}|\vy^\transpose\mW_0^j\vx|$.
    Each scalar in this maximum is $\fN(0,n^{-1})$, so a union bound gives
    \begin{equation}\label{eqn:fixed-k-initialization-operator-tail-bound}
        \PP\left(\max_{1\le j\le k}\normop*{\mW_0^j}>R\right)
        \le 2k\exp(2n\log9-nR^2/8).
    \end{equation}
    For fixed $k$ this gives a bounded initialization event with probability tending to one and bounds all fixed moments of the operator norms. Gaussian concentration likewise gives $\normrms{\vu},\normrms{\vv}\pto1$ and their fixed-moment bounds.
\end{proof}

\paragraph{Widthwise limit.}
We now take the width limit while retaining the correlations caused by repeated use of the same matrices.
The shared gradient in \cref{eqn:fixed-k-identities-and-bounds-matrix-gradient-and-update-identities} consists of normalized rank-one maps whose action depends on these correlations.
We represent the vectors in a fixed Hilbert space, where normalized inner products become ordinary inner products and the rank-one update algebra is preserved.

We realize this representation in full Fock space: its word basis records uncontracted matrix applications, while its creation and annihilation operators reproduce the Gaussian pairings that survive in the width limit.
For a real Hilbert space $H$, we use two orthogonal copies labeled $+$ and $-$, writing $(f,+)=(f,0)$ and $(f,-)=(0,f)$ in $H\oplus H$. Their full Fock space is
\[
    \fF(H\oplus H)=\bigoplus_{m\ge0}(H\oplus H)^{\otimes m},
\]
where $\bigoplus$ is the Hilbert direct sum: different degrees are orthogonal, and the sum of squared component norms is finite. The degree-zero summand is $\RR$, whose unit vector is the empty word, also called the vacuum. For $m\ge1$, a word of length $m$ is an ordered tensor in the degree-$m$ summand:
\[
    (f_1,\epsilon_1)\otimes\cdots\otimes(f_m,\epsilon_m),
    \qquad f_i\in H,\quad\epsilon_i\in\set{+,-}.
\]
Taking the letters from an orthonormal basis gives an orthonormal word basis.
The left \textit{creation} operator $\cre{f,\pm}$ prepends the signed letter $(f,\pm)$. Its adjoint, the \textit{annihilation} operator $\ann{f,\pm}$, removes the first letter, multiplies the remaining word by that letter's inner product with $(f,\pm)$, and sends the vacuum to zero.
We take two orthogonal copies $\fH_u,\fH_v$ of this entire space for the input and readout sources, with empty-word unit vectors $u_\star,v_\star$. Each source copy contains both signs; creation and annihilation act identically on the two copies. We define the circular operator
\begin{equation}\label{eqn:fixed-k-circular-system-operator-definition}
    \opC(f)=\cre{f,+}+\ann{f,-},\qquad f\in H.
\end{equation}
For this model, $H=\RR^k$ has orthonormal basis $e_1,\dots,e_k$, and
\begin{equation}\label{eqn:fixed-k-width-limit-colored-circular-system}
    \opC_j=\opC(e_j),\qquad \normop*{\opC_j}\le2.
\end{equation}
The two source copies have orthonormal word bases $e_{w,u},e_{w,v}$. Here $(j,\pm)$ denotes the signed basis letter $(e_j,\pm)$, and $e_{\varnothing,u}=u_\star$, $e_{\varnothing,v}=v_\star$.
For $1\le i,j\le k$, $\epsilon,\epsilon'\in\set{+,-}$, and $\rho\in\set{u,v}$,
\begin{equation}\label{eqn:fixed-k-circular-system-signed-actions}
    \opC_j^*=\cre{e_j,-}+\ann{e_j,+},
    \qquad
    \ann{e_j,\epsilon}e_{(i,\epsilon')w,\rho}
    =\1_{i=j}\1_{\epsilon=\epsilon'}e_{w,\rho}.
\end{equation}
Every annihilation sends the empty word to zero. These operations preserve the source index $\rho$.
For $k=1$, we write $\opC=\opC_1$; then $\opC^m u_\star=e_{(1,+)^m,u}$ and $(\opC^*)^m v_\star=e_{(1,-)^m,v}$.
The Hilbert rank-one map $(z\otimes w)y=z\langle w,y\rangle$ corresponds to $\vz\vw^\transpose/n$ at finite width.

\begin{lemma}[Width limit of the shared computation]\label{lem:fixed-k-width-limit}
    Fix $k$ and let $n\to\infty$ under \cref{def:fixed-k-linear-setup}. A mixed word is a finite ordered product of matrix symbols and their transposes, with the empty product interpreted as the identity. We take two such words $P,Q$ whose lengths and matrix indices are independent of $n$. We evaluate $P(\mW_0),Q(\mW_0)$ at the initial matrices and form $P(\opC),Q(\opC)$ by the following correspondence, preserving multiplication order:
    \[
        \mW_0^j\longleftrightarrow\opC_j,
        \qquad(\mW_0^j)^\transpose\longleftrightarrow\opC_j^*,
        \qquad\vu\longleftrightarrow u_\star,
        \qquad\vv\longleftrightarrow v_\star.
    \]
    For $\ve,\ve'\in\set{\vu,\vv}$, let $e_\star,e'_\star$ be their respective representatives. Then
    \begin{equation}\label{eqn:fixed-k-width-limit-mixed-word-gram-limit}
        \frac{1}{n}\langle P(\mW_0)\ve,Q(\mW_0)\ve'\rangle
        \pto
        \langle P(\opC)e_\star,Q(\opC)e'_\star\rangle.
    \end{equation}
    These Gram and scalar limits extend jointly, uniformly on compact scalar controls, to any fixed finite computation using matrix and transpose multiplication, linear combinations, normalized inner products and rank-one maps. For orthogonally equivariant computations with rotation-invariant tight controls, any fixed finite collection of coordinate rows converges along distributionally convergent control subsequences; conditional on the limiting controls, the rows are independent centered Gaussians with the represented Gram covariance.
\end{lemma}

\begin{proof}
    We first compute normalized traces. Write a mixed word of length $2m$ as $\mM_1\cdots\mM_{2m}$, where each factor is an initial matrix or its transpose. Its trace expansion is
    \[
        \frac{1}{n}\Tr(\mM_1\cdots\mM_{2m})
        =\frac{1}{n}\sum_{i_0,\dots,i_{2m-1}=1}^n
        \prod_{\ell=1}^{2m}(\mM_\ell)_{i_{\ell-1},i_\ell},
        \qquad i_{2m}=i_0.
    \]
    We call the cyclic sequence of indices the trace walk. A factor has sign $+$ if untransposed and sign $-$ if transposed. Applying \cref{lem:aux-isserlis}, each pair uses the covariance
    \[
        \EE[(\mW_0^j)_{il}(\mW_0^{j'})_{i'l'}]
        =\frac{1}{n}\1_{j=j'}\1_{i=i'}\1_{l=l'}.
    \]
    We call the initial matrix index its \textit{color}.
    Thus paired entries must have the same color and matching ordered row and column indices.
    We identify the vertices of the trace walk according to these equalities. The resulting quotient graph is connected, has at most $m$ undirected edges, and has one vertex for each free index. If it has $V$ vertices, the contribution is
    \[
        n^{-1}n^{-m}n^V=n^{V-m-1},\qquad V\le m+1.
    \]
    Survival requires a tree with $V=m+1$. Its closed walk traverses each edge twice in opposite directions and completes each descendant subtree before returning, enforcing opposite transpose signs and nested or disjoint pairs. Conversely, removing adjacent matching-color, opposite-sign pairs from a noncrossing pairing cancels each covariance factor against one free sum. The remaining normalized trace is one, so, writing $\fP_{\mathrm{nc}}$ for these pairings,
    \[
        \lim_{n\to\infty}\EE\left[\frac{1}{n}\Tr(\mM_1\cdots\mM_{2m})\right]
        =\#\fP_{\mathrm{nc}}.
    \]
    Other pairings vanish, and odd words have zero expectation. For the corresponding circular operators or adjoints $\opD_\ell$, each annihilation removes the most recently created letter. By \cref{eqn:fixed-k-circular-system-signed-actions},
    \[
        \langle u_\star,\opD_1\cdots\opD_{2m}u_\star\rangle
        =\#\fP_{\mathrm{nc}}.
    \]
    For the variance of a length-$2m$ trace, disconnected pairings cancel against the product of expectations. A connected pairing has $2m$ covariance pairs and at most $2m+1$ free vertices, contributing
    \[
        n^{V-2m-2}\le n^{-1}.
    \]
    Finitely many pairings give joint $L^2$ trace convergence.

    Conditional on the initial matrices, $\vu,\vv$ remain independent standard Gaussians. For any real matrix $\mM$ measurable with respect to the initial matrices, \cref{lem:aux-isserlis} gives
    \[
        \begin{aligned}
            \EE[\vu^\transpose\mM\vu/n\mid\mW_0]&=\Tr(\mM)/n,
            \qquad \Var(\vu^\transpose\mM\vu/n\mid\mW_0)\le 2\normop*{\mM}^2/n,\\
            \EE[\vu^\transpose\mM\vv/n\mid\mW_0]&=0,\\
            \EE[|\vu^\transpose\mM\vv/n|^2\mid\mW_0]
            &=\normfro*{\mM}^2/n^2
            \le\normop*{\mM}^2/n.
        \end{aligned}
    \]
    Apply these identities, also with the sources interchanged, to $\mM=P(\mW_0)^\transpose Q(\mW_0)$. The operator tail bound gives $\EE\normop{\mM}^2=\bigO(1)$ at fixed word degree; vanishing fluctuations and the trace limits prove \cref{eqn:fixed-k-width-limit-mixed-word-gram-limit}, including source orthogonality. Matrix multiplication, transposes, and linear combinations preserve mixed-word expansions, as do rank-one maps through
    \[
        (p\otimes q)z=p\langle q,z\rangle.
    \]
    Induction expresses every computed vector as a finite expansion
    \[
        \vz=\sum_{q=1}^M\theta_q P_q(\mW_0)\ve_q,
        \qquad \ve_q\in\set{\vu,\vv},
    \]
    with coefficients and scalar contractions polynomial in initial mixed-word Grams and controls. Convergence is therefore compact-uniform and transfers to tight controls by truncation.

    For coordinates, we apply an independent Haar rotation:
    \[
        \mW_0^j\longmapsto\mQ\mW_0^j\mQ^\transpose,
        \qquad\ve\longmapsto\mQ\ve,
        \qquad\vz\longmapsto\mQ\vz.
    \]
    Equivariance preserves the joint law and invariant controls. Conditional on the unrotated computation, \cref{lem:aux-haar}~\cref{it2:aux-haar} gives Gaussian row approximation uniformly on bounded, possibly singular Gram sets. Compact-uniform Gram convergence identifies the covariance, and distributionally convergent control subsequences give the stated conditional independence. By \cref{it3:aux-haar},
    \begin{equation}\label{eqn:fixed-k-width-limit-haar-coordinate-moments}
        \EE_{\mQ}\left[|(\mQ\vz)_i|^2\right]=\normrms*{\vz}^2, \qquad
        \EE_{\mQ}\left[|(\mQ\vz)_i|^4\right]=\frac{3n}{n+2}\normrms*{\vz}^4.
    \end{equation}
    Initialization and truncation events are rotation-invariant, so these moments transfer Hilbert-norm bounds to coordinate tightness. Positive limiting variances and \cref{eqn:conditional-gaussian-anti-concentration} give coordinate anti-concentration; uniform variance lower bounds on compact control sets provide the uniform position-wise statement, and increasing the control compacts removes truncation.
\end{proof}

\paragraph{Loopwise limit.}

At $(\alpha_N,\gamma_N)=(1,0)$, we write $\opC_{j,t,N}$ for the width-limit representative of $\mW_t^j$ at fixed $N$, and $\opC_{j,t}$ for its subsequent loop limit.
For fixed controls $d_t$, we define the continuum recursion
\begin{equation}\label{eqn:fixed-k-loop-limit-continuum-controlled-recursion}
    \begin{aligned}
        \opC_{j,0} &= \opC_j, \\
        \opR_{j,t} &= \opC_{j,t} \cdots \opC_{1,t}, \quad \opR_{0,t} = \opI, \\
        \opS_{j,t} &= \opR_{j-1,t}, \quad \opT_{j,t} = \opC_{k,t} \cdots \opC_{j+1,t},
        \quad \opB_t = \opC_{k,t} \cdots \opC_{1,t} = \opR_{k,t}, \\
        X_t (s) &= e^{as\opB_t} u_\star, \quad Y_t (s) = e^{a(1-s)\opB_t^*} v_\star, \\
        z_{j,t}(s) &= \opS_{j,t} X_t (s), \quad b_{j,t}(s) = \opT_{j,t}^* Y_t (s), \\
        \opJ_{j,t} &= a \int_0^1 b_{j,t}(s) \otimes z_{j,t}(s)\,\dif s,
        \qquad \opC_{j,t+1} = \opC_{j,t} - d_t \opJ_{j,t}.
    \end{aligned}
\end{equation}
For fixed $N$, we initialize $\opC_{j,0,N}=\opC_j$ and define $\opR_{j,t,N},\opS_{j,t,N},\opT_{j,t,N},\opB_{t,N}$ by the same products as above. Set $\opE_{t,N}=\opI+a\opB_{t,N}/N$. With the integer loop index $q=0,\dots,N-1$, the discrete update is
\[
    \begin{aligned}
        \opJ_{j,t,N}&=\frac{a}{N}\sum_{q=0}^{N-1}
        \opT_{j,t,N}^*(\opE_{t,N}^*)^{N-1-q}v_\star
        \otimes \opS_{j,t,N}\opE_{t,N}^qu_\star,\\
        \opC_{j,t+1,N}&=\opC_{j,t,N}-d_t\opJ_{j,t,N}.
    \end{aligned}
\]
The corresponding continuum fraction is $q/N$; the backward power includes the remaining $N-1-q$ loops.

\begin{lemma}\label{lem:fixed-k-loop-limit}
    Fix $k,a$ and a finite training horizon $T$. For every real control transcript $(d_0,\dots,d_{T-1})$, \cref{eqn:fixed-k-loop-limit-continuum-controlled-recursion} defines bounded operators $\opC_{j,t}$ and Hilbert--Schmidt updates $\opJ_{j,t}$. On each compact transcript set, there is a constant $C$ such that, for $1\le j\le k$,
    \[
        \normop*{\opC_{j,t}}\le C\quad(0\le t\le T),
        \qquad\normHS*{\opJ_{j,t}}\le C\quad(0\le t<T).
    \]
    All path norms are uniformly bounded there, and the paths are continuously differentiable in $s\in[0,1]$ and continuous in the transcript. For $1\le j\le k$ and $0\le t<T$, the discrete representatives satisfy the following estimates uniformly on compact control sets as $N\to\infty$; the estimates for $\opC_{j,t,N}$ and the vector paths also hold at $t=T$.
    \begin{equation}\label{eqn:fixed-k-loop-limit-discrete-to-continuum-convergence}
        \begin{aligned}
            \normop*{\opC_{j,t,N} - \opC_{j,t}}
            +\normHS*{\opJ_{j,t,N} - \opJ_{j,t}} &= \bigO(N^{-1}), \\
            \sup_{r \le N}\norm*{\opE_{t,N}^r u_\star - X_t (r / N)}&= \bigO(N^{-1}), \\
            \sup_{1 \le r \le N}\norm*{
            \opR_{j,t,N} \opE_{t,N}^{r-1} u_\star
            -\opR_{j,t} X_t ((r - 1) / N)}&= \bigO(N^{-1}).
        \end{aligned}
    \end{equation}
\end{lemma}

\begin{proof}
    The finite induction for \cref{eqn:fixed-k-identities-and-bounds-matrix-gradient-and-update-identities} applies with $\normop{\opC_j}\le2$, unit source vectors, and $\normHS{p\otimes q}=\norm{p}\norm{q}$. On a common ball of bounded operators, the identity
    \[
        e^{as\opD}-e^{as\opF}
        =\int_0^{as}e^{(as-v)\opD}(\opD-\opF)e^{v\opF}\dif v
    \]
    gives Lipschitz dependence of exponential paths. Taylor expansion and product telescoping give
    \begin{equation}\label{eqn:fixed-k-loop-limit-euler-exponential-error}
        \sup_{r \le N}\normop*{(\opI + a\opD / N)^r - e^{ar\opD/N}}\le C / N.
    \end{equation}

    Products of the fixed $k$ bounded operators are Lipschitz in their factors. Replacing discrete paths by exponentials, including the one-loop backward shift, and taking the Riemann integral both cost $\bigO(N^{-1})$ in Hilbert--Schmidt norm, since the integrand has uniformly bounded derivative. Starting from $\opC_{j,0,N}=\opC_{j,0}$, finite-time induction proves \cref{eqn:fixed-k-loop-limit-discrete-to-continuum-convergence}, control continuity and all compact-uniform path bounds.

    At fixed $N$, \cref{lem:fixed-k-width-limit} identifies the discrete recursion by replacing each normalized rank-one map with its Hilbert rank-one map. The preceding estimates then give the subsequent loop limit in one fixed Hilbert space. Compact truncation transfers the estimates to tight realized controls; under the rotation conditions of \cref{lem:fixed-k-width-limit}, the Haar moment identities and \cref{it2:aux-haar} give the coordinate laws along control subsequences converging in distribution.
\end{proof}

The preceding bounds establish bounded feature magnitudes.
To establish faithfulness, we also need the internal preactivation norms to remain bounded away from zero after training.
We construct a scalar evaluator that detects nonzero vectors through their word expansions.

\begin{lemma}\label{lem:fixed-k-entire-word-evaluation}
    We use the circular operators from \cref{eqn:fixed-k-width-limit-colored-circular-system}. For $A\ge1$ set
    \[
        \fE_A = \set{h = \sum_{w,\rho} h_{w,\rho} e_{w,\rho} : \norm*{h}_A \coloneqq \sum_{w,\rho}|h_{w,\rho}|A^{|w|} < \infty},
        \qquad
        \fE = \bigcap_{A \ge 1} \fE_A.
    \]
    Here $e_{w,\rho}$ is the basis vector for signed word $w$ in source $\rho\in\set{u,v}$; for $\omega\in\CC$, we complexify these spaces. For every finite real control transcript, the paths in \cref{eqn:fixed-k-loop-limit-continuum-controlled-recursion} and all finite products of trained operators $\opC_{j,t}$ and their adjoints applied to them belong to $\fE$. For every $\omega\in\CC$, a linear functional $\opPhi_\omega$ on $\fE$, continuous on each $\fE_A$ with $A\ge1+|\omega|$, satisfies
    \begin{equation}\label{eqn:fixed-k-entire-word-evaluation-evaluation-intertwining-identities}
        \begin{aligned}
            \opPhi_\omega(u_\star) &= 1,
            \qquad \opPhi_\omega(v_\star) = 0, \\
            \opPhi_\omega(\opC_j h) &= \omega \opPhi_\omega(h),
            \qquad \opPhi_\omega(\opC_j^* h) = 0.
        \end{aligned}
    \end{equation}
    It commutes with the path derivatives and integrals used below, but need not be bounded in Hilbert norm.
\end{lemma}

\begin{proof}
    For $b=(j,\pm)$, let $e_b=(e_j,\pm)$ and let $\bar{b}$ have the opposite sign and same color. Set $\opD_b=\cre{e_b}+\ann{e_{\bar{b}}}$, so $\opD_{(j,+)}=\opC_j$ and $\opD_{(j,-)}=\opC_j^*$. With $x_{(j,+)}=\omega$ and $x_{(j,-)}=0$, define
    \begin{equation}\label{eqn:fixed-k-entire-word-evaluation-word-polynomial-recursion}
        p_\varnothing(\omega)=1,\qquad
        p_{bw}(\omega)=x_bp_w(\omega)-\1_{w=\bar{b} w'}p_{w'}(\omega).
    \end{equation}
    Here the subtraction is present only when $w$ begins with $\bar{b}$, and $w'$ is the remaining suffix. Since $|\omega|/A+A^{-2}\le1$ at $A=1+|\omega|$, induction gives $|p_w(\omega)|\le A^{|w|}$. For every $A\ge1+|\omega|$, the series
    \[
        \opPhi_\omega(h)=\sum_w h_{w,u}p_w(\omega),
        \qquad |\opPhi_\omega(h)|\le\norm{h}_A
    \]
    converges absolutely. The identity $\opD_be_w=e_{bw}+\1_{w=\bar{b}w'}e_{w'}$ on either source copy proves \cref{eqn:fixed-k-entire-word-evaluation-evaluation-intertwining-identities} on finite words by the polynomial recursion. The bounds
    \[
        \norm*{\cre{e_b}h}_A=A\norm*{h}_A,\qquad
        \norm*{\ann{e_b}h}_A\le A^{-1}\norm*{h}_A,
    \]
    extend these relations to $\fE_A$ by continuity.

    Fix $A\ge1$. Its norm dominates the Hilbert norm, and both vacua have weighted norm one. Inductively, suppose $\opC_{j,t}$ and their Hilbert adjoints are bounded on $\fE_A$, as at $t=0$. Their finite products and exponential series give continuous paths $X_t,Y_t,z_{j,t},b_{j,t}$ in $\fE_A$. For real $p,q\in\fE_A$,
    \[
        \begin{aligned}
            \norm*{(p\otimes q)h}_A
            &\le\norm*{p}_A\norm*{q}_{\fH}\norm*{h}_A,\\
            \norm*{(p\otimes q)^*h}_A
            &\le\norm*{q}_A\norm*{p}_{\fH}\norm*{h}_A,
            \qquad \fH=\fH_u\oplus\fH_v.
        \end{aligned}
    \]
    To complete the induction, we apply these estimates to the integrand $b_{j,t}(s)\otimes z_{j,t}(s)$ defining $\opJ_{j,t}$ in \cref{eqn:fixed-k-loop-limit-continuum-controlled-recursion}.
    This integrand and its Hilbert adjoint are continuous in the bounded-operator norm on $\fE_A$, so their Riemann sums converge to bounded operators there.
    Hence $\opC_{j,t+1}=\opC_{j,t}-d_t\opJ_{j,t}$ and its Hilbert adjoint remain bounded on $\fE_A$.
    Since $A$ was arbitrary, all asserted vectors belong to $\fE$; norm domination identifies their exponential and integral limits with the Hilbert-space limits. 
    Bounded-operator exponentials give differentiability in the loop fraction and continuity in controls, with bounds uniform on compact control sets for fixed $A$. 
    Extending real inner-product functionals complex-linearly and taking $A\ge1+|\omega|$ gives, for any stated path $h(s)$,
    \[
        \opPhi_\omega(h'(s))=\frac{\dif}{\dif s}\opPhi_\omega(h(s)),
        \qquad
        \opPhi_\omega\left(\int_0^1h(s)\dif s\right)
        =\int_0^1\opPhi_\omega(h(s))\dif s.
    \]
\end{proof}

For input one at $(\alpha_N,\gamma_N)=(1,0)$, $X_t(s)$ represents the residual state and $\opR_{j,t}X_t(s)$ represents internal preactivation $j$ for $1\le j\le k$.
We seek positive lower bounds for their norms to establish faithfulness, and for $\norm{\opB_tX_t(s)}$ to establish nondegenerate loop increments.
We obtain these bounds by ruling out cancellation of every internal prefix under the shared training recursion.

\begin{lemma}[Nonvanishing residual states and internal preactivations]\label{lem:fixed-k-prefix-noncancellation}
    For every finite horizon $T$ and real control transcript $(d_0,\dots,d_{T-1})$ in \cref{eqn:fixed-k-loop-limit-continuum-controlled-recursion}, every $0\le t\le T$, every $0\le j\le k$, and every $s\in[0,1]$,
    \begin{equation}\label{eqn:fixed-k-prefix-noncancellation-prefix-noncancellation}
        \opR_{j,t} X_t (s) \neq 0.
    \end{equation}
    On each compact set of such transcripts, these norms have a positive lower bound uniform over $s\in[0,1]$, $0\le j\le k$, and $0\le t\le T$. In particular, this applies to $\opB_te^{as\opB_t}u_\star=\opR_{k,t}X_t(s)$.
\end{lemma}

\begin{proof}
    We use \cref{lem:fixed-k-entire-word-evaluation} to test whether training can cancel the initialization contribution to each prefix.
    We compare their evaluations at an auxiliary complex parameter, whose radius will tend to infinity. Choose $r\ge1$ and set
    \[
        \omega = r e^{\ii \pi / (2k)}, \qquad \omega^k = \ii r^k.
    \]
    Since $|e^{as\omega^k}|=1$, oscillatory integration controls the evaluated training corrections without requiring the update coefficients to be small.
    All estimates in this proof are uniform over $s \in [0, 1]$ and over a fixed compact set of finite real controls. Constants may depend on that compact set, $a, k, t$, but not on $r$. Write $F_{j,t} (s, \omega) = \opPhi_\omega(\opR_{j,t} X_t (s))$. We claim
    \begin{equation}\label{eqn:fixed-k-prefix-noncancellation-prefix-evaluation-asymptotic}
        F_{j,t} (s, \omega)
        = \omega^j e^{as \omega^k} + \bigO(r^{j-k}), \qquad 0 \le j \le k.
    \end{equation}
    At $t=0$, \cref{eqn:fixed-k-entire-word-evaluation-evaluation-intertwining-identities} and the exponential power series give exact equality without the error. We induct jointly on this asymptotic and the backward-vector bounds in \cref{eqn:fixed-k-prefix-noncancellation-proper-prefix-backward-bound,eqn:fixed-k-prefix-noncancellation-full-prefix-backward-bound}. Their time-zero evaluations vanish because the initial backward vectors lie in the readout source, on which $\opPhi_\omega$ is zero.
    Put
    \begin{equation}\label{eqn:fixed-k-prefix-noncancellation-trained-matrix-perturbation}
        \opK_{m,t} = \opC_{m,t} - \opC_m
        = -a \sum_{q<t} d_q \int_0^1
        b_{m,q}(s) \otimes \opR_{m-1,q} X_q (s)\,\dif s.
    \end{equation}
    Suppose \cref{eqn:fixed-k-prefix-noncancellation-prefix-evaluation-asymptotic} is known at every $q < t$. For $h \in \fE$, applying $\opPhi_\omega$ to the adjoint of \cref{eqn:fixed-k-prefix-noncancellation-trained-matrix-perturbation} and using $\opPhi_\omega(\opC_m^* h) = 0$ gives
    \begin{equation}\label{eqn:fixed-k-prefix-noncancellation-adjoint-evaluation-identity}
        \opPhi_\omega(\opC_{m,t}^* h)
        = -a \sum_{q<t} d_q \int_0^1
        F_{m-1,q} (s, \omega) \langle b_{m,q}(s), h \rangle\,\dif s.
    \end{equation}
    The scalar function $g(s) = \langle b_{m,q}(s), h \rangle$ has $\sup_s(| g(s) | + | g'(s) |) \le C\norm{h}$, by \cref{eqn:fixed-k-loop-limit-continuum-controlled-recursion}. Integration by parts therefore gives
    \[
        \left| \int_0^1 e^{\ii a r^k s} g(s)\,\dif s \right|
        \le \frac{| g(0) | + | g(1) | + \int_0^1|g'(s)|\,\dif s}{a r^k}
        \le C r^{-k}\norm*{h}.
    \]
    Inserting \cref{eqn:fixed-k-prefix-noncancellation-prefix-evaluation-asymptotic} for $F_{m-1,q}$ in \cref{eqn:fixed-k-prefix-noncancellation-adjoint-evaluation-identity}, the leading term contributes $\bigO(r^{m-1-k})\norm{h}$; its error has the same bound by direct integration. Hence
    \begin{equation}\label{eqn:fixed-k-prefix-noncancellation-adjoint-evaluation-bound}
        | \opPhi_\omega(\opC_{m,t}^* h) |
        \le C r^{m-1-k}\norm*{h}.
    \end{equation}
    This composite functional, unlike $\opPhi_\omega$ itself, extends to a bounded Hilbert-space functional with the displayed norm.
    We need corresponding estimates for the backward range vectors. For $j < k$,
    \[
        b_{j,t}(s) = \opC_{j+1,t}^*
        \left(\opC_{j+2,t}^* \cdots \opC_{k,t}^* Y_t (s)\right).
    \]
    We apply \cref{eqn:fixed-k-prefix-noncancellation-adjoint-evaluation-bound} with $m=j+1$ and the Hilbert bounds from \cref{eqn:fixed-k-loop-limit-continuum-controlled-recursion} to obtain
    \begin{equation}\label{eqn:fixed-k-prefix-noncancellation-proper-prefix-backward-bound}
        | \opPhi_\omega(b_{j,t}(s)) | \le C r^{j-k}\quad(j < k).
    \end{equation}
    For $j = k$, a slightly sharper bound is useful. Since $\opB_t^* = \opC_{1,t}^* \cdots \opC_{k,t}^*$, \cref{eqn:fixed-k-prefix-noncancellation-adjoint-evaluation-bound} with $m = 1$ shows that $\opPhi_\omega \opB_t^*$ is bounded in Hilbert norm by $Cr^{-k}$. As $\opPhi_\omega(v_\star) = 0$,
    \[
        \opPhi_\omega(e^{a v \opB_t^*} v_\star)
        = a \int_0^v \opPhi_\omega(\opB_t^* e^{a u \opB_t^*} v_\star)\,\dif u
        = \bigO(r^{-k}), \qquad 0 \le v \le 1.
    \]
    Thus
    \begin{equation}\label{eqn:fixed-k-prefix-noncancellation-full-prefix-backward-bound}
        | \opPhi_\omega(b_{k,t}(s)) | \le C r^{-k}.
    \end{equation}
    This establishes both backward-vector bounds at time $t$.
    Next apply $\opPhi_\omega$ to \cref{eqn:fixed-k-prefix-noncancellation-trained-matrix-perturbation} itself. The bounds in \cref{eqn:fixed-k-prefix-noncancellation-proper-prefix-backward-bound,eqn:fixed-k-prefix-noncancellation-full-prefix-backward-bound} at the earlier times $q < t$ and the bounded forward paths show
    \begin{equation}\label{eqn:fixed-k-prefix-noncancellation-perturbation-evaluation-bound}
        | \opPhi_\omega(\opK_{m,t} h) | \le
        \begin{cases}
            C r^{m-k}\norm*{h}, &m < k, \\
            C r^{-k}\norm*{h}, &m = k.
        \end{cases}
    \end{equation}
    Each functional therefore extends boundedly to $\fH_u\oplus\fH_v$. We apply $\opPhi_\omega(\opC_m h)=\omega\opPhi_\omega(h)$ successively along a prefix to obtain the exact identity on $\fE$
    \begin{equation}\label{eqn:fixed-k-prefix-noncancellation-prefix-evaluation-decomposition}
        \begin{aligned}
            \opPhi_\omega(\opR_{j,t} h) &= \omega^j \opPhi_\omega(h) + \opD_{j,t,\omega} (h), \\
            \opD_{j,t,\omega} (h) &= \sum_{m=1}^j \omega^{j-m} \opPhi_\omega(\opK_{m,t} \opR_{m-1,t} h).
        \end{aligned}
    \end{equation}
    The initial term $\opD_{0,t,\omega}$ is zero. 
    By \cref{eqn:fixed-k-prefix-noncancellation-perturbation-evaluation-bound} and the Hilbert operator bounds for the prefixes, $\opD_{j,t,\omega}$ extends to a bounded linear functional satisfying
    \begin{equation}\label{eqn:fixed-k-prefix-noncancellation-remainder-functional-bound}
        \norm*{\opD_{j,t,\omega}}_{(\fH_u\oplus\fH_v)^*} \le C r^{j-k}
        \quad(1 \le j \le k),
    \end{equation}
    where the dual norm is the supremum over the Hilbert unit ball.
    For $j = k$ the term $m = k$ has the stronger bound $Cr^{-k}$; the other terms are $\bigO(1)$. For $k = 1$ there are no other terms, and the asserted $\bigO(1)$ bound is still valid.
    Now $X_t'(s) = a\opB_t X_t (s)$, so \cref{eqn:fixed-k-prefix-noncancellation-prefix-evaluation-decomposition} for $j = k$ gives the scalar ODE
    \[
        \begin{aligned}
            F_{0,t}'(s, \omega) &= a\omega^k F_{0,t} (s, \omega) + a H_{t,\omega} (s), \\
            F_{0,t} (0, \omega) &= 1,
            \qquad H_{t,\omega} (s) = \opD_{k,t,\omega} (X_t (s)).
        \end{aligned}
    \]
    The crucial consequence of \cref{eqn:fixed-k-prefix-noncancellation-remainder-functional-bound} is
    \begin{equation}\label{eqn:fixed-k-prefix-noncancellation-forcing-function-bound}
        \sup_s\left(| H_{t,\omega} (s) | + | H_{t,\omega}'(s) |\right) \le C,
    \end{equation}
    independently of $r$: both $X_t$ and $X_t' = a\opB_t X_t$ have bounded Hilbert norms. Solving the ODE gives
    \[
        F_{0,t} (s, \omega) = e^{as \omega^k}
        +a \int_0^s e^{a(s-v)\omega^k} H_{t,\omega} (v)\,\dif v.
    \]
    Using $|\omega^k|=r^k$ and $|e^{a(s-v)\omega^k}|=1$, integration by parts gives
    \[
        \begin{aligned}
            a\left|\int_0^s e^{a(s-v)\omega^k}H_{t,\omega}(v)\,\dif v\right|
            &\le r^{-k}\left(
            |H_{t,\omega}(s)|+|H_{t,\omega}(0)|
            +\int_0^s|H_{t,\omega}'(v)|\,\dif v\right)\\
            &\le Cr^{-k},
        \end{aligned}
    \]
    uniformly over $s\in[0,1]$, by \cref{eqn:fixed-k-prefix-noncancellation-forcing-function-bound}. This proves \cref{eqn:fixed-k-prefix-noncancellation-prefix-evaluation-asymptotic} for $j = 0$. Applying \cref{eqn:fixed-k-prefix-noncancellation-prefix-evaluation-decomposition} with $h = X_t (s)$ and using \cref{eqn:fixed-k-prefix-noncancellation-remainder-functional-bound} now proves it for all $j$: multiplication of the $j = 0$ error by $\omega^j$ has size $\bigO(r^{j-k})$. This completes the induction.
    The leading term in \cref{eqn:fixed-k-prefix-noncancellation-prefix-evaluation-asymptotic} has modulus exactly $r^j$. Choose $r$ sufficiently large that its error is at most $r^j / 2$. Then $F_{j,t} (s, \omega) \neq 0$ at every $s$, proving \cref{eqn:fixed-k-prefix-noncancellation-prefix-noncancellation}, because the zero vector has zero evaluation on the stated domain. The argument applies to arbitrary real controls. The Hilbert-norm lower bounds follow from continuity of $(d_0, \dots, d_{t-1}, s) \mapsto\norm{\opR_{j,t} X_t (s)}$ in \cref{eqn:fixed-k-loop-limit-continuum-controlled-recursion} and compactness, rather than boundedness of the evaluator. Since these norms are positive everywhere and there are finitely many $j,t$, their minimum on each compact control set times $[0,1]$ is positive.
\end{proof}

\paragraph{First-step response.}
We use a singleton first SGD step with fixed $\xi_0\ne0$ and $\chi_0\to\chi_\star\ne0$, and write $d_\star=\eta\chi_\star\xi_0$.
A loss derivative continuous and nonzero at zero supplies this signal because the initialized output has width limit zero.
All lower bounds use a fixed nondegenerate evaluation batch.

\begin{proposition}\label{pro:fixed-k-canonical-response}
    Fix a finite block depth $k\ge1$ and $a,\eta>0$, and set $(\alpha_N,\gamma_N)=(1,0)$. Consider a singleton first SGD step with $\xi_0\ne0$ and $\chi_0\to\chi_\star\ne0$ in the width limit.
    On every fixed nondegenerate evaluation batch, this update has a nonzero ordered-limit output response. For every fixed $\lambda\in(0,1]$, each fixed coordinate of $\Delta\vx_0^{\floor{\lambda N}}$ and of $\Delta\vh_0^{\floor{\lambda N},j}$ for every $1\le j\le k$ has weak $\bigTheta(1)$ scale.
\end{proposition}

\begingroup
\setlength{\emergencystretch}{2em}
\begin{proof}
    We write $d_\star=\eta\chi_\star\xi_0$ and suppress time-zero subscripts on $\opB_t,\opS_{j,t},\opT_{j,t},\opR_{j,t},\opJ_{j,t}$ in \cref{eqn:fixed-k-loop-limit-continuum-controlled-recursion}. Each initial operator preserves $\fH_u$ and $\fH_v$. In the rank-one integral for $\opJ_j$, the range vectors lie in $\fH_v$ and the domain vectors in $\fH_u$. Hence $\opJ_j$ maps $\fH_u$ into $\fH_v$ and vanishes on $\fH_v$. For every product $P$ of initial operators and every $j,m$,
    \[
        \opJ_jP\opJ_m=0.
    \]
    Every term containing two update insertions therefore vanishes. The same identity holds for the discrete operators $\opJ_{j,0,N}$.
    It follows, for the simultaneous update of all matrices, that
    \begin{equation}\label{eqn:fixed-k-canonical-response-matrix-and-branch-updates}
        \begin{aligned}
            \opB_1 &= \opB - d_\star \opH,
            \qquad \opH = \sum_{j=1}^k \opT_j \opJ_j \opS_j, \\
            \opR_{j,1} &= \opR_j - d_\star \opM_j,
            \qquad \opM_j = \sum_{m=1}^j \opC_j \cdots \opC_{m+1} \opJ_m \opS_m.
        \end{aligned}
    \end{equation}
    Here $\opH, \opM_j$ also map the first source into the second and vanish on the second. In particular all powers and exponentials involving these perturbations contain at most one perturbation. Duhamel yields the exact identity
    \begin{equation}\label{eqn:fixed-k-canonical-response-exact-state-response}
        \begin{aligned}
            X_1 (\lambda) - X_0 (\lambda) &= -d_\star D(\lambda), \\
            D(\lambda) &= a \int_0^\lambda e^{a(\lambda-s) \opB} \opH e^{as\opB} u_\star\,\dif s.
        \end{aligned}
    \end{equation}
    There is no Taylor remainder in $d_\star$. Pair with $v_\star$ at $\lambda=1$. The contribution of matrix $j$ is
    \begin{equation}\label{eqn:fixed-k-canonical-response-matrix-hs-pairing}
        a \int_0^1
        \langle \opT_j^* e^{a(1-s)\opB^*} v_\star,
        \opJ_j \opS_j e^{as\opB} u_\star \rangle\,\dif s
        = \langle \opJ_j, \opJ_j \rangle_{\mathrm{HS}}.
    \end{equation}
    The rank-one identity $\langle\opJ_j,b\otimes z\rangle_{\mathrm{HS}}=\langle b,\opJ_jz\rangle$ and the integral definition of $\opJ_j$ give this equality. Therefore
    \begin{equation}\label{eqn:fixed-k-canonical-response-output-response-limit}
        f_1(\zeta) - f_0(\zeta)
        \pto -d_\star \zeta \sum_{j=1}^k \normHS*{\opJ_j}^2.
    \end{equation}
    At fixed $N$ the identical computation with residual powers and $\opJ_{j,0,N}$ gives the same formula with those discrete $\opJ$'s. Thus \cref{eqn:fixed-k-loop-limit-discrete-to-continuum-convergence} proves the actual ordered limit, including all simultaneous matrix-update terms.
    For positivity, expand the two exponentials in the integral defining $\opJ_j$. Their constant terms give the coefficient $a$ of the unit Hilbert--Schmidt basis vector $(\opT_j^*v_\star)\otimes(\opS_ju_\star)$, while every remaining all-creation term is orthogonal to it. Hence, directly within this proposition,
    \[
        \normHS*{\opJ_j} \ge a.
    \]
    Hence \cref{eqn:fixed-k-canonical-response-output-response-limit} is nonzero for every nondegenerate input. At $\zeta=\xi_0$ it equals $-\eta\chi_\star\xi_0^2\sum_j\normHS{\opJ_j}^2$.
    We project onto explicit word-basis vectors to obtain lower bounds for both the residual-state and internal-preactivation responses. All coefficients of every initial $\opC_j, \opC_j^*$, their products, and their exponentials at nonnegative times are nonnegative in the word basis. The same is true of $\opJ_j, \opH, \opM_j$ and $D(\lambda)$. The $j = k$ term of $\opH$ is $\opJ_k \opS_k$. For $0 \le s \le 1$,
    \begin{equation}\label{eqn:fixed-k-canonical-response-terminal-matrix-pairing-lower-bound}
        \begin{aligned}
            \langle v_\star, \opJ_k \opS_k e^{as\opB} u_\star \rangle
            &= a \int_0^1
            \langle v_\star, e^{a(1-q)\opB^*} v_\star \rangle
            \langle \opS_k e^{aq\opB} u_\star, \opS_k e^{as\opB} u_\star \rangle\,\dif q \\
            &\ge a.
        \end{aligned}
    \end{equation}
    Indeed the first factor equals one and the second is at least one: the same-direction prefix vectors have a common shortest word with coefficient one. The identity term in the outer exponential in \cref{eqn:fixed-k-canonical-response-exact-state-response} therefore gives
    \begin{equation}\label{eqn:fixed-k-canonical-response-state-vacuum-lower-bound}
        \langle v_\star, D(\lambda) \rangle \ge a^2 \lambda.
    \end{equation}
    For every internal position $j$, the prefix vector $\opR_{j,t}X_t(\lambda)$ represents the preactivation in the identity-activation model. Using \cref{eqn:fixed-k-canonical-response-matrix-and-branch-updates,eqn:fixed-k-canonical-response-exact-state-response}, terms with two perturbations vanish, so exactly
    \begin{equation}\label{eqn:fixed-k-canonical-response-prefix-state-response}
        \opR_{j,1} X_1 (\lambda) - \opR_j X_0 (\lambda)
        = -d_\star\left(\opM_j X_0 (\lambda) + \opR_j D(\lambda)\right).
    \end{equation}
    Both vectors inside parentheses have nonnegative coefficients, and $\opR_jv_\star=e_{(j,+)\cdots(1,+),v}$ is a unit vector. The creation part of $\opR_j$ carries the vacuum coefficient of $D(\lambda)$ to this word, so
    \[
        \begin{aligned}
            \langle \opR_jv_\star,\opM_jX_0(\lambda)+\opR_jD(\lambda)\rangle
            &\ge\langle \opR_jv_\star,\opR_jD(\lambda)\rangle\\
            &\ge\langle v_\star,D(\lambda)\rangle
            \ge a^2\lambda.
        \end{aligned}
    \]
    Together with \cref{eqn:fixed-k-canonical-response-state-vacuum-lower-bound}, this gives norm at least $|d_\star|a^2\lambda>0$ for the residual-state response and every internal preactivation response. Bounded operators give finite upper bounds. The lower constant may depend on the fixed positive $\lambda$; no lower bound at $\lambda=0$ is required for feature learning.
    The convergence in \cref{eqn:fixed-k-loop-limit-discrete-to-continuum-convergence} transfers both sets of response vectors and their Gram norms to the ordered finite-network limit; the preactivation uses the branch input at loop fraction $(\floor{\lambda N}-1)/N$, which differs from $\lambda$ by $\bigO(N^{-1})$. The coordinate kernel in \cref{lem:fixed-k-width-limit} gives nondegenerate centered Gaussian laws with these squared norms as variances. Scalar nondegenerate inputs and fixed finite batches preserve the stated scales. This proves feature learning in both the residual state and every internal preactivation at the update $t=0\to1$, as required by \cref{def:one-block-feature-learning}.
\end{proof}
\endgroup

For $k=1$, set $\opB_t=\opC_{1,t}$ and $\opG_t=\opJ_{1,t}$. Then \cref{eqn:fixed-k-loop-limit-continuum-controlled-recursion} becomes
\[
    \opB_0=\opC,\qquad
    \opG_t=a\int_0^1 e^{a(1-s)\opB_t^*}v_\star\otimes e^{as\opB_t}u_\star\,\dif s,
    \qquad \opB_{t+1}=\opB_t-d_t\opG_t.
\]
\cref{lem:fixed-k-prefix-noncancellation} with $j=1$ gives $\opB_te^{as\opB_t}u_\star\ne0$ throughout $[0,1]$. At the first update, \cref{pro:fixed-k-canonical-response} gives output response $-d_\star\zeta\normHS{\opG_0}^2$ and state projection of magnitude at least $|d_\star|a^2\lambda$. These specializations also apply to a one-color invariant subspace of a larger circular system.

\subsection{Classification of loopwise parameterizations at fixed block depth}\label{app:fixed-k-classification}

The following results apply to every fixed finite $k\ge1$.
The finite-step bounds in \cref{lem:fixed-k-identities-and-bounds} give the stable-region upper estimates. We now test the matching initialization and local-update scales.
Nontriviality is classified only within the initialization- and training-stable region.

\begin{lemma}[Initialization and matrix-gradient normalization]\label{lem:fixed-k-initialization-and-gradient}
    For every fixed finite $k\ge1$, initialization is stable if and only if $\alpha_N\ge1$.
    In this region, every initial internal preactivation has weak $\bigTheta(1)$ coordinatewise scale on every fixed nondegenerate evaluation batch, uniformly over $1\le r\le N$ and $1\le j\le k$.
    If also $\gamma_N\ge0$, every fixed SGD routine satisfies, for each matrix $1\le j\le k$ and step $t<T$,
    \begin{equation}\label{eqn:fixed-k-initialization-and-gradient-matrix-gradient-upper-scale}
        \mG_{j,t}=\bigO\left(n^{-1}N^{1-\alpha_N}\right).
    \end{equation}
    For a nondegenerate initial training signal at $t=0$, this scale is attained at every fixed entry, including diagonal entries. This justifies the gradient rescaling $\widetilde{\mG}_{j,t}=nN^{\alpha_N-1}\mG_{j,t}$ for every hidden matrix $j$.
\end{lemma}

\begin{proof}
    For $\alpha_N\ge1$, \cref{lem:fixed-k-identities-and-bounds} gives bounded residual-state and internal-preactivation RMS norms and bounded outputs at initialization, and \cref{eqn:fixed-k-width-limit-haar-coordinate-moments} gives the coordinatewise upper bounds uniformly over loop position. For necessity, write $\opB=\opC_k\cdots \opC_1$ and $w=(k,+)\cdots(1,+)$. Every multiplication in $\opB^m u_\star$ is a creation, so $\opB^m u_\star=e_{w^m,u}$, and these vectors are orthonormal. Consequently the fixed-$N$ limiting squared norm of $\mA_0^N\vu$ is
    \begin{equation}\label{eqn:fixed-k-initialization-and-gradient-initial-coordinate-variance-lower-bound}
        \sigma_N^2=\sum_{m=0}^N\binom{N}{m}^2c_N^{2m}
        \ge1+a^2N^{2-2\alpha_N}.
    \end{equation}
    By \cref{lem:fixed-k-width-limit}, a fixed coordinate converges at fixed $N$ to $\fN(0,\sigma_N^2)$. If $\alpha_N<1$, then for every $M>0$ its subsequent probability of lying in $[-M,M]$ is at most $\sqrt{2/\pi}M/\sigma_N\to0$, so initialization stability fails.

    For initialization faithfulness, $\opR_j\opB^m u_\star$ are also orthonormal as $m$ varies. In loop $r$, the input-one preactivation variance is therefore $\sum_{m=0}^{r-1}\binom{r-1}{m}^2c_N^{2m}$, between $1$ and $e^{2a}$ when $\alpha_N\ge1$. The Gaussian coordinate law gives the claimed lower scale uniformly over loops and the fixed set of matrices.

    For the upper gradient bound, work first on an invariant compact-control truncation event. Every $\vb,\vz$ in \cref{eqn:fixed-k-identities-and-bounds-matrix-gradient-and-update-identities} then has normalized norm at most a fixed $C$. For fixed entry indices $i,l$, \cref{eqn:fixed-k-width-limit-haar-coordinate-moments} and Cauchy--Schwarz give
    \[
        \EE_{\mQ}\left|(\mQ \vb)_i(\mQ \vz)_l\right|^2\le3C^4.
    \]
    The $L^2(\mQ)$ triangle inequality bounds the sum over $N$ loops by $\sqrt{3}C^2N$, without requiring independence between loops. Apply \cref{eqn:fixed-k-identities-and-bounds-matrix-gradient-and-update-identities}, use compact bounds on $d_t/\eta$, and then remove truncation to obtain \cref{eqn:fixed-k-initialization-and-gradient-matrix-gradient-upper-scale}.

    For the matching first-step scale, retain $\xi_0\ne0$ and $\chi_0\to\chi_\star\ne0$, and put $\opS_j=\opC_{j-1}\cdots \opC_1$ and $\opT_j=\opC_k\cdots \opC_{j+1}$ at initialization. The families
    \[
        f_{j,p}=\opS_j\opB^pu_\star,
        \qquad e_{j,q}=\opT_j^*(\opB^*)^qv_\star,
        \qquad p,q\ge0,
    \]
    are orthonormal within each family and orthogonal across the two source copies. Expanding the exact gradient and using
    $\sum_{s=0}^{N-1}\binom{s}{p}\binom{N-1-s}{q}=\binom{N}{p+q+1}$ gives the coefficient of $e_{j,q}\otimes f_{j,p}$ in
    \[
        \opJ_{j,N}\coloneqq\frac{a}{N}\sum_{s<N}
        \opT_j^*(\opI+c_N\opB^*)^{N-1-s}v_\star
        \otimes \opS_j(\opI+c_N\opB)^su_\star
    \]
    as
    \begin{equation}\label{eqn:fixed-k-initialization-and-gradient-finite-loop-gradient-coefficients}
        b_{qp,N}=\frac{a}{N}c_N^{p+q}\binom{N}{p+q+1}.
    \end{equation}
    Take this coefficient to be zero when $p+q+1>N$. For $\alpha_N\ge1$,
    \[
        0\le b_{qp,N}\le\frac{a^{p+q+1}}{(p+q+1)!},
        \qquad b_{00,N}=a.
    \]
    The coefficient matrices converge in Hilbert--Schmidt norm to
    \begin{equation}\label{eqn:fixed-k-initialization-and-gradient-limiting-gradient-coefficients}
        b_{qp}=
        \begin{cases}
            a^{p+q+1}/(p+q+1)!,&\alpha_N=1,\\
            a\,\1_{p=q=0},&\alpha_N>1.
        \end{cases}
    \end{equation}
    The Gaussian row limit applied to the two orthonormal families gives $\chi_\star\xi_0\sum_{p,q}b_{qp,N}E_qF_p$ as the fixed-$N$ width-limit law of each entry of $nN^{\alpha_N-1}\mG_{j,0}$. Orthogonality of the sources makes the standard Gaussian sequences $E,F$ independent even for diagonal entries. Using the same sequences for every $N$ gives
    \[
        \EE\left|\sum_{p,q}(b_{qp,N}-b_{qp})E_qF_p\right|^2
        =\sum_{p,q}|b_{qp,N}-b_{qp}|^2\to0.
    \]
    Write $b=(b_{qp})$ and define $(bF)_q=\sum_p b_{qp}F_p$ by its $L^2$-convergent Gaussian series.
    With $\norm{bF}_{\ell^2}^2=\sum_q|(bF)_q|^2$, we have $\EE\norm{bF}_{\ell^2}^2=\sum_{p,q}b_{qp}^2<\infty$.
    Thus $bF$ is square-summable almost surely, and the conditional variance is finite.
    Its zeroth component $\sum_pb_{0p}F_p$ is a nondegenerate Gaussian because $b_{00}=a>0$, so $\norm{bF}_{\ell^2}>0$ almost surely. The limiting entry is conditionally centered Gaussian with variance $(\chi_\star\xi_0)^2\norm{bF}_{\ell^2}^2>0$. It is tight and nonzero almost surely, proving the matching weak scale.
\end{proof}

\begin{proof}[Proof of \cref{clm:one-block-initialization-stability}]
    By \cref{lem:fixed-k-initialization-and-gradient}, initialization is stable exactly when $\alpha_N\ge1$ for every fixed finite block depth $k\ge1$.
    This proves the claim for the single-matrix model and its extension to all internal preactivations in \cref{app:fixed-k}.
\end{proof}

The initial gradient scale lets us test training stability at the first internal preactivation, where the incoming residual state is unchanged.

\begin{lemma}\label{lem:fixed-k-negative-gamma-instability}
    Assume $\alpha_N\ge1$ and $\gamma_N<0$. For a first update with a nondegenerate initial training signal, training stability fails: already $\Delta \vh_0^{1,1}$ is not $\bigO(1)$ in the ordered coordinatewise limit.
\end{lemma}

\begin{proof}
    Here $d_0=\eta\chi_0\xi_0\to d_\star\ne0$ in the width limit. At finite width, set
    \[
        \mJ_{1,N}^{(n)}\coloneqq\frac{a}{N}\sum_{s=0}^{N-1}
        \frac{\vb_{1,0,s}\vz_{1,0,s}^\transpose}{n}.
    \]
    With $\opB=\opC_k\cdots \opC_1$, $\opS_1=\opI$, $\opT_1=\opC_k\cdots \opC_2$, and $\opE_N=\opI+aN^{-\alpha_N}\opB$, its fixed-$N$ width-limit representative is
    \[
        \opJ_{1,N}=\frac{a}{N}\sum_{s=0}^{N-1}
        \opT_1^*(\opE_N^*)^{N-1-s}v_\star\otimes \opS_1\opE_N^su_\star.
    \]
    The exact update in \cref{eqn:fixed-k-identities-and-bounds-matrix-gradient-and-update-identities} gives
    \[
        \mW_1^1-\mW_0^1=-d_0N^{-\gamma_N}\mJ_{1,N}^{(n)}.
    \]
    By \cref{eqn:fixed-k-initialization-and-gradient-finite-loop-gradient-coefficients}, the coefficient of $\opT_1^*v_\star\otimes u_\star$ in $\opJ_{1,N}$ is $b_{00,N}=a$. All other all-creation coefficients are orthogonal, and therefore
    \begin{equation}\label{eqn:fixed-k-negative-gamma-instability-first-matrix-response-lower-bound}
        \norm*{\opJ_{1,N}u_\star}\ge a.
    \end{equation}
    On the fixed input $\zeta=1$, since $\vh_0^{1,1}=\mW_0^1\vu$ and $\vh_1^{1,1}=\mW_1^1\vu$, the finite-width update is
    \[
        \Delta \vh_0^{1,1}=-d_0N^{-\gamma_N}\mJ_{1,N}^{(n)}\vu.
    \]
    By \cref{lem:fixed-k-width-limit}, the normalized Gram norm of $\mJ_{1,N}^{(n)}\vu$ converges, with $N$ fixed, to $\norm{\opJ_{1,N}u_\star}$, and each fixed coordinate converges to the corresponding centered Gaussian limit. The random-rotation moment identities in \cref{eqn:fixed-k-width-limit-haar-coordinate-moments} and \cref{eqn:fixed-k-negative-gamma-instability-first-matrix-response-lower-bound} therefore imply a uniform Paley--Zygmund bound: there exist $c,p>0$, independent of $N$, such that
    \[
        \PP\set{|Z_{N,i}|\ge c}\ge p,
    \]
    where $Z_{N,i}$ denotes the fixed-coordinate width-limit variable represented by $\opJ_{1,N}u_\star$. Because $|d_\star|N^{-\gamma_N}\to\infty$, the coordinates of $\Delta \vh_0^{1,1}$ are not tight in the subsequent $N\to\infty$ limit. Thus training stability fails for every $\gamma_N<0$, independently of the fixed finite $k$.
\end{proof}

\begin{proof}[Proof of \cref{clm:fixed-k-training-and-nontriviality}]
    Within the initialization-stable and training-stable region, we prove that nontriviality holds if and only if
    \begin{equation}\label{eqn:fixed-k-training-classification-nontriviality-condition}
        \alpha_N+\gamma_N\le1.
    \end{equation}

    Initialization stability is equivalent to $\alpha_N\ge1$ by \cref{lem:fixed-k-initialization-and-gradient}. If $\gamma_N<0$, \cref{lem:fixed-k-negative-gamma-instability} supplies a nondegenerate first step whose first internal-preactivation change is not $\bigO(1)$. If $\alpha_N\ge1$ and $\gamma_N\ge0$, \cref{eqn:fixed-k-identities-and-bounds-training-change-bounds,eqn:fixed-k-width-limit-haar-coordinate-moments} give coordinatewise tightness of the residual-state changes, all internal-preactivation changes, and the output change, uniformly over loop position; compact truncation covers arbitrary fixed SGD routines. This proves the training-stability equivalence.

    Summing \cref{eqn:fixed-k-identities-and-bounds-training-change-bounds} over the fixed number of updates gives
    \[
        \left|f_t(\zeta)-f_0(\zeta)\right|
        =\bigO\left(N^{1-\alpha_N-\gamma_N}\right)
    \]
    on every fixed bounded input batch. Thus $\alpha_N+\gamma_N>1$ implies triviality. Within the stable region, $\alpha_N+\gamma_N\le1$ forces $(\alpha_N,\gamma_N)=(1,0)$, and \cref{eqn:fixed-k-canonical-response-output-response-limit} proves nontriviality there. This establishes \cref{eqn:fixed-k-training-classification-nontriviality-condition} and uniqueness.
\end{proof}
\begin{proof}[Proof of \cref{clm:fixed-k-loop-mup-and-feature-behavior}]
    \cref{clm:fixed-k-training-and-nontriviality} proves that $(1,0)$ is the unique stable and nontrivial point. We now prove its feature and diversity properties.

    At the critical point, \cref{lem:fixed-k-prefix-noncancellation} and compactness give constants $0<m\le M<\infty$, uniform in the loop fraction and in the finitely many matrix indices, such that each of $\norm{X_t(s)}$ and $\norm{\opR_{j,t}X_t(s)}$ lies in $[m,M]$ on every compact set of finite control transcripts. By \cref{eqn:fixed-k-width-limit-mixed-word-gram-limit,eqn:fixed-k-loop-limit-discrete-to-continuum-convergence}, these are the ordered normalized Gram norms of the finite-network residual states and internal preactivations. Their fixed coordinates are the corresponding Gaussian mixtures. For the fixed nondegenerate evaluation batch $\vzeta$, set $c_\zeta=\min_b|\zeta_b|>0$. Let $\fD_t=(d_0,\dots,d_{t-1})$ be the limiting scalar control vector and $\mathscr{D}_t=\sigma(\fD_t)$. On a fixed compact-control truncation, every such coordinate satisfies
    \[
        Z\mid\mathscr{D}_t\sim\fN\left(0,\sigma_Z^2(\fD_t)\right),
        \qquad \sigma_Z(\fD_t)\ge c_\zeta m
        \quad\text{a.s.}
    \]
    Consequently, for every $r>0$,
    \begin{equation}\label{eqn:fixed-k-training-classification-gaussian-anti-concentration-bound}
        \PP\left(\left|Z\right|\le r\given\mathscr{D}_t\right)
        \le\sqrt{\frac{2}{\pi}}\,\frac{r}{c_\zeta m}
        \quad\text{a.s.}
    \end{equation}
    Together with the uniform upper-tail bounds, truncation proves the weak $\bigTheta(1)$ scale for all routines. Equivalently, for each fixed routine and training step,
    \[
        \mH_t^{r,j}(\vzeta)=\bigTheta(1),
        \qquad
        \mX_t^r(\vzeta)=\bigTheta(1),
    \]
    uniformly over $1\le r\le N$ and the fixed set $1\le j\le k$. The first assertion proves faithfulness, while the second also bounds the residual states away from zero. Feature learning in both the residual states and every internal preactivation follows from \cref{pro:fixed-k-canonical-response} for the same first SGD step.

    Finally, Taylor expansion in the bounded limiting operator gives
    \begin{equation}\label{eqn:fixed-k-training-classification-loop-increment-expansion}
        \begin{aligned}
            X_t(\lambda+\varepsilon)-X_t(\lambda) &=a\varepsilon \opB_tX_t(\lambda)+E_t(\lambda,\varepsilon), \\
            \norm*{E_t(\lambda,\varepsilon)}&\le C\varepsilon^2.
        \end{aligned}
    \end{equation}
    The bound is uniform on compact controls and $0\le\lambda\le1-\varepsilon$. By \cref{lem:fixed-k-prefix-noncancellation}, $\norm{\opB_tX_t(\lambda)}\ge m>0$ on each such compact set. Thus, for sufficiently small $\varepsilon$,
    \[
        \norm{X_t(\lambda+\varepsilon)-X_t(\lambda)}
        \ge am\varepsilon-C\varepsilon^2\ge\frac{am}{2}\varepsilon.
    \]
    The bounded leading vector and the same remainder estimate give the matching $\bigO(\varepsilon)$ upper bound. The limits in \cref{eqn:fixed-k-width-limit-mixed-word-gram-limit,eqn:fixed-k-loop-limit-discrete-to-continuum-convergence} transfer this to the normalized finite-network batch norm; floor errors are $\bigO(N^{-1})$ and vanish before $\varepsilon\downarrow0$. Hence $\kappa_N=0$, and the matching upper bound at initialization excludes every $\kappa_N>0$.
\end{proof}

\subsection{Sharper nontriviality classification at block depth 1}\label{app:one-block-linear-classification}

For block depth one, we set $\mW_t=\mW_t^1$, $\vz_{t,s}=\vz_{1,t,s}$, and $\vb_{t,s}=\vb_{1,t,s}$ in the setup of \cref{def:fixed-k-linear-setup}. We rescale the shared matrix to control the output throughout the initialization-stable region, including $\gamma_N<0$, and use the exact first-step response for the lower bound.

\begin{lemma}\label{lem:one-block-rescaled-training-and-real-exponents}
    Suppose $\alpha_N \ge 1$, and allow any $\gamma_N \in \RR$. Put
    \begin{equation}
        \accN = N^{1-\alpha_N}, \qquad
        \respN = N^{1-\alpha_N-\gamma_N}.
    \end{equation}
    If $\alpha_N + \gamma_N > 1$, every fixed SGD routine is trivial:
    \begin{equation}\label{eqn:one-block-rescaled-training-and-real-exponents-trivial-output-bound}
        f_t(\zeta) - f_0(\zeta) = \bigO(\respN) \to 0.
    \end{equation}
    For a nondegenerate initial training signal, at every fixed $N$, the width-limit first output response is exactly
    \begin{equation}\label{eqn:one-block-rescaled-training-and-real-exponents-canonical-output-response}
        f_1(\zeta) - f_0(\zeta)
        = -\eta\chi_\star\xi_0\zeta\,\respN\normHS*{\opJ_N}^2,
    \end{equation}
    where
    \begin{equation}\label{eqn:one-block-rescaled-training-and-real-exponents-residual-and-response-operators}
        \begin{aligned}
            \opE_N &= \opI + \frac{a \accN}{N} \opC, \\
            \opJ_N &= \frac{a}{N} \sum_{s=0}^{N-1} (\opE_N^*)^{N-1-s} v_\star \otimes \opE_N^s u_\star.
        \end{aligned}
    \end{equation}
    Uniformly over $N$,
    \begin{equation}\label{eqn:one-block-rescaled-training-and-real-exponents-response-norm-bounds}
        a^2\le\normHS*{\opJ_N}^2\le a^2e^{4a},
        \qquad a\le\norm*{\opJ_Nu_\star}\le ae^{2a}.
    \end{equation}
    At $\zeta=\xi_0$ the coefficient is $-\eta\chi_\star\xi_0^2$. Consequently the model is nontrivial whenever $\alpha_N + \gamma_N \le 1$.
\end{lemma}

\begin{proof}
    We first prove the universal upper assertion, including points with $\gamma_N < 0$, for which the unscaled matrix norm need not be bounded. Define $\widetilde{\mW}_t = \accN \mW_t$. Then
    \begin{equation}\label{eqn:one-block-rescaled-training-and-real-exponents-rescaled-parameter-recursion}
        \mA_t=\mI+\frac{a}{N}\widetilde{\mW}_t,
        \qquad
        \widetilde{\mW}_{t+1}-\widetilde{\mW}_t
        =-\frac{ad_t\respN}{N}\sum_{s=0}^{N-1}\frac{\vb_{t,s}\vz_{t,s}^\transpose}{n}.
    \end{equation}
    Here $\vb, \vz, d$ are the exact finite-network quantities in \cref{eqn:fixed-k-identities-and-bounds-matrix-gradient-and-update-identities}. When $\alpha_N + \gamma_N \ge 1$, both $\accN$ and $\respN$ are at most one. On the invariant initialization and compact-control events used in \cref{lem:fixed-k-identities-and-bounds}, $\normop{\widetilde{\mW}_0} \le\normop{\mW_0}$. Let $D$ bound $\normrms{\vu}$ and $\normrms{\vv}$ on this event. If $\normop{\widetilde{\mW}_t} \le R_t$, then
    \[
        \begin{aligned}
            \normrms*{\vz_{t,s}},\normrms*{\vb_{t,s}} &\le D e^{aR_t}, \\
            \normop*{\widetilde{\mW}_{t+1} - \widetilde{\mW}_t} &\le a|d_t|\respN D^2 e^{2aR_t}.
        \end{aligned}
    \]
    Thus the same finite induction as in \cref{lem:fixed-k-identities-and-bounds} bounds all $\normop{\widetilde{\mW}_t}$ independently of $n, N$, with the update norm $\bigO(\respN)$. Telescoping the $N$ residual maps gives
    \[
        \normrms*{(\mA_{t+1}^r - \mA_t^r) \vu} \le C \respN\quad(r \le N).
    \]
    Pair with $\vv / n$, multiply by the bounded input, and sum over finitely many training steps to prove \cref{eqn:one-block-rescaled-training-and-real-exponents-trivial-output-bound}. Compact truncation handles the realized tight controls. This is a bound on the actual finite networks, so it respects the specified order of limits. It does not assert that the unscaled preactivations are bounded.
    For the lower bounds, fix $N$ first and let $d_\star=\eta\chi_\star\xi_0$. \cref{lem:fixed-k-width-limit} applies to any fixed finite coefficient, including $N^{-\gamma_N}$. Thus the width-limit initial operator is $\opC$, its residual map is $\opE_N$, and its first update is exactly
    \[
        \opB_{1,N} = \opC - d_\star N^{-\gamma_N} \opJ_N.
    \]
    The two source summands are orthogonal and preserved by $\opC, \opE_N$. Every range vector in $\opJ_N$ belongs to $\fH_v$, whereas every domain vector belongs to $\fH_u$. Hence $\opJ_N \opE_N^k \opJ_N = 0$ for every $k \ge 0$. Write $c_N = a \accN / N$. Expanding the residual power therefore yields the exact identity
    \begin{equation}\label{eqn:one-block-rescaled-training-and-real-exponents-exact-perturbed-power-response}
        (\opE_N - d_\star c_N N^{-\gamma_N} \opJ_N)^N - \opE_N^N
        = -d_\star c_N N^{-\gamma_N} \sum_{q=0}^{N-1}
        \opE_N^{N-1-q} \opJ_N \opE_N^q.
    \end{equation}
    There are no higher-order terms in the update. Pairing with the two Fock vacua, the sum on the right is
    \[
        \sum_{q=0}^{N-1}
        \langle(\opE_N^*)^{N-1-q} v_\star, \opJ_N \opE_N^q u_\star \rangle
        = \frac{N}{a}\normHS*{\opJ_N}^2.
    \]
    The initial scalar contraction is zero. Since $c_N N^{-\gamma_N} (N / a) = \respN$, this proves \cref{eqn:one-block-rescaled-training-and-real-exponents-canonical-output-response}, as an equality of fixed-$N$ width limits.
    For completeness, the bounds in \cref{eqn:one-block-rescaled-training-and-real-exponents-response-norm-bounds} require no continuum approximation. The vectors $\opE_N^s u_\star$ have nonnegative coefficients along $e_{(1,+)^m,u}$, with vacuum coefficient one; the backward vectors have the corresponding property along $e_{(1,-)^m,v}$. All pairwise inner products within either family are at least one. Expanding the squared Hilbert--Schmidt norm in \cref{eqn:one-block-rescaled-training-and-real-exponents-residual-and-response-operators} therefore gives
    \[
        \normHS*{\opJ_N}^2
        = \frac{a^2}{N^2} \sum_{s,q=0}^{N-1}
        \langle(\opE_N^*)^{N-1-s} v_\star, (\opE_N^*)^{N-1-q} v_\star \rangle
        \langle \opE_N^s u_\star, \opE_N^q u_\star \rangle \ge a^2.
    \]
    Also each vector in these two same-direction families has norm at most $(1 + c_N)^N \le e^a$, giving the asserted upper bound. Since $\langle \opE_N^s u_\star, u_\star \rangle = 1$ and every backward vacuum coefficient is one, $\langle v_\star, \opJ_N u_\star \rangle = a$. The upper vector bound follows from $\normop{\opJ_N} \le\normHS{\opJ_N}$. This proves \cref{eqn:one-block-rescaled-training-and-real-exponents-response-norm-bounds}.
    If $\alpha_N + \gamma_N < 1$, \cref{eqn:one-block-rescaled-training-and-real-exponents-canonical-output-response} tends in magnitude to infinity for a nondegenerate input. If $\alpha_N + \gamma_N = 1$, its magnitude is bounded below by $|d_\star\zeta|a^2$, so in either case it cannot tend to zero. One can also identify the limits on the equality line explicitly: for $\alpha_N = 1$, \cref{eqn:fixed-k-loop-limit-discrete-to-continuum-convergence} gives $\opJ_N \to \opG_0$; for $\alpha_N > 1$, $\accN \to 0$ and
    \[
        \begin{aligned}
            \sup_{s \le N}\norm*{\opE_N^s u_\star - u_\star} +\sup_{s \le N}\norm*{(\opE_N^*)^s v_\star - v_\star}&\to 0, \\
            \opJ_N &\to a\,v_\star \otimes u_\star.
        \end{aligned}
    \]
    For example, this follows by telescoping from $\normop{\opE_N^s - \opI}\le e^{2a \accN} - 1$. Thus the equality-line output is $-d_\star \zeta\normHS{\opG_0}^2$ at $\alpha_N = 1$, and $-d_\star \zeta a^2$ at $\alpha_N > 1$. This one-step routine satisfies the bounded-control assumptions for every exponent considered here.
\end{proof}

\begin{proof}[Proofs of \cref{clm:one-block-initialization-stability,clm:one-block-training-and-nontriviality,clm:one-block-loop-mup-and-diversity}]
    Initialization stability was proved for every fixed finite block depth in \cref{app:fixed-k-classification}. Setting $k=1$ there gives training stability exactly when $\gamma_N\ge0$; under $\alpha_N\ge1$, this also implies $\alpha_N+\gamma_N\ge1$.
    By \cref{lem:one-block-rescaled-training-and-real-exponents}, nontriviality holds throughout the initialization-stable region exactly when $\alpha_N+\gamma_N\le1$. Thus the unique stable and nontrivial point is $(\alpha_N,\gamma_N)=(1,0)$.
    Setting $k=1$ in \cref{pro:fixed-k-canonical-response,lem:fixed-k-prefix-noncancellation} gives feature learning in both residual states and preactivations, all-routine faithfulness, and feature-diversity exponent $\kappa_N=0$.
\end{proof}

\subsection{A block-depth-2 counterexample outside training stability}
\label{app:fixed-k-counterexample}

\paragraph{Counterexample setup.}
We construct an initialization-stable but training-unstable $k=2$ model whose output changes by a nonzero amount after two SGD updates. Take $a=\eta=1$, inputs $\xi_0=\xi_1=1$, losses $\ell_0(f)=\ell_1(f)=f$, and
\[
    \gamma_N<0,
    \qquad
    \alpha_N=1-3\gamma_N.
\]
Then $\alpha_N\ge1$, so initialization is stable, while $\gamma_N<0$ makes training unstable and $\alpha_N+\gamma_N=1-2\gamma_N>1$. The simplest instance is $(\alpha_N,\gamma_N)=(4,-1)$. Write
\[
    \upd=\updN=N^{-\gamma_N}\to\infty,
    \qquad
    \acc=\accN=N^{1-\alpha_N}=\upd^{-3}.
\]

After taking the width limit at fixed $N$, write the two matrix operators as $\opC_t$ and $\opD_t$, corresponding to $\mW_t^1$ and $\mW_t^2$, and set
\[
    u=u_\star,\quad v=v_\star,\quad \opC=\opC_0,\quad \opD=\opD_0,
    \quad p=\opC{}u,\quad q=\opD^*v,
    \quad \opR(y,z)=y\otimes z.
\]
The circular initialization gives $\norm{u}=\norm{v}=\norm{p}=\norm{q}=1$, with $u,p$ in the input source and $v,q$ in the orthogonal readout source. Hence $\langle p,q\rangle=\langle u,v\rangle=0$. All remainder estimates below are in operator norm after taking the width first, uniformly as $N\to\infty$.

The exact shared-gradient sums in \cref{eqn:fixed-k-identities-and-bounds-matrix-gradient-and-update-identities} give
\[
    \begin{aligned}
        \opC_{t+1}-\opC_t
        &=-\frac{\upd}{N}\sum_{r=0}^{N-1}
        \opR\left(\opD_t^*(\opA_t^*)^{N-1-r}v,\opA_t^ru\right),\\
        \opD_{t+1}-\opD_t
        &=-\frac{\upd}{N}\sum_{r=0}^{N-1}
        \opR\left((\opA_t^*)^{N-1-r}v,\opC_t\opA_t^ru\right),\\
        \opA_t&=\opI+\frac{\acc}{N}\opD_t\opC_t.
    \end{aligned}
\]

\paragraph{Two gradient steps.}
At initialization $\sup_{r\le N}\normop{\opA_0^r-\opI}=\bigO(\acc)$, also for adjoints. Replacing these powers by the identity therefore gives
\[
    \opC_1=\opC-\upd\opR(q,u)+\bigO(\upd^{-2}),
    \qquad \opD_1=\opD-\upd\opR(v,p)+\bigO(\upd^{-2}).
\]
The apparent quadratic term in their product is zero:
\[
    \upd^2\opR(v,p)\opR(q,u)
    =\upd^2\langle p,q\rangle \opR(v,u)=0.
\]
Hence $\normop{\opD_1\opC_1}=\bigO(\upd)$ and
$\sup_{r\le N}\normop{\opA_1^r-\opI}=\bigO(\acc\upd)=\bigO(\upd^{-2})$.
Both power estimates follow from
$\normop{(\opI+\acc \opB/N)^r-\opI}\le e^{\acc\normop{\opB}}-1$.
We now replace the powers in the second shared-gradient sums by the identity. Uniformly over $r$,
\[
    \begin{aligned}
        \norm*{\opA_1^ru-u}&=\bigO(\upd^{-2}),\\
        \norm*{\opD_1^*(\opA_1^*)^{N-1-r}v-\opD_1^*v}
        &=\bigO(\upd^{-1}),
    \end{aligned}
\]
because $\normop{\opD_1}=\bigO(\upd)$. Using
\[
    \normop*{\opR(y,z)-\opR(y',z')}
    \le\norm{y-y'}\norm{z}+\norm{y'}\norm{z-z'},
\]
we obtain $\bigO(\upd^{-1})$ error in each rank-one summand of the $\opC$ update. The same estimate applies to the $\opD$ update, with $\opC_1$ acting on the forward vector. Averaging preserves this bound, and multiplication by $\upd$ gives $\bigO(1)$ error in each update. Consequently
\[
    \opD_1^*v=q-\upd p+\bigO(\upd^{-2}),
    \qquad
    \opC_1u=p-\upd q+\bigO(\upd^{-2}),
\]
and
\[
    \begin{aligned}
        \opC_2&=\opC_1-\upd\opR(\opD_1^*v,u)+\bigO(1)
        =\opC-2\upd\opR(q,u)+\upd^2\opR(p,u)+\bigO(1),\\
        \opD_2&=\opD_1-\upd\opR(v,\opC_1u)+\bigO(1)
        =\opD-2\upd\opR(v,p)+\upd^2\opR(v,q)+\bigO(1).
    \end{aligned}
\]
In $\opD_2\opC_2$, the fourth-order term vanishes because
$\opR(v,q)\opR(p,u)=\langle q,p\rangle \opR(v,u)=0$.
The two cubic terms are
\[
    \begin{aligned}
        (-2\upd \opR(v,p))(\upd^2\opR(p,u))
        &=-2\upd^3\norm*{p}^2\opR(v,u),\\
        (\upd^2\opR(v,q))(-2\upd \opR(q,u))
        &=-2\upd^3\norm*{q}^2\opR(v,u).
    \end{aligned}
\]
All other explicit terms have degree at most two in $\upd$; multiplying an $\bigO(1)$ matrix remainder by the other matrix's $\bigO(\upd^2)$ leading part also gives $\bigO(\upd^2)$. Thus
\[
    \opD_2\opC_2=-4\upd^3\opR(v,u)+\bigO(\upd^2).
\]

\paragraph{Output response.}
Since $\acc\upd^3=1$ and $\acc\bigO(\upd^2)\to0$,
\[
    \acc \opD_2\opC_2\to-4\opR(v,u).
\]
After the second update, the loop operator is
\[
    \opA_2=\opI+\frac{\acc \opD_2\opC_2}{N}.
\]
The generators $\acc \opD_2\opC_2$ are uniformly bounded and converge in operator norm to $-4\opR(v,u)$. Applying \cref{eqn:fixed-k-loop-limit-euler-exponential-error} on this bounded operator set gives
\[
    \opA_2^N=e^{\acc \opD_2\opC_2}+\bigO(N^{-1})
    \to e^{-4\opR(v,u)}.
\]
Since $\opR(v,u)^2=\langle u,v\rangle \opR(v,u)=0$, the exponential is $\opI-4\opR(v,u)$. Pairing with the two unit vacua and using the earlier bound on $\opA_1^N-\opI$ gives
\[
    f_1(1)-f_0(1)=\bigO(\acc\upd)=\bigO(\upd^{-2})\to0,
    \qquad f_2(1)-f_0(1)\to-4.
\]
Thus the second update produces a nonzero limiting output change even though the first does not. This is the block-depth-$2$ example marked by the diamond in \cref{fig:loop-phase-diagrams}~(b): it has $\alpha_N+\gamma_N>1$ and is nontrivial outside training stability. The nontriviality statement in \cref{clm:fixed-k-training-and-nontriviality} must therefore remain restricted to the training-stable region.

\section{Proofs for the Infinite Loop-Depth Limit in the Linear Case}\label{app:loop-depth-proofs}

We prove \cref{clm:loop-depth-initialization-stability,clm:loop-depth-training-and-nontriviality,clm:loop-depth-distinguished-point} for linear networks trained with SGD.
We use the repeated-stack model of \cref{app:loop-depth-setup} and the batch conventions of \cref{app:loop-definitions}. The matrices are initialized independently across layers, but each layer uses the same trained matrix in every loop.

\begin{definition}[Linear repeated-stack model]\label{def:loop-depth-linear-setup}
    We consider an ordered stack of $L$ distinct linear residual blocks reused over $N$ loops, with identity activations and no mean subtraction.
    The branch constant $a>0$, base learning rate $\eta>0$, and finite training horizon $T$ are fixed, while $\alpha_N,\alpha_L,\gamma_N,\gamma_L\in\RR$.
    We take $n\to\infty$ at fixed $(L,N)$, then $L\to\infty$ at fixed $N$, and finally $N\to\infty$.
    \begin{itemize}[leftmargin=*]
        \item \textbf{Model.}
        The input and readout matrices $\mU,\mV\in\RR^{n\times1}$ are frozen, while the hidden matrices $\mW_t^l\in\RR^{n\times n}$, $1\le l\le L$, are trainable.
        All initial matrices are mutually independent, with $\mU_i\iid\fN(0,1)$, $\mV_i\iid\fN(0,n^{-2})$, and $(\mW_0^l)_{ij}\iid\fN(0,n^{-1})$.
        For a scalar input $\zeta$, the hidden states, preactivations, and output satisfy
        \[
            \begin{aligned}
                \vx_t^{1,0}(\zeta)&=\mU\zeta,\qquad
                \vx_t^{r,0}(\zeta)=\vx_t^{r-1,L}(\zeta)\quad(r\ge2),\\
                \vh_t^{r,l}(\zeta)&=\mW_t^l\vx_t^{r,l-1}(\zeta),\\
                \vx_t^{r,l}(\zeta)&=\vx_t^{r,l-1}(\zeta)+c\vh_t^{r,l}(\zeta),\\
                f_t(\zeta)&=\mV^\transpose\vx_t^{N,L}(\zeta).
            \end{aligned}
        \]
        Here $1\le r\le N$ and $1\le l\le L$.
        \item \textbf{Training.}
        For a training batch $\Xi_t$ with loss $\ell_t$, each block uses its complete gradient through all $N$ loops:
        \[
            \mG_t^l=\nabla_{\mW^l}\ell_t(\vf_t(\Xi_t)),
            \qquad
            \widetilde{\mG}_t^l=nN^{\alpha_N-1}L^{\alpha_L}\mG_t^l.
        \]
        All hidden matrices are updated simultaneously by
        \[
            \mW_{t+1}^l
            =\mW_t^l-\eta n^{-1}N^{-\gamma_N}L^{-\gamma_L}\widetilde{\mG}_t^l
            =\mW_t^l-\eta N^{\alpha_N-1-\gamma_N}L^{\alpha_L-\gamma_L}\mG_t^l.
        \]
        \item \textbf{Scales.}
        We write
        \begin{equation}
            \begin{aligned}
                c &= aN^{-\alpha_N}L^{-\alpha_L},
                \qquad \accNL=N^{1-\alpha_N}L^{1/2-\alpha_L}, \\
                \updNL &=N^{-\gamma_N}L^{-\gamma_L},
                \qquad \respNL=N^{1-\alpha_N-\gamma_N}L^{1-\alpha_L-\gamma_L}.
            \end{aligned}
        \end{equation}
        Thus $c\updNL=a\respNL/(NL)$, where $c$ is the branch multiplier.
        We use $\accNL$ for initialization accumulation, $\updNL$ for the local parameter-update scale, and $\respNL$ for the aggregate state/output response scale.
        \item \textbf{Conventions.}
        The routine data and all finite batch sizes are independent of network size.
        We write $\Xi_t$ for a training batch, $\vzeta$ for an evaluation batch, and $\vf_t(\Xi_t)$ for the corresponding training outputs.
        For a singleton batch, we write $\Xi_t=(\xi_t)$ and $\chi_t=\ell_t'(f_t(\xi_t))$.
        Changes compare the same input at the same network position before and after a training step.
    \end{itemize}
\end{definition}

\subsection{SGD training dynamics}\label{app:loop-depth-dynamics}

\paragraph{Widthwise limit.}

We set $\vu=\mU$ and $\vv=n\mV$.
The repeated stack combines independent initializations across layers with weight sharing across loops.
We preserve both structures by assigning one color to each initial layer matrix and reusing its operator in every loop.
Using the word representation in \cref{app:fixed-k-dynamics}, we take $H=\RR^L$ with orthonormal basis $e_1,\dots,e_L$.
We use two orthogonal copies $\fH_u,\fH_v$ of $\fF(H\oplus H)$ for the input and readout sources, with empty-word unit vectors $u_\star,v_\star$.
\begin{equation}\label{eqn:loop-depth-scales-and-operators-circular-system-definition}
    \opC(f)=\cre{f,+}+\ann{f,-},\qquad f\in H.
\end{equation}
The operators act identically on both sources and satisfy $\normop{\opC(f)}\le2\norm{f}$.
The Hilbert rank-one map $(z\otimes w)y=z\langle w,y\rangle$ corresponds to $\vz\vw^\transpose/n$ at finite width.
For fixed $L$, we use the limiting representatives
\begin{equation}\label{eqn:loop-depth-colored-width-limit-width-limit-substitutions}
    \begin{aligned}
        \mW_0^l &\longmapsto \opC(e_l) \eqqcolon \opC_l,
        \qquad \vu\longmapsto u_\star, \\
        \vv &\longmapsto v_\star,
        \qquad \vz\vw^\transpose/n\longmapsto z\otimes w,
    \end{aligned}
\end{equation}
where $e_1,\dots,e_L$ are orthonormal and $\longmapsto$ denotes representative correspondence rather than a mode of convergence.

\begin{lemma}[Width limit of the repeated stack]\label{lem:loop-depth-colored-width-limit}
    Fix $L,N,T$. Let $P,Q$ be finite products, possibly empty, of $\mW_0^1,\dots,\mW_0^L$ and their transposes, with lengths independent of $n$. In $P(\opC),Q(\opC)$, replace $\mW_0^l$ and $(\mW_0^l)^\transpose$ by $\opC_l$ and $\opC_l^*$, respectively. For $\ve,\ve'\in\set{\vu,\vv}$, use representatives $e_\star,e'_\star$ under $\vu\leftrightarrow u_\star$, $\vv\leftrightarrow v_\star$. Then
    \begin{equation}\label{eqn:loop-depth-colored-width-limit-mixed-word-gram-limit}
        \frac{1}{n}\langle P(\mW_0)\ve,Q(\mW_0)\ve'\rangle
        \pto
        \langle P(\opC)e_\star,Q(\opC)e'_\star\rangle.
    \end{equation}
    The same convergence holds jointly for Grams and scalar contractions in every fixed finite loop-depth SGD computation, uniformly on compact scalar controls. Along distributionally convergent subsequences of tight realized controls, any fixed finite collection of coordinate rows has limiting law
    \[
        \Law(Z_1,\dots,Z_p\mid\text{limiting controls})
        =\fN(0,\mSigma)^{\otimes p},
    \]
    where $Z_1,\dots,Z_p$ are the limiting rows and $\mSigma$ is the represented Gram matrix at the limiting controls. Uniform Hilbert-norm upper bounds give coordinatewise weak $\bigO(1)$ bounds; squared norms uniformly bounded away from zero on compact control sets also give \cref{eqn:coordinate-anti-concentration} and weak $\bigTheta(1)$ scales.
\end{lemma}

\begin{proof}
    We retain all $L$ colors in the trace argument of \cref{lem:fixed-k-width-limit}. For a word $\mM_1\cdots\mM_{2m}$, \cref{lem:aux-isserlis} pairs entries of the same layer color and identifies their ordered endpoints. Each covariance contributes $n^{-1}$, so a pairing with $V$ free indices contributes
    \[
        n^{-1}n^{-m}n^V=n^{V-m-1}.
    \]
    The connected quotient graph has at most $m$ edges, hence $V\le m+1$. Survival requires a tree whose closed walk traverses each edge twice in opposite directions and completes each descendant component before returning. This enforces opposite signs and noncrossing pairs. Conversely, removing adjacent opposite-sign, matching-color pairs deletes one leaf and edge at each step, giving contribution one. Writing $\fP_{\mathrm{nc}}$ for this family,
    \[
        \lim_{n\to\infty}\EE\left[\frac{1}{n}\Tr(\mM_1\cdots\mM_{2m})\right]
        =\#\fP_{\mathrm{nc}}.
    \]
    Other pairings vanish, and odd words have zero expectation. In the variance expansion of two such traces, disconnected pairings cancel against the product of expectations. Each connected pairing has $2m$ covariance pairs and at most $2m+1$ free vertices, contributing at most
    \[
        n^{V-2m-2}\le n^{-1}.
    \]
    Finiteness of the pairing count gives joint $L^2$ trace convergence. The circular operators and adjoints produce the same vacuum pairings through the relations
    \[
        \ann{e_l,\epsilon}\cre{e_j,\epsilon'}
        =\1_{l=j}\1_{\epsilon=\epsilon'}\opI,
        \qquad \epsilon,\epsilon'\in\set{+,-}.
    \]
    Thus contractions are nested or disjoint, pairing $\opC_l$ with $\opC_l^*$ of the same color. Trace and vacuum moments therefore agree.

    For a real matrix $\mM$ measurable with respect to the initial matrices, the independent Gaussian sources and \cref{lem:aux-isserlis} give
    \[
        \begin{aligned}
            \EE[\vu^\transpose\mM\vu/n\mid\mW_0^1,\dots,\mW_0^L]
            &=\Tr(\mM)/n, \\
            \Var(\vu^\transpose\mM\vu/n\mid\mW_0^1,\dots,\mW_0^L)
            &\le 2\normop*{\mM}^2/n, \\
            \EE[\vu^\transpose\mM\vv/n\mid\mW_0^1,\dots,\mW_0^L]
            &=0, \\
            \EE[|\vu^\transpose\mM\vv/n|^2\mid\mW_0^1,\dots,\mW_0^L]
            &=\normfro*{\mM}^2/n^2
            \le\normop*{\mM}^2/n.
        \end{aligned}
    \]
    A $1/4$-net of the unit sphere and a Gaussian union bound give, for fixed $L$,
    \[
        \PP\left(\max_{1\le l\le L}\normop*{\mW_0^l}>R\right)
        \le 2L\exp(2n\log9-nR^2/8).
    \]
    The finitely many initial operator norms and source RMS norms therefore have bounded fixed moments. Taking $\mM=P(\mW_0)^\transpose Q(\mW_0)$ gives vanishing $L^2$ fluctuations; the trace limits prove \cref{eqn:loop-depth-colored-width-limit-mixed-word-gram-limit} and source orthogonality.

    At fixed $L,N,T$ and controls, each complete gradient is a finite sum of normalized rank-one maps, represented by $(z\otimes w)y=z\langle w,y\rangle$. Matrix and transpose multiplication, linear combinations, and simultaneous updates preserve finite expansions
    \[
        \vz=\sum_{q=1}^M\theta_q P_q(\mW_0)\ve_q,
        \qquad \ve_q\in\set{\vu,\vv}.
    \]
    Polynomial coefficients and contractions give compact-uniform convergence, extended to tight controls by truncation and subsequences. For coordinates, one independent Haar matrix $\mQ$ rotates both sources and conjugates every initial matrix, preserving the law, Grams and invariant controls. By \cref{lem:aux-haar}~\cref{it2:aux-haar}, the conditional row approximation is uniform on bounded, possibly singular Gram sets. Gram convergence identifies its covariance along distributionally convergent control subsequences, giving the stated conditional independence. The identity in \cref{it3:aux-haar},
    \[
        \EE_{\mQ}|(\mQ\vz)_i|^2=\normrms{\vz}^2,
    \]
    transfers uniform Hilbert-norm bounds to coordinate tightness. Uniformly positive squared norms on compact controls give \cref{eqn:coordinate-anti-concentration} by \cref{eqn:conditional-gaussian-anti-concentration}; truncation gives the asserted weak scales. The width limit is taken at fixed $L,N,T$, before depth and loop limits.
\end{proof}

\paragraph{Depthwise limit.}
We set $\cycNL=c\sqrt{L}=(a/N)\accNL$ to control initialization products within one cycle.

\begin{lemma}\label{lem:loop-depth-analytic-products}
    For a coefficient tensor $A=(A_{l_1,\dots,l_k})$, we use $\normHS{A}^2=\sum_{l_1,\dots,l_k}|A_{l_1,\dots,l_k}|^2$. The homogeneous analytic polynomial
    \[
        \sum_{l_1,\dots,l_k=1}^L A_{l_1,\dots,l_k}\opC_{l_1}\cdots \opC_{l_k}
    \]
    has operator norm at most $(k+1)\normHS{A}$. Every ordered segment within one cycle of the maps $\opI+c\opC_l$ has operator norm at most
    \begin{equation}\label{eqn:loop-depth-analytic-products-analytic-product-majorant}
        F(\cycNL) \coloneqq
        \sum_{k \ge 0}\frac{(k + 1)\cycNL^k}{\sqrt{k!}}.
    \end{equation}
    Moreover, $F(\cycNL)=1+\bigO(\cycNL)$ as $\cycNL\downarrow0$, and every segment spanning at most $N$ loops has operator norm at most $F(\cycNL)^{N+2}$.
    If $\alpha_L>1/2$, this majorant is $1+\littleo_L(1)$ for each fixed $N$. If $\alpha_L=1/2$ and $\alpha_N\ge1$, it is bounded uniformly in $N$. Thus it is bounded throughout the initialization-stable region when $L\to\infty$ precedes $N\to\infty$.
\end{lemma}

\begin{proof}
    Since $\ann{e_l,-}\cre{e_j,+}=0$, a nonzero term in an analytic word has all creations to the left of all annihilations. We write the polynomial as $\sum_{j=0}^k \opT_j(A)$, where $\opT_j(A)$ contains $j$ creations and $k-j$ annihilations.
    We flatten $A$ into a matrix $\opA^{(j)}$ from the annihilated tensor block to the created tensor block, reversing the annihilated indices. On each remaining suffix, $\opT_j(A)$ acts as $\opA^{(j)}$ tensored with the identity and vanishes on words shorter than $k-j$. Hence
    \[
        \begin{aligned}
            \normop*{\opT_j(A)}&\le\normop*{\opA^{(j)}}
            \le\normHS*{\opA^{(j)}}=\normHS{A},\\
            \normop*{\sum_{j=0}^k\opT_j(A)}&\le(k+1)\normHS{A}.
        \end{aligned}
    \]
    Finite tensor approximations extend the estimate to coefficient tensors over any real Hilbert space.
    The degree-$k$ term of a segment with $m \le L$ layers has $\binom{m}{k}$ distinct orthonormal word coefficients, each $c^k$. Its coefficient norm is $c^k \sqrt{\binom{m}{k}} \le \cycNL^k / \sqrt{k!}$, proving \cref{eqn:loop-depth-analytic-products-analytic-product-majorant}. Split a longer segment at cycle boundaries. When $\alpha_L > 1 / 2$, $\cycNL \to 0$ first for every fixed $N$, so the ordered bound is one. When $\alpha_L = 1 / 2, \alpha_N \ge 1$, $\cycNL = aN^{-\alpha_N}$ and $N\cycNL \le a$. The bound follows from $\log F(\cycNL)=\bigO(\cycNL)$ near zero and the finitely many remaining $N$.
\end{proof}

For the width-limit computation, let $z_{t,r,l}$ be the forward vector at input one immediately before block application $(r,l)$, and let $w_{t,r,l}$ be the transpose product of all later residual maps applied to $v_\star$.

For a finite training batch, we use the linear control $d_t=\eta\sum_b\xi_{t,b}\partial_b\ell_t$, with derivatives evaluated at the batch outputs; for a singleton batch, $d_t=\eta\chi_t\xi_t$.
The upper bounds below use a fixed finite horizon and bounded inputs and loss-derivative controls, with estimates uniform on compact control sets and transferred to tight realized controls by truncation.

\begin{lemma}\label{lem:loop-depth-controlled-training-bounds}
    In the initialization-stable region, the layer operators and responses satisfy exactly
    \begin{equation}\label{eqn:loop-depth-controlled-training-bounds-layer-response-recursion}
        \begin{aligned}
            \opW_t^l &= \opC_l + \updNL \opK_t^l, \qquad \opK_0^l = 0, \\
            \opJ_t^l &= \frac{a}{N} \sum_{r=1}^N w_{t,r,l} \otimes z_{t,r,l}, \\
            \opK_{t+1}^l &= \opK_t^l - d_t \opJ_t^l,
        \end{aligned}
    \end{equation}
    and every residual map is
    \begin{equation}\label{eqn:loop-depth-controlled-training-bounds-trained-residual-map}
        \opI + c\opC_l + \frac{a\respNL}{NL} \opK_t^l.
    \end{equation}
    Fix a finite horizon, compact scalar controls, and bounded evaluation inputs. Bounded $\respNL$ in the ordered limit gives uniform Hilbert--Schmidt bounds on the fields and Hilbert-norm bounds on forward and backward vectors over all times, layers and loops. Residual-state and output changes are $\bigO(\respNL)$; bounded $\updNL$ also gives $\bigO(1)$ preactivations and changes. Truncation transfers these bounds to tight realized controls, and $\respNL\to0$ makes every fixed routine with such controls trivial.
\end{lemma}

\begin{proof}
    With $\vz,\vw$ the finite-width vectors, the exact gradient is
    \begin{equation}\label{eqn:loop-depth-controlled-training-bounds-layer-gradient-identity}
        \mG_t^l = \frac{c d_t}{\eta n} \sum_r \vw_{t,r,l} \vz_{t,r,l}^\transpose.
    \end{equation}
    Substituting into the SGD update gives
    \[
        \mW_{t+1}^l-\mW_t^l=-\frac{ad_t\updNL}{N}\sum_{r=1}^N
        \frac{\vw_{t,r,l}\vz_{t,r,l}^\transpose}{n}.
    \]
    The width correspondence and $c\updNL=a\respNL/(NL)$ give \cref{eqn:loop-depth-controlled-training-bounds-layer-response-recursion,eqn:loop-depth-controlled-training-bounds-trained-residual-map}.

    Let $\opPi_R$ perturb $\opPi_0$ by insertions $\opR_j$, where every unperturbed subproduct has norm at most $M\ge1$ and $\sum_j\normop{\opR_j}\le A$. For $k$ insertions,
    \[
        \sum_{j_1<\cdots<j_k}\prod_{q=1}^k\normop{\opR_{j_q}}
        \le\frac{A^k}{k!}.
    \]
    With at most $k+1$ unperturbed subproducts per term,
    \[
        \normop{\opPi_R}\le\sum_{k\ge0}M^{k+1}\frac{A^k}{k!}=Me^{MA},
        \qquad
        \normop*{\opPi_R-\opPi_0}\le M(e^{MA}-1).
    \]
    The bounds also apply to subproducts and adjoints. For \cref{eqn:loop-depth-controlled-training-bounds-trained-residual-map}, take $R_t\ge\max_l\normHS{\opK_t^l}$ and $A\le a\respNL R_t$. With $M$ from \cref{lem:loop-depth-analytic-products}, forward and backward norms are at most $M\exp(Ma\respNL R_t)$, giving
    \begin{equation}\label{eqn:loop-depth-controlled-training-bounds-response-and-field-bounds}
        \begin{aligned}
            \max_l\normHS*{\opJ_t^l} &\le aM^2 \exp(2Ma\respNL R_t), \\
            R_{t+1} &\le R_t + | d_t |aM^2 \exp(2Ma\respNL R_t).
        \end{aligned}
    \end{equation}
    Starting from $R_0=0$, finite induction gives uniform field and path bounds for bounded $\respNL$ and compact controls.

    Telescoping products before and after an update, we use
    \[
        \sum_{r=1}^N\sum_{l=1}^L\normop*{c(\opW_{t+1}^l-\opW_t^l)}
        \le a\respNL|d_t|\max_l\normop{\opJ_t^l}.
    \]
    Bounded surrounding subproducts give $\bigO(\respNL)$ changes in states and outputs, even for unbounded $\updNL$. For incoming vectors,
    \[
        \begin{aligned}
            \opW_{t+1}^lz_{t+1}-\opW_t^lz_t
            &=\opW_{t+1}^l(z_{t+1}-z_t)-d_t\updNL \opJ_t^lz_t,\\
            \normop{\opW_{t+1}^l}&\le2+\updNL R_{t+1}.
        \end{aligned}
    \]
    so bounded $\updNL$ also bounds preactivations and their changes. By \cref{lem:loop-depth-colored-width-limit} and truncation, these norm bounds give coordinatewise weak bounds for tight realized controls. For $1\le t\le T$,
    \[
        |f_t(\zeta)-f_0(\zeta)|\le\sum_{q=0}^{t-1}|\Delta f_q(\zeta)|
        =\bigO(\respNL).
    \]
    Thus $\respNL\to0$ makes every fixed routine with such controls trivial.
\end{proof}

\paragraph{Loopwise limit.}

The product bounds allow us to place the width representations for all $L$ in one Hilbert space. At $(\alpha_N,\alpha_L;\gamma_N,\gamma_L)=(1,1/2;0,1/2)$, we take $H=L^2[0,1]$ and write $\1=\1_{[0,1]}$. We represent layer $l$ by
\[
    e_{L,l}=\sqrt{L}\,\1_{((l-1)/L,l/L]},\qquad \opC_l=\opC(e_{L,l}).
\]
These vectors are orthonormal, so the embedding preserves the fixed-$L$ representation.
We interpolate the fields in \cref{eqn:loop-depth-controlled-training-bounds-layer-response-recursion} as cellwise constant functions $\opK_{t,L,N}(\lambda),\opJ_{t,L,N}(\lambda)$. We write $x_t^{r,l}(\zeta)$ and $h_t^{r,l}(\zeta)$ for the Hilbert representatives of $\vx_t^{r,l}(\zeta)$ and $\vh_t^{r,l}(\zeta)$, respectively, and $\opP_{t,N}(s_1,s_0)$ for the depth-limit residual product from $s_0$ to $s_1$, with $\opP_{t,N}=\opP_{t,N}(1,0)$. For a bounded operator $\opK$, we set $\opB(\opK)=\opC(\1)+\opK$ and define its forward and backward paths and gradient by
\[
    \begin{aligned}
        X_{\opK}(s)&=e^{as\opB(\opK)}u_\star,
        \qquad Y_{\opK}(s)=e^{a(1-s)\opB(\opK)^*}v_\star, \\
        \opG(\opK)&=a\int_0^1Y_{\opK}(s)\otimes X_{\opK}(s)\dif s.
    \end{aligned}
\]

\begin{lemma}\label{lem:loop-depth-continuum-cycle-products}
    At the distinguished point, fix a finite horizon $T$ and real controls $(d_0,\dots,d_{T-1})$. For fixed $N$, the following Hilbert--Schmidt limits are continuous in layer position, with convergence uniform in position and the indicated times:
    \[
        \begin{aligned}
            \opK_{t,L,N}&\xrightarrow[L\to\infty]{\mathrm{HS}}\opK_{t,N},
            &&0\le t\le T,\\
            \opJ_{t,L,N}&\xrightarrow[L\to\infty]{\mathrm{HS}}\opJ_{t,N},
            &&0\le t<T.
        \end{aligned}
    \]
    For $0\le t\le T$, segment propagators converge in operator norm uniformly in endpoints, and forward and backward vectors converge in Hilbert norm uniformly over layer and loop positions. There are bounded Hilbert--Schmidt loop-limit operators $\opK_t^{\infty}$, $0\le t\le T$. Writing $\opB_t=\opB(\opK_t^{\infty})$, $X_t=X_{\opK_t^{\infty}}$, $Y_t=Y_{\opK_t^{\infty}}$, and $\opG_t=\opG(\opK_t^{\infty})$, for $0\le t<T$ we have
    \[
        \sup_{\lambda\in[0,1]}
        \left(
        \normHS*{\opK_{t,N}(\lambda)-\opK_t^{\infty}}
        +\normHS*{\opJ_{t,N}(\lambda)-\opG_t}
        \right)
        =\bigO(N^{-1}).
    \]
    \begin{equation}\label{eqn:loop-depth-continuum-cycle-products-effective-generator-recursion}
        \opK_0^{\infty} = 0, \qquad \opK_{t+1}^{\infty} = \opK_t^{\infty} - d_t \opG_t.
    \end{equation}
    The estimate for $\opK_{t,N}$ also holds at $t=T$. For $0\le t\le T$, uniformly in $1\le r\le N$ and $\lambda\in[0,1]$,
    \begin{equation}\label{eqn:loop-depth-continuum-cycle-products-layer-path-limit}
        \begin{aligned}
            \lim_{L \to \infty} x_t^{r,\floor{\lambda L}}(1) 
            &= \opP_{t,N} (\lambda, 0) \opP_{t,N}^{r-1} u_\star \\
            &= X_t ((r - 1) / N) + \bigO(N^{-1}).
        \end{aligned}
    \end{equation}
    The corresponding later-map backward vector equals $Y_t ((r - 1) / N) + \bigO(N^{-1})$. Moreover
    \begin{equation}\label{eqn:loop-depth-continuum-cycle-products-segment-first-order-expansion}
        \opP_{t,N} (s_1, s_0) = \opI + \frac{a}{N}
        \left[\opC(\1_{(s_0,s_1)}) + \int_{s_0}^{s_1} \opK_{t,N} (s)\dif s\right]
        +\bigO(N^{-2}),
    \end{equation}
    uniformly in endpoints. All estimates are uniform on compact control sets and transfer to tight realized controls along convergent control subsequences.
\end{lemma}

\begin{proof}
    \proofparagraph{Depthwise limit at fixed $N$.}
    Write $A_{j,L}$ for the degree-$j$ coefficient tensor of an initialization product over an interval. 
    Each block is $\opI+(a/N)\opC(\1_{((l-1)/L,l/L]})$, so $A_{j,L}$ is $(a/N)^j$ times the indicator of strictly ordered cell indices. 
    In operator multiplication order, we define
    \[
        A_j(\tau_1,\dots,\tau_j)
        =\left(\frac{a}{N}\right)^j
        \1_{\set{s_0<\tau_j<\cdots<\tau_1<s_1}}.
    \]
    The indicators differ only on repeated or endpoint cells, a set of measure $\bigO_j(L^{-1})$ uniformly in endpoints. 
    Hence
    \[
        \begin{aligned}
            \norm*{A_{j,L}-A_j}_{L^2}
            &=\left(\frac{a}{N}\right)^j\bigO_j(L^{-1/2})\to0,\\
            \norm*{A_{j,L}}_{L^2},\norm*{A_j}_{L^2}
            &\le\frac{(a/N)^j}{\sqrt{j!}}.
        \end{aligned}
    \]
    By \cref{lem:loop-depth-analytic-products}, the summable majorant $(j+1)(a/N)^j/\sqrt{j!}$ gives uniform operator-norm convergence to $\opS_{a/N}(s_1,s_0)$. Simplex convergence and finite-product concatenation give endpoint continuity and the concatenation identity, with
    \begin{equation}\label{eqn:loop-depth-continuum-cycle-products-base-segment-expansion}
        \begin{aligned}
            \opS_{a/N} (s_1, s_0) &= \opI + \frac{a}{N}\opC(\1_{(s_0,s_1)}) + \bigO(N^{-2}), \\
            \sup_{s_0 \le s_1}\normop*{\opS_{a/N} (s_1, s_0) - \opI}&= \bigO(N^{-1}).
        \end{aligned}
    \end{equation}
    Training adds $(a/(NL))\opK_t^l$ to each block. For a bounded continuous Hilbert--Schmidt field $\opK$, set $\tau_0\coloneqq s_0$ and define the drifted product by the norm-convergent insertion series
    \begin{equation}\label{eqn:loop-depth-continuum-cycle-products-drifted-segment-insertion-series}
        \begin{aligned}
            \opP_{a/N,\opK} (s_1, s_0) &= \opS_{a/N} (s_1, s_0) \\
            &\quad + \sum_{j \ge 1} \left(\frac{a}{N}\right)^j
            \int_{s_0 < \tau_1 < \cdots < \tau_j < s_1}
            \opS_{a/N} (s_1, \tau_j) \opK(\tau_j) \\
            &\qquad \opS_{a/N} (\tau_j, \tau_{j-1}) \cdots
            \opK(\tau_1) \opS_{a/N} (\tau_1, s_0)\dif \tau_1 \cdots \dif \tau_j.
        \end{aligned}
    \end{equation}
    If $M$ bounds initialization segments and $R$ bounds $\opK$, the $j$-th insertion term $\opA_j$ satisfies
    \[
        \normop{\opA_j}\le M^{j+1}(aR/N)^j/j!.
    \]
    Each fixed-order insertion sum is a Riemann sum with endpoint shifts of at most one cell. 
    This summable bound gives uniform convergence for uniformly convergent fields, norm continuity in endpoints and in the field under $\sup_{\lambda\in[0,1]}\normop{\opK(\lambda)}$, and uniformity on compact field families.
    Starting from $\opK_{0,L,N}=0$, suppose the fields converge uniformly in Hilbert--Schmidt norm to continuous $\opK_{t,N}$. 
    The product construction gives uniformly convergent forward and backward paths.
    The rank-one estimate
    \[
        \normHS{z\otimes w-z'\otimes w'}
        \le\norm{z-z'}\norm{w}+\norm{z'}\norm{w-w'}
    \]
    then gives uniform Hilbert--Schmidt convergence of the finite sum defining $\opJ_{t,L,N}$. 
    By \cref{eqn:loop-depth-controlled-training-bounds-layer-response-recursion}, the next parameter field converges to
    \[
        \opK_{t+1,N}(\lambda)=\opK_{t,N}(\lambda)-d_t\opJ_{t,N}(\lambda).
    \]
    Finite induction and the product construction at $t=T$ prove the depth-limit assertions; continuity and \cref{eqn:loop-depth-controlled-training-bounds-response-and-field-bounds} give compact-control uniformity.

    \proofparagraph{Loopwise limit.} 
    With $\respNL=1$ and $\cycNL=a/N$, \cref{eqn:loop-depth-controlled-training-bounds-response-and-field-bounds} is uniform in $N$. For segment and field bounds $M,R$, terms with at least two insertions have total norm at most
    \[
        M\sum_{j\ge2}\frac{(MaR/N)^j}{j!}=\bigO(N^{-2}).
    \]
    For one insertion, \cref{eqn:loop-depth-continuum-cycle-products-base-segment-expansion} gives
    \[
        \frac{a}{N}\int_{s_0}^{s_1}
        \opS_{a/N}(s_1,s)\opK_{t,N}(s)\opS_{a/N}(s,s_0)\dif s
        =\frac{a}{N}\int_{s_0}^{s_1}\opK_{t,N}(s)\dif s+\bigO(N^{-2}).
    \]
    Together with the initialization expansion, this proves \cref{eqn:loop-depth-continuum-cycle-products-segment-first-order-expansion} uniformly in endpoints. We induct from $\opK_{0,N}=\opK_0^{\infty}=0$. If $\opK_{t,N}=\opK_t^{\infty}+\bigO(N^{-1})$, that expansion gives
    \begin{equation}\label{eqn:loop-depth-continuum-cycle-products-cycle-first-order-expansion}
        \opP_{t,N}=\opI+\frac{a}{N}\opB_t+\bigO(N^{-2}),
        \qquad \opP_{t,N}(s_1,s_0)=\opI+\bigO(N^{-1}).
    \end{equation}
    Taylor expansion gives one-cycle error $\bigO(N^{-2})$ relative to $e^{a\opB_t/N}$. Telescoping at most $N$ factors with bounded surrounding subproducts yields, uniformly for $1\le r\le N$,
    \begin{equation}\label{eqn:loop-depth-continuum-cycle-products-cycle-exponential-approximation}
        \normop*{\opP_{t,N}^r - e^{ar\opB_t/N}}= \bigO(N^{-1}).
    \end{equation}
    Prefix and suffix bounds in \cref{eqn:loop-depth-continuum-cycle-products-cycle-first-order-expansion} and their adjoints give \cref{eqn:loop-depth-continuum-cycle-products-layer-path-limit} and its backward counterpart, which excludes the current block. Substitution into the gradient recursion gives, uniformly in $\lambda$ in Hilbert--Schmidt norm,
    \[
        \begin{aligned}
            \opJ_{t,N}(\lambda)
            &=\frac{a}{N}\sum_{r=1}^N
            Y_t((r-1)/N)\otimes X_t((r-1)/N)+\bigO(N^{-1})\\
            &=a\int_0^1Y_t(s)\otimes X_t(s)\dif s+\bigO(N^{-1})
            =\opG_t+\bigO(N^{-1}).
        \end{aligned}
    \]
    The errors follow from the rank-one estimate and the Lipschitz paths. The update recursions give
    \[
        \begin{aligned}
            \sup_\lambda\normHS*{\opK_{t+1,N}(\lambda)-\opK_{t+1}^{\infty}}
            &\le\sup_\lambda\normHS*{\opK_{t,N}(\lambda)-\opK_t^{\infty}}\\
            &\quad+|d_t|\sup_\lambda\normHS*{\opJ_{t,N}(\lambda)-\opG_t}
            =\bigO(N^{-1}).
        \end{aligned}
    \]
    After the final update, the established terminal Hilbert--Schmidt estimate feeds into the same segment and cycle bounds, yielding the forward and later-map backward path limits at $t=T$. This closes the induction. Compact truncation gives the ordered limits for tight realized controls as in \cref{lem:loop-depth-colored-width-limit}.
\end{proof}

The normalized sum of the layer colors is the constant depth direction $\1$.
The continuum loop recursion preserves the Fock subspace $\fH_0\subset\fH_u\oplus\fH_v$ generated by $(\1,+),(\1,-)$ in both sources.
On this subspace, the dynamics have the single-matrix form in \cref{eqn:fixed-k-loop-limit-continuum-controlled-recursion}.
We use this structure to obtain lower bounds for both the state norm and the loop derivative $X_t'(s)=a\opB_tX_t(s)$.

\begin{lemma}[Nonvanishing states and loop derivatives]\label{lem:loop-depth-effective-generator-noncancellation}
    For every finite horizon $T$, every real control transcript $(d_0,\dots,d_{T-1})$ in \cref{eqn:loop-depth-continuum-cycle-products-effective-generator-recursion}, and every $0\le t\le T$,
    \begin{equation}\label{eqn:loop-depth-effective-generator-noncancellation-generator-noncancellation-and-bounds}
        \opB_tu_\star\neq0,
        \qquad 0<m\le\norm*{\opB_tX_t(s)}\le M<\infty\quad(0\le s\le1),
    \end{equation}
    where the two bounds are uniform on compact control sets. Also $\norm{X_t (s)}$ is bounded above and away from zero on such sets.
\end{lemma}

\begin{proof}
    The constant-color space $\fH_0$ is preserved by $\opC(\1)$ and its adjoint. Induction in \cref{eqn:loop-depth-continuum-cycle-products-effective-generator-recursion} keeps both paths in $\fH_0$, while each $\opK_t^{\infty}$ has range there and annihilates $\fH_0^\perp$. Restricted to $\fH_0$, the recursion is exactly \cref{eqn:fixed-k-loop-limit-continuum-controlled-recursion} with $k=1$ and $\opC_1=\opC(\1)$.
    In this restriction, $\opR_{0,t}X_t(s)=X_t(s)$ and $\opR_{1,t}X_t(s)=\opB_tX_t(s)$.
    Applying \cref{lem:fixed-k-prefix-noncancellation} therefore gives both positive lower bounds, uniformly on compact control sets.
\end{proof}

\subsection{Classification of loop-depth parameterizations}\label{app:loop-depth-linear-classification}

The bounds in \cref{lem:loop-depth-controlled-training-bounds} separate the local-update scale $\updNL$ from the aggregate-response scale $\respNL$. We now express their boundedness conditions in terms of the exponents.

Because $L\to\infty$ precedes $N\to\infty$, a power $N^{u_N}L^{u_L}$ is bounded exactly when $(u_L,u_N)\preceq_{\lex}(0,0)$: we compare $u_L$ first and use $u_N$ to break ties. We write $\succeq_{\lex}$ for the reverse order.

\begin{lemma}\label{lem:loop-depth-initialization-classification}
    Initialization stability is equivalent to
    \begin{equation}\label{eqn:loop-depth-initialization-classification-initialization-stability-condition}
        (1 / 2 - \alpha_L, 1 - \alpha_N) \preceq _{\lex}(0, 0).
    \end{equation}
    In this region the initialized output has width limit zero. Moreover, within this region, a nonzero limiting initialization loop-endpoint increment is possible only at $\alpha_L = 1 / 2, \alpha_N = 1$.
\end{lemma}

\begin{proof}
    Under \cref{eqn:loop-depth-initialization-classification-initialization-stability-condition}, \cref{eqn:loop-depth-analytic-products-analytic-product-majorant} bounds all hidden states. The inequality $\normop{\opC_l}\le 2$ bounds all preactivations, and the scalar output is bounded by the forward and readout norms. \cref{lem:loop-depth-colored-width-limit} transfers these bounds to coordinates. The two source spaces are preserved by each initial map, so the output contraction is zero in the width limit.
    For necessity, we take input one. Every analytic word applied to $u_\star$ consists only of plus creations. The degree-one component of the final state is
    \[
        Nc \sum_{l = 1}^L \opC_l u_\star.
    \]
    It is orthogonal to all other degrees, and the vectors $\opC_lu_\star$ are orthonormal. Thus the final-state coordinate variance satisfies
    \[
        \sigma_{L,N}^2\coloneqq\norm*{x_0^{N,L}(1)}^2
        \ge N^2Lc^2=a^2\accNL^2.
    \]
    By \cref{lem:loop-depth-colored-width-limit}, the fixed-coordinate width law is $Z_{L,N}\sim\fN(0,\sigma_{L,N}^2)$. For every $M>0$,
    \[
        \PP(|Z_{L,N}|\le M)\le\sqrt{\frac{2}{\pi}}\frac{M}{a\accNL}.
    \]
    If $\alpha_L<1/2$, this bound tends to zero as $L\to\infty$ at fixed $N$. If $\alpha_L=1/2$ and $\alpha_N<1$, it tends to zero in the subsequent $N$-limit. Initialization stability therefore requires \cref{eqn:loop-depth-initialization-classification-initialization-stability-condition}.
    If $\alpha_L > 1 / 2$, every subpath tends in operator norm to the identity at fixed $N$: apply \cref{eqn:loop-depth-analytic-products-analytic-product-majorant} to each one-cycle map, also bounding its difference from the identity by $F(\cycNL)-1$, and telescope the finitely many loops. If $\alpha_L = 1 / 2, \alpha_N > 1$, the same telescoping bound is $\bigO(N\cycNL) \to 0$. Thus all initialization endpoint increments vanish in these cases. The remaining case is analyzed explicitly below.
\end{proof}

For the response lower bounds, we use a singleton first SGD step with fixed $\xi_0\ne0$ and $\chi_0\to\chi_\star\ne0$, and set $d_\star=\eta\chi_\star\xi_0$ independently of $L,N$.

\begin{lemma}\label{lem:loop-depth-canonical-response-and-phase}
    Assume initialization stability. Training stability holds exactly when both $\updNL$ and $\respNL$ are bounded in the ordered limit, equivalently when
    \begin{equation}\label{eqn:loop-depth-canonical-response-and-phase-training-stability-conditions}
        (-\gamma_L,-\gamma_N)\preceq_{\lex}(0,0),
        \qquad
        (1-\alpha_L-\gamma_L,1-\alpha_N-\gamma_N)\preceq_{\lex}(0,0).
    \end{equation}
    Nontriviality holds exactly when $\respNL$ does not tend to zero, equivalently when its exponent pair is $\succeq_{\lex}(0,0)$.
    Consequently, the parameterization is stable and nontrivial if and only if $\gamma_L=1-\alpha_L$, $\gamma_N=1-\alpha_N$, and
    \begin{equation}\label{eqn:loop-depth-canonical-response-and-phase-stable-nontrivial-phase}
        \begin{cases}
            \alpha_L = 1 / 2, &\alpha_N \ge 1, \\
            1 / 2 < \alpha_L < 1, &\alpha_N \in \RR, \\
            \alpha_L = 1, &\alpha_N \le 1.
        \end{cases}
    \end{equation}
\end{lemma}

\begin{proof}
    At initialization, $z_{0,r,l}\in\fH_u$ and $w_{0,r,l}\in\fH_v$. Hence
    \[
        \opJ_0^l(\fH_u)\subset\fH_v,\qquad \opJ_0^l|_{\fH_v}=0.
    \]
    Every product $\opD$ of initial residual maps preserves both sources, so
    \[
        \opJ_0^l \opD \opJ_0^{l'}=0\qquad(1\le l,l'\le L).
    \]
    Every term containing two update insertions therefore vanishes, including insertions from different layers. This is the two-source calculation underlying \cref{pro:fixed-k-canonical-response}. The first-step response is exactly linear in the signed control $d_\star=\eta\chi_\star\xi_0$ after the width limit, for every fixed $L,N$. We take the output contraction and use \cref{eqn:loop-depth-controlled-training-bounds-layer-gradient-identity} to obtain
    \begin{equation}\label{eqn:loop-depth-canonical-response-and-phase-canonical-output-response}
        f_1(\zeta) - f_0(\zeta)
        = -d_\star \zeta\,\frac{\respNL}{L} \sum_{l = 1}^L
        \normHS*{\opJ_0^l}^2.
    \end{equation}
    For verification of the coefficient, the output differential with respect to a layer matrix is $c \sum_r w_r \otimes z_r = (cN / a) \opJ_0^l$; pair it with $-d_\star\updNL \opJ_0^l$ and use $c\updNL N/a=\respNL/L$.
    Every initial forward or backward vector has vacuum coefficient one. Hence
    \begin{equation}\label{eqn:loop-depth-canonical-response-and-phase-layer-response-norm-bounds}
        \langle v_\star,\opJ_0^lu_\star\rangle=a,
        \qquad a\le\norm*{\opJ_0^lu_\star}\le\normHS*{\opJ_0^l}\le C,
    \end{equation}
    where the last bound holds in the initialization-stable ordered region by \cref{eqn:loop-depth-analytic-products-analytic-product-majorant}. Consequently the factor multiplying $-d_\star \zeta\respNL$ in \cref{eqn:loop-depth-canonical-response-and-phase-canonical-output-response} lies between $a^2$ and a finite positive constant. If $\respNL$ diverges, the response is unbounded; if it has monomial scale one, the response does not vanish. If $\respNL \to 0$, \cref{lem:loop-depth-controlled-training-bounds} proves triviality for all routines. At the training input the coefficient is $-\eta\chi_\star\xi_0^2\respNL$ times this positive average. This settles nontriviality and the necessity of the aggregate-response condition for stability.
    The first block receives the unchanged input $\zeta u_\star$ at both training times. Its width-limit preactivation change from this first update is
    \begin{equation}\label{eqn:loop-depth-canonical-response-and-phase-first-layer-preactivation-response}
        \Delta h_0^{1,1}(\zeta) = -d_\star \zeta\updNL \opJ_0^1 u_\star.
    \end{equation}
    Its coordinate law is centered Gaussian with standard deviation at least $|d_\star\zeta|a\updNL$. If $\updNL$ diverges, this escapes every bounded interval, whether divergence occurs at the $L$ or the $N$ stage. Thus the local-update condition is independently necessary. Sufficiency is \cref{lem:loop-depth-controlled-training-bounds}.
    For these monomials, boundedness is equivalent to \cref{eqn:loop-depth-canonical-response-and-phase-training-stability-conditions}, and nonvanishing requires the response exponent pair to be $\succeq_{\lex}(0,0)$. Thus stability and nontriviality force both response exponents to vanish.
    Substituting $\gamma_L=1-\alpha_L$, $\gamma_N=1-\alpha_N$ into \cref{eqn:loop-depth-initialization-classification-initialization-stability-condition,eqn:loop-depth-canonical-response-and-phase-training-stability-conditions} gives \cref{eqn:loop-depth-canonical-response-and-phase-stable-nontrivial-phase}, exhausting all real exponent choices.
\end{proof}

At the distinguished point, $\respNL=1$ and $\updNL=L^{-1/2}$. We use the continuum construction in \cref{lem:loop-depth-continuum-cycle-products} to determine the local preactivation scales and their response to training.

\begin{lemma}\label{lem:loop-depth-local-features-and-gradients}
    At the distinguished point $(\alpha_N,\alpha_L;\gamma_N,\gamma_L)=(1,1/2;0,1/2)$, every fixed SGD routine is faithful on every fixed nondegenerate evaluation batch, uniformly over layer and loop positions.
    For the singleton first step with $\xi_0\ne0$ and $\chi_0\to\chi_\star\ne0$, the same update gives weak $\bigTheta(1)$ changes in both residual states and preactivations on that batch at every fixed loop fraction $\rho\in(0,1]$, uniformly over layer positions. These scales use the ordered width, depth, and loop limits. At every fixed training step,
    \begin{equation}\label{eqn:loop-depth-local-features-and-gradients-layer-gradient-scale}
        \mG_t^l = \bigO(n^{-1} L^{-1/2}),
    \end{equation}
    uniformly over layer positions, with matching coordinatewise $\bigTheta(n^{-1} L^{-1/2})$ under the same initial signal at time zero, including diagonal coordinates.
\end{lemma}

\begin{proof}
    We first show that $e_{L,l}$ converges weakly to zero, uniformly in $l$. For every $f\in L^2[0,1]$,
    \[
        |\langle e_{L,l},f\rangle|^2
        \le\int_{((l-1)/L,l/L]}|f(s)|^2\dif s.
    \]
    The supremum over $l$ tends to zero by absolute continuity of the integral. For each compact set $\fC$ of Fock vectors, this gives
    \begin{equation}\label{eqn:loop-depth-local-features-and-gradients-local-annihilation-vanishing}
        \sup_{z \in \fC,l}
        \norm*{\ann{e_{L,l}, -} z}\to 0.
    \end{equation}
    Approximation by finite-degree simple tensors, the norm-one annihilators, and a finite net of $\fC$ give uniformity. For the local circular operator,
    \[
        \opC_l=\cre{e_{L,l},+}+\ann{e_{L,l},-},
        \qquad \norm*{\cre{e_{L,l},+}z}=\norm{z}.
    \]
    For fixed $N$ and compact controls, norm-continuity of the finitely many limiting paths makes the incoming vectors compact. Uniform depth convergence in \cref{lem:loop-depth-continuum-cycle-products} and the preceding limit therefore give
    \begin{equation}\label{eqn:loop-depth-local-features-and-gradients-local-circular-norm-preservation}
        \norm*{\opC_l z_{t,r,l}}-\norm*{z_{t,r,l}}\to 0
        \quad(L \to \infty),
    \end{equation}
    uniformly in positions and compact controls. Since $\updNL\opK_t^l=L^{-1/2}\opK_t^l$, the trained correction vanishes as $L\to\infty$ at fixed $N$.

    We then take $N\to\infty$. By \cref{eqn:loop-depth-continuum-cycle-products-layer-path-limit}, the limiting preactivation norm at loop fraction $\rho$ is $|\zeta|\norm{X_t(\rho)}$. The bounds following \cref{eqn:loop-depth-effective-generator-noncancellation-generator-noncancellation-and-bounds}, together with \cref{lem:loop-depth-colored-width-limit}, prove faithfulness uniformly over $\rho$ and compact controls.
    For the first update put $\opG=\opG_0$ and $d_\star=\eta\chi_\star\xi_0$. Restricting to $\fH_0$ and applying \cref{pro:fixed-k-canonical-response} with $k=1$ gives, for every fixed $a,\eta>0$,
    \begin{equation}\label{eqn:loop-depth-local-features-and-gradients-exact-feature-response}
        X_1(\rho)-X_0(\rho)
        =-d_\star a\int_0^\rho e^{a(\rho-s)\opC(\1)}\opG e^{as\opC(\1)}u_\star\,\dif s.
    \end{equation}
    The state projection bound \cref{eqn:fixed-k-canonical-response-state-vacuum-lower-bound} specializes to
    \begin{equation}\label{eqn:loop-depth-local-features-and-gradients-feature-response-lower-bound}
        |\langle v_\star,X_1(\rho)-X_0(\rho)\rangle|
        \ge |d_\star|a^2\rho>0\qquad(\rho>0).
    \end{equation}
    Uniform norm bounds give the matching upper bound. For the same first update, $\opK_0^l=0$ and $\updNL=L^{-1/2}$, so the local preactivation representative satisfies
    \[
        \Delta h_0^{r,l}(\zeta)
        =\zeta \opC_l(z_{1,r,l}-z_{0,r,l})
        +\zeta L^{-1/2}\opK_1^lz_{1,r,l}.
    \]
    The correction vanishes as $L\to\infty$ at fixed $N$. Applying \cref{eqn:loop-depth-local-features-and-gradients-local-annihilation-vanishing} and the norm-preservation argument to $z_{1,r,l}-z_{0,r,l}$ gives the ordered limiting norm $|\zeta|\norm*{X_1(\rho)-X_0(\rho)}$. By \cref{lem:loop-depth-colored-width-limit}, the state and preactivation responses are weak $\bigTheta(1)$ coordinatewise at every fixed positive loop fraction.

    For gradients, \cref{eqn:loop-depth-controlled-training-bounds-layer-gradient-identity} becomes
    \begin{equation}\label{eqn:loop-depth-local-features-and-gradients-normalized-gradient-coordinate}
        n \sqrt{L}\,(\mG_t^l)_{ij}
        = \frac{ad_t}{\eta N} \sum_{r=1}^N
        (\vw_{t,r,l})_i(\vz_{t,r,l})_j.
    \end{equation}
    On compact control sets, \cref{eqn:loop-depth-controlled-training-bounds-response-and-field-bounds} bounds both vector norms. The Haar moments in \cref{eqn:fixed-k-width-limit-haar-coordinate-moments} and Cauchy--Schwarz give
    \[
        \begin{aligned}
            \EE_{\mQ}\left[\left|n\sqrt{L}\,(\mQ\mG_t^l\mQ^\transpose)_{ij}\right|\right]
            &\le\frac{a|d_t|}{\eta N}\sum_{r=1}^N
            \normrms{\vw_{t,r,l}}\normrms{\vz_{t,r,l}}\\
            &=\bigO(1).
        \end{aligned}
    \]
    Rotation invariance, Markov's inequality, and truncation prove \cref{eqn:loop-depth-local-features-and-gradients-layer-gradient-scale}. For the matching coordinate lower bound, the orthogonal source spaces make the forward and backward Gaussian families independent at initialization, including when $i=j$.
    For a Hilbert--Schmidt operator $\opJ:\fH_u\to\fH_v$, we choose orthonormal bases $(e_k)$ of $\fH_u$ and $(f_j)$ of $\fH_v$, and independent standard Gaussians $(F_k)$ and $(H_j)$. The bilinear Gaussian form satisfies
    \[
        Z(\opJ)=\sum_{j,k}\langle f_j,\opJ e_k\rangle H_jF_k,
        \qquad \EE[Z(\opJ)^2]=\normHS{\opJ}^2.
    \]
    The series converges in $L^2$, and this association is an isometry for differences in $\opJ$. For fixed $L,N$, the width-limit law of \cref{eqn:loop-depth-local-features-and-gradients-normalized-gradient-coordinate} is $\chi_\star\xi_0 Z(\opJ_0^l)$. We embed the operators in the same continuum source spaces. By \cref{lem:loop-depth-continuum-cycle-products}, the subsequent limits satisfy
    \[
        \opJ_0^l
        \xrightarrow[L\to\infty]{\mathrm{HS}}\opJ_{0,N}(\lambda)
        \xrightarrow[N\to\infty]{\mathrm{HS}}\opG_0,
    \]
    uniformly over layer positions.
    The bilinear isometry gives the limiting law $\chi_\star\xi_0 Z(\opG_0)$. Conditional on $(F_k)$, its variance is
    \[
        \sigma_F^2=(\chi_\star\xi_0)^2
        \sum_j\left|\sum_k\langle f_j,\opG_0e_k\rangle F_k\right|^2.
    \]
    Since $\EE[\sigma_F^2]=(\chi_\star\xi_0)^2\normHS{\opG_0}^2<\infty$, it is finite almost surely. The identity $\langle v_\star,\opG_0u_\star\rangle=a$ gives a nonzero row, whose Gaussian sum is nonzero almost surely. Thus the conditional law is centered Gaussian with $\sigma_F^2>0$ almost surely and has no atom at zero. The bilinear isometry transfers uniform Hilbert--Schmidt convergence to uniform $L^2$ convergence under the common Gaussian coupling. Together with the upper bound, this proves the weak $\bigTheta$ scale uniformly over layer sequences in the ordered limits.
\end{proof}

\paragraph{Feature increments.}
We now quantify variation along the loop and depth directions.
For a fixed input, we set $\vy_t^0=\mU\zeta$, $\vy_t^r=\vx_t^{r,L}(\zeta)$, and $r_N(\rho)=1+\floor{\rho(N-1)}$.
For $\rho,\lambda\in[0,1)$ and sufficiently small $\varepsilon>0$, we use the loop and rescaled depth increments of \cref{def:loop-depth-feature-increments}:
\[
    \begin{aligned}
        D_{N,t}(\rho,\varepsilon)
        &=\normrms*{\vy_t^{\floor{(\rho+\varepsilon)N}}-\vy_t^{\floor{\rho N}}},\\
        D_{L,t}(\rho,\lambda,\varepsilon)
        &=N\normrms*{\vx_t^{r_N(\rho),\floor{(\lambda+\varepsilon)L}}
        -\vx_t^{r_N(\rho),\floor{\lambda L}}}.
    \end{aligned}
\]
The factor $N$ compares within-cycle depth variation with variation over the full loop trajectory.
For a finite input batch, we stack the state vectors and use $\normrms{\cdot}$.
We associate $\kappa_N$ with $D_{N,t}$ and $\kappa_L$ with $D_{L,t}$, using the supremum and empty-set convention of \cref{def:one-block-feature-diversity}, with the network limits preceding $\varepsilon\downarrow0$.

\begin{lemma}\label{lem:loop-depth-anisotropic-increments-and-uniqueness}
    At $(\alpha_N,\alpha_L;\gamma_N,\gamma_L)=(1,1/2;0,1/2)$, fix an SGD routine, a step $0\le t\le T$, and a scalar input $\zeta\ne0$.
    For $\rho,\lambda\in[0,1)$ and sufficiently small $\varepsilon>0$ with $\rho+\varepsilon\le1$ and $\lambda+\varepsilon\le1$, take $n\to\infty$, then $L\to\infty$, then $N\to\infty$. The resulting increments satisfy, as $\varepsilon\downarrow0$,
    \begin{equation}\label{eqn:loop-depth-anisotropic-increments-and-uniqueness-loop-increment-scale}
        D_{N,t} (\rho, \varepsilon)
        = | \zeta |\norm*{X_t (\rho + \varepsilon) - X_t (\rho)}
        = \bigTheta(\varepsilon),
    \end{equation}
    and, more explicitly,
    \begin{equation}\label{eqn:loop-depth-anisotropic-increments-and-uniqueness-layer-increment-squared-scale}
        \begin{aligned}
            D_{L,t} (\rho, \lambda, \varepsilon)^2 &= a^2 \zeta^2\left[ \varepsilon^2\norm*{\opB_t X_t (\rho)}^2 +\varepsilon(1 - \varepsilon)\norm*{X_t (\rho)}^2\right] \\
            &= \bigTheta(\varepsilon).
        \end{aligned}
    \end{equation}
    Thus the diversity exponents are attained at $\kappa_N=0$ and $\kappa_L=1/2$. At each fixed $\varepsilon$, the unrescaled depth increment vanishes as $N\to\infty$ after the depth limit. The norm bounds are uniform over positions and compact control sets.
    Within the stable and nontrivial region in \cref{eqn:loop-depth-canonical-response-and-phase-stable-nontrivial-phase}, this is the unique point with nondegenerate initialization loop increments and an attained rescaled depth-diversity exponent $\kappa_L=1/2$.
\end{lemma}

\begin{proof}
    Endpoint convergence \cref{eqn:loop-depth-continuum-cycle-products-layer-path-limit,eqn:loop-depth-continuum-cycle-products-cycle-exponential-approximation} gives the first equality in \cref{eqn:loop-depth-anisotropic-increments-and-uniqueness-loop-increment-scale}. Floors change the macroscopic loop time by $\bigO(N^{-1})$ and disappear before $\varepsilon \downarrow 0$. The norm Taylor expansion gives
    \[
        X_t (\rho + \varepsilon) - X_t (\rho)
        = a \varepsilon \opB_t X_t (\rho) + \bigO(\varepsilon^2),
    \]
    uniformly in positions and compact controls. By \cref{eqn:loop-depth-effective-generator-noncancellation-generator-noncancellation-and-bounds} the leading coefficient is bounded above and away from zero, proving \cref{eqn:loop-depth-anisotropic-increments-and-uniqueness-loop-increment-scale}.
    For a layer interval $I = (\lambda, \lambda + \varepsilon)$, apply \cref{eqn:loop-depth-continuum-cycle-products-segment-first-order-expansion} to its forward input in loop $r_N(\rho)$, and then \cref{eqn:loop-depth-continuum-cycle-products-effective-generator-recursion,eqn:loop-depth-continuum-cycle-products-layer-path-limit}. It follows, taking depth first, that the rescaled difference has the final Hilbert-space limit
    \begin{equation}\label{eqn:loop-depth-anisotropic-increments-and-uniqueness-layer-interval-increment-limit}
        a \zeta\,[\opC(\1_I) + \varepsilon \opK_t^{\infty}] X_t (\rho).
    \end{equation}
    Indeed, after the depth limit its expansion error is $N\,\bigO(N^{-2}) \to 0$, and its input differs from $X_t (\rho)$ by $\bigO(N^{-1})$. In this step $\varepsilon$ is fixed. All layer floors are removed by the fixed-$N$ uniform endpoint convergence before the factor $N$ is allowed to tend to infinity.
    We split the interval color into its constant and orthogonal components. Put $g=\1_I-\varepsilon\1$, so $\langle g,\1\rangle=0$ and $\norm{g}^2=\varepsilon(1-\varepsilon)$.
    The constant-color construction gives $X_t(\rho),\opB_tX_t(\rho)\in\fH_0$, whereas $\cre{g,+}X_t(\rho)\in\fH_0^\perp$. Hence
    \[
        \ann{g,-}X_t(\rho)=0,
        \qquad\langle \opB_tX_t(\rho),\cre{g,+}X_t(\rho)\rangle=0.
    \]
    Using $\opC(\1_I)=\varepsilon \opC(\1)+\opC(g)$, we decompose the increment in \cref{eqn:loop-depth-anisotropic-increments-and-uniqueness-layer-interval-increment-limit} as the orthogonal sum
    \begin{equation}\label{eqn:loop-depth-anisotropic-increments-and-uniqueness-orthogonal-layer-increment-decomposition}
        a \zeta\left[
        \varepsilon \opB_t X_t (\rho)
        +\cre{g, +} X_t (\rho)\right].
    \end{equation}
    Creation satisfies $\norm{\cre{g, +} z}=\norm{g}\norm{z}$, proving \cref{eqn:loop-depth-anisotropic-increments-and-uniqueness-layer-increment-squared-scale}. In particular its square-root lower bound uses the positive norm of $X_t$ and cannot be canceled by training. At fixed $\varepsilon$, the unrescaled increment is $\bigO(N^{-1})$ after the depth limit by \cref{eqn:loop-depth-continuum-cycle-products-segment-first-order-expansion}, so it vanishes in the final limit.
    The two increment estimates establish the attained exponents $\kappa_N=0$ and $\kappa_L=1/2$. 
    A proposed $\kappa_N>0$ gives $D_{N,0}/\varepsilon^{1-\kappa_N}=\bigO(\varepsilon^{\kappa_N})\to0$; a proposed $\kappa_L>1/2$ gives $D_{L,0}/\varepsilon^{1-\kappa_L}=\bigO(\varepsilon^{\kappa_L-1/2})\to0$ as $\varepsilon\downarrow0$. Thus the initialization upper bounds exclude every larger exponent in the universal definitions. For uniqueness, \cref{lem:loop-depth-initialization-classification} shows that elsewhere in \cref{eqn:loop-depth-canonical-response-and-phase-stable-nontrivial-phase} the initialization endpoint increments vanish. The estimates at the distinguished point give nondegenerate initialization loop increments and the attained rescaled depth-diversity exponent, proving their simultaneous uniqueness.

    \proofparagraph{Finite-batch extension.}
    For a finite scalar input batch, every vector path is its path at input one times the corresponding scalar. Thus all squared batch norms in these arguments are multiplied by $\sum_j \zeta_j^2$, bounded above and away from zero for the declared batch. Coordinate scales follow componentwise. Uniform estimates on compact controls give the same weak laws for all routines by truncation. This completes the batch assertions.
\end{proof}

\begin{proof}[Proof of \cref{clm:loop-depth-initialization-stability}]
    \cref{lem:loop-depth-initialization-classification} proves sufficiency by the analytic product majorant and necessity by the orthogonal first-degree term of squared norm $N^2Lc^2$.
\end{proof}

\begin{proof}[Proof of \cref{clm:loop-depth-training-and-nontriviality}]
    \cref{lem:loop-depth-controlled-training-bounds} supplies universal finite-step upper bounds. The exact output and first-layer responses in \cref{lem:loop-depth-canonical-response-and-phase} prove that bounded $\respNL,\updNL$ are necessary for training stability and that nonvanishing $\respNL$ is equivalent to nontriviality under initialization stability. The same lemma gives the two lexicographic conditions and their three-region intersection.
\end{proof}

\begin{proof}[Proof of \cref{clm:loop-depth-distinguished-point}]
    The initialization and anisotropic-increment results \cref{lem:loop-depth-initialization-classification,lem:loop-depth-anisotropic-increments-and-uniqueness} select $(1,\frac{1}{2};0,\frac{1}{2})$ within the stable and nontrivial region. \cref{lem:loop-depth-local-features-and-gradients} gives faithfulness, feature learning in both the residual states and preactivations, and the gradient scales. \cref{lem:loop-depth-anisotropic-increments-and-uniqueness} gives $D_{N,t}=\bigTheta(\varepsilon)$, $D_{L,t}=\bigTheta(\sqrt{\varepsilon})$, their attained exponents, and the vanishing unrescaled depth increment, including the batch and ordered-limit assertions. Substituting the distinguished exponents gives branch multiplier $a/(N\sqrt{L})$ and learning-rate coefficient $\eta/(n\sqrt{L})$ on the rescaled SGD gradient, hence learning rate $\eta$ on the raw gradient.
\end{proof}

\section{Heuristics for the Proofs in the General Case}\label{app:nonlinear-justifications}

The linear analyses in \cref{app:fixed-k-proofs,app:loop-depth-proofs} prove the claims stated in \cref{app:loop-classification,app:loop-depth} for linear networks trained with SGD; here we give conditional justifications for nonlinear networks and discuss the additional assumptions needed for Adam.
We first state the nonlinear assumptions and the additional requirements for entrywise optimizers.
We then analyze a single shared matrix through finite-network identities, initialization stability, training responses, and feature properties, before extending the argument to residual branches of fixed finite depth and repeated stacks.

\subsection{Setup and assumptions}\label{app:nonlinear-setup-assumptions}

\begin{definition}[Nonlinear shared-block model]\label{def:one-block-nonlinear-setup}
    We consider a residual branch containing one trainable matrix shared across $N$ loops, with a coordinatewise activation $\phi$ and mean subtraction.
    The branch constant $a>0$, base learning rate $\eta>0$, and finite training horizon $T$ are fixed, while $\alpha_N,\gamma_N\in\RR$.
    We take $n\to\infty$ at fixed $N$, followed by $N\to\infty$.
    \begin{itemize}[leftmargin=*]
        \item \textbf{Model.}
        The input and readout matrices $\mU,\mV\in\RR^{n\times1}$ are frozen, and the shared hidden matrix $\mW_t\in\RR^{n\times n}$ is trainable.
        All initial matrices are mutually independent, with $\mU_i\iid\fN(0,1)$, $\mV_i\iid\fN(0,n^{-2})$, and $\mW_{0,ij}\iid\fN(0,n^{-1})$.
        We set $\mP_n=\mI-\vone_n\vone_n^\transpose/n$, so $\MS(\vz)=\mP_n\vz$.
        For a scalar input $\zeta$, the hidden states, preactivations, and output satisfy
        \[
            \begin{aligned}
                \vx_t^0(\zeta)&=\mU\zeta,\\
                \vh_t^r(\zeta)&=\mW_t\vx_t^{r-1}(\zeta),\\
                \vx_t^r(\zeta)&=\vx_t^{r-1}(\zeta)+c_N\mP_n\phi(\vh_t^r(\zeta)),\qquad 1\le r\le N,\\
                f_t(\zeta)&=\mV^\transpose\vx_t^N(\zeta).
            \end{aligned}
        \]
        \item \textbf{Training.}
        For a training batch $\Xi_t$ with loss $\ell_t$, we use the complete shared gradient and its rescaling:
        \[
            \mG_t=\nabla_{\mW}\ell_t(\vf_t(\Xi_t)),\qquad
            \widetilde{\mG}_t=nN^{\alpha_N-1}\mG_t.
        \]
        The maps $Q_t:\RR^{t+1}\to\RR$ act entrywise on the rescaled gradient history.
        The optimizer update is
        \[
            \mW_{t+1}=\mW_t-\eta n^{-1}N^{-\gamma_N} Q_t(\widetilde{\mG}_0,\dots,\widetilde{\mG}_t).
        \]
        SGD uses $Q_t(z_0,\dots,z_t)=z_t$.
        \item \textbf{Scales.}
        We write
        \[
            c_N=aN^{-\alpha_N},\qquad \accN=N^{1-\alpha_N},
            \qquad \updN=N^{-\gamma_N},\qquad \respN=\accN\updN.
        \]
        Thus $Nc_N=a\accN$, where $c_N$ is the branch multiplier.
        We use $\accN$ for initialization accumulation, $\updN$ for the local parameter-update scale, and $\respN$ for the aggregate state/output response scale.
        \item \textbf{Conventions.}
        The routine data and all finite batch sizes are independent of network size.
        We write $\Xi_t$ for a training batch, $\vzeta$ for an evaluation batch, and $\vf_t(\Xi_t)$ for the corresponding training outputs.
        For a singleton batch, we write $\Xi_t=(\xi_t)$ and $\chi_t=\ell_t'(f_t(\xi_t))$.
        Changes compare the same input at the same network position before and after a training step.
    \end{itemize}
\end{definition}

For a fixed input $\zeta$, we write $f_N(\mW;\zeta)$ for the output of the $N$-loop network as a function of the shared matrix, so $f_N(\mW_t;\zeta)=f_t(\zeta)$.
We omit the fixed input argument in response calculations and write $\vx^0=\mU\xi_0$ for the corresponding first training input.

\begin{assumption}\label{asm:one-block-nonlinear-limits}
    The nonlinear arguments below use the assumptions listed in this display, with joint convergence over the fixed finite training and evaluation batches.
    \begin{enumerate}[label=\upshape(\roman*)]
        \item\label{it1:one-block-nonlinear-limits} \textbf{Regularity.} $\phi(0)=0$, $\phi\in C^2$, and $\phi',\phi''$ are bounded. Products and response remainders used below have uniform fourth-moment control on bounded effective trajectories.
        
        \item\label{it2:one-block-nonlinear-limits} \textbf{Initialization closure.} At fixed $N$, the initialization coordinates have joint width limits $X^r,H^r$ satisfying
        \[
            \begin{aligned}
                X^r&=X^0+c_N\sum_{j=1}^r\left(\phi(H^j)-\EE\phi(H^j)\right), \\
                X^0&\sim\fN(0,q_0),\quad q_0=\zeta^2,
            \end{aligned}
        \]
        where $(H^1,\dots,H^N)$ is jointly centered Gaussian, independent of $X^0$, with $\EE[H^rH^s]=\EE[X^{r-1}X^{s-1}]$. Its activation covariance obeys
        \[
            \Cov(\phi(A),\phi(B))\ge b_\phi^2c
        \]
        for centered jointly Gaussian $(A,B)$ of covariance $c\ge0$, with a fixed $b_\phi>0$.
        
        \item\label{it3:one-block-nonlinear-limits} \textbf{Trained closure.} Forward, backward, and response variables admit compatible joint width limits retaining correlations induced by training and transpose multiplication. On bounded effective trajectories, empirical inner products, products, and Riemann sums have the required moment control and converge uniformly on compact control sets. Coordinates are exchangeable, and normalized responses satisfy \cref{eqn:coordinate-anti-concentration} uniformly over the declared positions and compact control sets.
        
        \item\label{it4:one-block-nonlinear-limits} \textbf{Jacobian response.}
        For a fixed nondegenerate scalar training input with a finite nonzero initial loss-derivative limit and $\alpha_N\ge1$, we set
        \[
            \bar{\mJ}_N(\mW)=\accN^{-1}\nabla_{\mW} f_N(\mW).
        \]
        At initialization, its Frobenius norm has a finite positive limit, and fixed coordinates of $\bar{\mJ}_N(\mW_0)\vx^0$ have nonzero limits satisfying \cref{eqn:coordinate-anti-concentration}. 
        At $(1,0)$, the directional responses of the output, residual state, and full preactivation to the same first SGD update are jointly nondegenerate on the declared inputs at every fixed positive loop fraction. 
        Uniform small-learning-rate remainder control permits one fixed $\eta>0$ for all such fractions.
        
        \item\label{it5:one-block-nonlinear-limits} \textbf{Trained nondegeneracy.} At $(1,0)$, every training routine has uniformly nondegenerate preactivation coordinate laws and
        \[
            0<m_{t,\zeta}\le
            \norm*{\phi(H_t(\lambda;\zeta))-\EE[\phi(H_t(\lambda;\zeta))]}_{L^2}
            \le M_{t,\zeta}<\infty
        \]
        uniformly over $\lambda\in[0,1]$ and compact control sets.
        
        \item\label{it6:one-block-nonlinear-limits} \textbf{Integrated secant alignment.}
        For the full nontriviality classification with $\alpha_N\ge1$ and $\respN\ge1$, assume a fixed scalar training input with initial loss derivative $\chi_0\to\chi_\star>0$ and a fixed $\eta>0$ such that
        \[
            \int_0^1
            \langle\bar{\mJ}_N(\mW_0-\theta\eta\chi_0 \updN\bar{\mJ}_N(\mW_0)),
            \bar{\mJ}_N(\mW_0)\rangle_{\mathrm{F}}\,\dif\theta
        \]
        has a positive lower bound in the ordered limit. At $(1,0)$, the local Jacobian-response condition in \cref{it4:one-block-nonlinear-limits} suffices instead.
    \end{enumerate}
\end{assumption}

Under \cref{asm:one-block-nonlinear-limits}, \cref{it1:one-block-nonlinear-limits,it2:one-block-nonlinear-limits} govern initialization stability; \cref{it1:one-block-nonlinear-limits,it3:one-block-nonlinear-limits,it4:one-block-nonlinear-limits,it6:one-block-nonlinear-limits} govern training stability and nontriviality; and \cref{it3:one-block-nonlinear-limits,it4:one-block-nonlinear-limits,it5:one-block-nonlinear-limits} govern faithfulness, feature learning, and diversity at the critical point.

\begin{remark}[Adam and other entrywise optimizers]\label{rmk:entrywise-optimizer-extension}
    The SGD response conditions in \cref{it4:one-block-nonlinear-limits,it6:one-block-nonlinear-limits} must be checked along the direction produced by a different optimizer.
    For Adam, with zero initial moments, fixed $0\le\beta_1,\beta_2<1$, and $\epsilon_{\mathrm{opt}}>0$, we use
    \[
        Q_t(\widetilde{\mG}_0,\dots,\widetilde{\mG}_t)
        =
        \frac{
            \displaystyle
            \frac{1-\beta_1}{1-\beta_1^{t+1}}
            \sum_{s=0}^{t}\beta_1^{t-s}\widetilde{\mG}_s
        }{
            \displaystyle
            \sqrt{
                \frac{1-\beta_2}{1-\beta_2^{t+1}}
                \sum_{s=0}^{t}\beta_2^{t-s}
                (\widetilde{\mG}_s)^{\odot 2}
            }
            +\epsilon_{\mathrm{opt}}
        }.
    \]
    All squares, square roots, and quotients act entrywise.
    For a general entrywise rule, we write
    \[
        \mA_{t,N}=n^{-1}Q_t(\widetilde{\mG}_0,\dots,\widetilde{\mG}_t),
        \qquad \Delta\mW_t=-\eta\updN\mA_{t,N}.
    \]
    For this direction, we require joint closure and upper bounds as in \cref{it3:one-block-nonlinear-limits}, together with the output and joint residual-state/preactivation response conditions corresponding to \cref{it4:one-block-nonlinear-limits}.
    For finite updates in the large-response region, the integrated alignment in \cref{it6:one-block-nonlinear-limits} must also be formulated along the chosen optimizer direction.
    The learning-rate coefficient outside $Q_t$ is $\eta n^{-1}N^{-\gamma_N}$ for fixed finite branch depth and $\eta n^{-1}N^{-\gamma_N}L^{-\gamma_L}$ for repeated stacks. At $(\alpha_N,\gamma_N)=(1,0)$, the first is $\eta/n$; at \LoopmuP, the second is $\eta/(n\sqrt{L})$.
\end{remark}

The linear classifications in \cref{app:fixed-k-proofs,app:loop-depth-proofs} were proved without mean subtraction. We now show that $\MS$ preserves their width-limit observables.

\begin{lemma}[Mean subtraction preserves the linear phase diagrams]\label{lem:nonlinear-linear-centering}
    For the identity activation, fix the block or stack depth, loop count, and training horizon in \cref{app:fixed-k-proofs,app:loop-depth-proofs}. Uniformly over bounded scalar update transcripts, if $\vz'$ is the counterpart of any computed vector $\vz$ after inserting $\MS$, then $\normrms*{\vz'-\vz}=\littleo_{\PP}(1)$ as $n\to\infty$; every normalized Gram entry, scalar contraction, and fixed coordinate changes by $\littleo_{\PP}(1)$. Hence mean subtraction preserves the linear initialization, training, feature, and diversity phase diagrams.
\end{lemma}
\begin{proof}
    We couple the centered and uncentered computations using the same Gaussian initialization and scalar transcript, and write $\ve=\vone_n/\sqrt{n}$ and $\mP_n=\mI-\ve\ve^\transpose$. Each uncentered forward or backward vector is orthogonally equivariant. Conditioning on the unrotated computation and averaging over an independent Haar rotation $\mQ$, \cref{it3:aux-haar} gives, on bounded-initialization events,
    \[
        \EE_{\mQ}\left[\normrms*{(\mI-\mP_n)\mQ\vz}^2\right]
        =\frac{\EE_{\mQ}|\ve^\transpose\mQ\vz|^2}{n}
        =\frac{\normrms{\vz}^2}{n}=\bigO(n^{-1}).
    \]
    At fixed structural parameters and horizon, this bound is uniform over bounded transcripts. We propagate it through the finitely many vector operations and normalized rank-one updates; the vector errors remain $\littleo_{\PP}(1)$ in RMS norm, and
    \[
        \begin{aligned}
            \normop*{\frac{\vz\vw^\transpose-\vz'(\vw')^\transpose}{n}}
            &\le\normrms*{\vz-\vz'}\normrms{\vw}
            +\normrms{\vz'}\normrms*{\vw-\vw'}\\
            &=\littleo_{\PP}(1),
        \end{aligned}
    \]
    so normalized Grams and scalar contractions have the same limits. Permutation equivariance transfers the mean-square estimate to each fixed coordinate, and polynomial dependence on bounded scalar transcripts gives the required uniformity.
\end{proof}

\paragraph{Activation checks for \cref{asm:one-block-nonlinear-limits}.}
The following lemma verifies the activation regularity in \cref{asm:one-block-nonlinear-limits}~\cref{it1:one-block-nonlinear-limits} and the covariance lower bound in \cref{it2:one-block-nonlinear-limits} for SiLU and GELU.

\begin{lemma}[SiLU/GELU activation regularity and covariance]\label{lem:nonlinear-activation-covariance}
    Let $(A,B)$ be centered jointly Gaussian with $\Cov(A,B)=c\ge0$. Suppose $\phi(z)=b_\phi z+e(z)$, where $b_\phi>0$ and $e$ is even and nondecreasing in $|z|$, with the required Gaussian moments. Then
    \[
        \Cov(\phi(A),\phi(B))\ge b_\phi^2c.
    \]
    Both $\silu(z)=z/(1+e^{-z})$ and exact $\gelu(z)=z\Phi(z)$ satisfy the covariance bound with $b_\phi=\frac{1}{2}$, where $\Phi$ is the standard Gaussian distribution function. Both vanish at zero and have bounded first and second derivatives. For every $q\ge0$ and $Z\sim\fN(0,1)$, they satisfy
    \[
        \EE[\phi'(\sqrt{q}Z)]=\frac{1}{2},
        \qquad \EE[\phi'(\sqrt{q}Z)^2]\ge\frac{1}{4},
        \qquad \Var(\phi(\sqrt{q}Z))\ge\frac{q}{4}.
    \]
\end{lemma}
\begin{proof}
    We expand the covariance and use the invariance of the joint law under $(A,B)\mapsto(-A,-B)$ to remove the mixed linear--even terms:
    \[
        \Cov(\phi(A),\phi(B))=b_\phi^2c+\Cov(e(A),e(B)).
    \]
    Conditional on $A=x$, the Gaussian variable $B$ has variance independent of $x$ and a mean whose absolute value is nondecreasing in $|x|$. For a standard Gaussian $Z$, $\sigma>0$, and $m,t\ge0$,
    \[
        \frac{\partial}{\partial m}\PP(|m+\sigma Z|>t)
        =\frac{1}{\sigma}
        \left[\varphi((t-m)/\sigma)-\varphi((t+m)/\sigma)\right]\ge0,
    \]
    where $\varphi$ is the standard Gaussian density. The tail bound makes $\EE\left[e(B)\given A=x\right]$ nondecreasing in $|x|$. Since $e(A)$ has the same monotonicity,
    \[
        \Cov(e(A),e(B))
        =\Cov\left(e(A),\EE\left[e(B)\given A\right]\right)\ge0
    \]
    by the monotone covariance inequality. With zero conditional variance, $B$ is a scalar multiple of $A$ and the same monotonicity holds directly; if either variable has zero variance, the covariance bound is immediate. For SiLU and GELU, we use
    \[
        \silu(z)=\frac{z}{2}+\frac{z}{2}\tanh(z/2),
        \qquad \gelu(z)=\frac{z}{2}+z\left(\Phi(z)-\frac{1}{2}\right).
    \]
    For $z\ge0$, we differentiate the even parts to obtain
    $\frac{1}{2}\tanh(z/2)+\frac{z}{4}\sech^2(z/2)$ and $\Phi(z)-\frac{1}{2}+z\varphi(z)$, both nonnegative. The second derivatives are
    \[
        \silu''(z)=\frac{1}{4}\sech^2(z/2)\bigl(2-z\tanh(z/2)\bigr),
        \qquad \gelu''(z)=(2-z^2)\varphi(z).
    \]
    Hyperbolic and Gaussian decay bound these derivatives and the first derivatives above. The derivatives of the even parts are odd, so symmetry gives the derivative mean. Jensen's inequality gives its second-moment bound, and taking $A=B$ in the covariance inequality gives the activation variance bound.
\end{proof}

\subsection{Basic properties}\label{app:nonlinear-basic-properties}

We first differentiate the finite shared computation, retaining all loop contributions, mean subtraction, and the loss derivative.

\begin{lemma}[Exact shared nonlinear SGD]\label{lem:nonlinear-exact-shared-sgd}
    In singleton notation, we set $\vv=n\mV$ and
    \[
        \vr_t^r=n\nabla_{\vx_t^r}\ell_t,\qquad
        \mD_t^r=\diag(\phi'(\vh_t^r)),\qquad
        \vb_t^r=\mD_t^r\mP_n\vr_t^r.
    \]
    Then
    \begin{equation}\label{eqn:nonlinear-exact-shared-sgd-backward-recursion}
        \vr_t^N=\vv\chi_t,
        \qquad \vr_t^{r-1}=\vr_t^r+c_N\mW_t^\transpose\mD_t^r\mP_n\vr_t^r.
    \end{equation}
    We collect the loop contributions in the shared update kernel
    \begin{equation}\label{eqn:nonlinear-exact-shared-sgd-shared-update}
        \mK_{t,N}=\frac{a}{N}\sum_{r=1}^N
        \frac{\vb_t^r(\vx_t^{r-1})^\transpose}{n}.
    \end{equation}
    The exact identities are $\mG_t=\accN\mK_{t,N}$, $\widetilde{\mG}_t=n\mK_{t,N}$, and
    \[
        \mW_{t+1}=\mW_t-\eta \updN\mK_{t,N}.
    \]
    For a finite batch $\Xi_t=(\xi_{t,b})_{b=1}^B$, we attach $b$ to the singleton state and backward variables and set $\chi_{t,b}=\partial_b\ell_t(\vf_t(\Xi_t))$.
    Thus $\vx_{t,b}^r=\vx_t^r(\xi_{t,b})$, $\vr_{t,b}^r=n\nabla_{\vx_{t,b}^r}\ell_t$, and $\vb_{t,b}^r=\diag(\phi'(\vh_{t,b}^r))\mP_n\vr_{t,b}^r$.
    The backward recursion holds per input with terminal value $\vv\chi_{t,b}$, while
    \[
        \mG_t=\sum_{b=1}^{|\Xi_t|}\chi_{t,b}
        \nabla_{\mW}f_t(\xi_{t,b}),\qquad
        \mK_{t,N}=\frac{a}{N}\sum_{b=1}^{|\Xi_t|}\sum_{r=1}^N
        \frac{\vb_{t,b}^r(\vx_{t,b}^{r-1})^\transpose}{n}.
    \]
    All gradient and SGD identities above hold with this summed kernel.
\end{lemma}
\begin{proof}
    We transpose the loop-state Jacobian $\mI+c_N\mP_n\mD_t^r\mW_t$ to obtain \cref{eqn:nonlinear-exact-shared-sgd-backward-recursion}. Differentiating all occurrences of the shared matrix gives
    \[
        \mG_t=\frac{c_N}{n}\sum_{r=1}^N\vb_t^r(\vx_t^{r-1})^\transpose
        =\accN\mK_{t,N},
        \qquad
        \widetilde{\mG}_t=n\accN^{-1}\mG_t=n\mK_{t,N},
    \]
    where we used $Nc_N=a\accN$ and \cref{eqn:nonlinear-exact-shared-sgd-shared-update}. Substituting the SGD rule into \cref{eqn:one-block-optimizer-update} yields $\Delta\mW_t=-\eta\updN\mK_{t,N}$.
    The backward variables already include the loss derivative. For a finite batch, we sum the calculation over its inputs. At $(\alpha_N,\gamma_N)=(1,0)$, the update reduces to $\mW_{t+1}=\mW_t-\eta\mG_t$.
\end{proof}

To bound the exact update, we absorb $\accN$ into the shared matrix.

\begin{lemma}[Effective variables and finite-step bounds]\label{lem:nonlinear-effective-operator-and-upper-bounds}
    Under \cref{it1:one-block-nonlinear-limits}, set $\mB_{t,N}=\accN\mW_t$ and $\psi_\beta(z)=\beta\phi(z/\beta)$ for $\beta>0$. The finite-network recursions become
    \begin{equation}\label{eqn:nonlinear-effective-operator-rescaled-recursion}
        \begin{aligned}
            \vx_t^r&=\vx_t^{r-1}+\frac{a}{N}\mP_n\psi_{\accN}(\mB_{t,N}\vx_t^{r-1}), \\
            \vh_t^r&=\accN^{-1}\mB_{t,N}\vx_t^{r-1}, \\
            \vr_t^{r-1}&=\vr_t^r+\frac{a}{N}\mB_{t,N}^\transpose \diag\left(\phi'(\mB_{t,N}\vx_t^{r-1}/\accN)\right)\mP_n\vr_t^r, \\
            \mB_{t+1,N}&=\mB_{t,N}-\eta\respN\mK_{t,N}.
        \end{aligned}
    \end{equation}
    Fix a finite horizon $T$ and finite training and evaluation batches with bounded input and loss-derivative controls. If $\alpha_N\ge1$ and $\respN=\bigO(1)$, then on bounded-initialization events,
    \[
        \begin{aligned}
            \max_{0\le t\le T}\normop{\mB_{t,N}}&=\bigO(1),\\
            \max_{0\le t\le T}\sup_{0\le r\le N}\normrms{\vx_t^r}&=\bigO(1),\\
            \max_{0\le t<T}\left(\sup_{0\le r\le N}\normrms{\vr_t^r}
            +\normop{\mK_{t,N}}\right)&=\bigO(1).
        \end{aligned}
    \]
    The constants are independent of $N$ and uniform over the fixed batches. For every $t<T$,
    \begin{equation}\label{eqn:nonlinear-effective-operator-response-bounds}
        \sup_{r\le N}\normrms*{\Delta \vx_t^r}=\bigO(\respN),
        \qquad |\Delta f_t|=\bigO(\respN).
    \end{equation}
    If also $\gamma_N\ge0$, then $\normop{\mW_t}$ is bounded and $\sup_r\normrms{\Delta \vh_t^r}=\bigO(\updN)$. In the first loop, without any small-update approximation,
    \begin{equation}\label{eqn:nonlinear-effective-operator-direct-preactivation}
        \Delta \vh_t^1(\zeta)=-\eta \updN\mK_{t,N}\mU\zeta.
    \end{equation}
\end{lemma}
\begin{proof}
    We put $L_\phi=\norm{\phi'}_\infty$. For $\mB=\beta\mW$, the identity $\beta\phi(\mW\vx)=\psi_\beta(\mB\vx)$ gives the rescaled forward recursion. Uniformly in $\beta>0$, $\psi_\beta(0)=0$ and $|\psi_\beta'|\le L_\phi$. On a bounded effective-operator ball this gives
    \[
        \begin{aligned}
            \normrms*{\vx_t^r}&\le e^{aL_\phi\normop*{\mB_{t,N}}}\normrms*{\mU\xi_t}, \\
            \normrms*{\vr_t^r}&\le e^{aL_\phi\normop*{\mB_{t,N}}}\normrms*{\vv\chi_t},
        \end{aligned}
    \]
    and
    \[
        \normop*{\mK_{t,N}}\le
        \frac{aL_\phi}{N}\sum_{r=1}^N\normrms*{\vr_t^r}\normrms*{\vx_t^{r-1}}.
    \]
    We sum the kernel bound over the fixed training batch. Writing $M_t=\normop{\mB_{t,N}}$, the preceding estimates and \cref{eqn:nonlinear-effective-operator-rescaled-recursion} give
    \[
        \normop{\mK_{t,N}}\le Ce^{2aL_\phi M_t},
        \qquad M_{t+1}\le M_t+C\eta\respN e^{2aL_\phi M_t},
    \]
    where $C$ depends on the bounded controls and batch size. Since $T$ is finite and $\respN$ is bounded, induction bounds $M_t$ uniformly in $N$.
    We compare state trajectories using $\Delta\mB_{t,N}=-\eta\respN\mK_{t,N}$ and apply discrete Gronwall. This gives the state bound in \cref{eqn:nonlinear-effective-operator-response-bounds}; the output bound follows from $|\mV^\transpose\vz|\le\normrms{\vv}\normrms{\vz}$.
    If $\updN$ is bounded, $\Delta\mW_t=-\eta\updN\mK_{t,N}$ also bounds $\mW_t$ over the finite horizon. Finally,
    \[
        \Delta \vh_t^r=\mW_{t+1}\Delta \vx_t^{r-1}+\Delta \mW_t\vx_t^{r-1}
    \]
    has norm $\bigO(\respN+\updN)=\bigO(\updN)$ because $\accN\le1$. At $r=1$ the input is unchanged, giving \cref{eqn:nonlinear-effective-operator-direct-preactivation}. These are deterministic norm estimates on the specified events.
\end{proof}

Under \cref{asm:one-block-nonlinear-limits}~\cref{it3:one-block-nonlinear-limits}, permutation equivariance and moment truncation transfer the vector bounds to coordinate tightness. For matrix entries, the product-moment bound in \cref{eqn:nonlinear-exact-shared-sgd-shared-update} gives $(\mK_{t,N})_{ij}=\bigO(n^{-1})$, hence $(\mG_t)_{ij}=\bigO(n^{-1}\accN)$. Coordinate nonlinearities use permutation symmetry here.

\subsection{Justification of initialization stability}\label{app:nonlinear-initialization}

Under \cref{it2:one-block-nonlinear-limits}, we obtain a discrete covariance recursion for $\alpha_N<1$.
We set $C_{r,s}=\EE[X^rX^s]$ and define $\fV_\phi(q,q',c)=\Cov(\phi(A),\phi(B))$, where $(A,B)$ is centered jointly Gaussian with $\EE[A^2]=q$, $\EE[B^2]=q'$, and $\EE[AB]=c$.
Then
\begin{equation}\label{eqn:nonlinear-initialization-covariance-recursion}
    C_{r,s}=q_0+c_N^2\sum_{i=1}^r\sum_{j=1}^s
    \fV_\phi(C_{i-1,i-1},C_{j-1,j-1},C_{i-1,j-1}).
\end{equation}
Induction using the kernel lower bound yields $C_{r,s}\ge q_0$ and
\begin{equation}\label{eqn:nonlinear-initialization-variance-lower-bound}
    C_{N-1,N-1}\ge q_0+b_\phi^2q_0c_N^2(N-1)^2.
\end{equation}
Since $H^N$ is Gaussian with this variance after the fixed-$N$ width limit, its coordinates escape bounded intervals when $\alpha_N<1$. For $0<\alpha_N<1$, when the required joint local-time limits exist, we use their real $L^2$ coordinate-law representation with $\langle Y,X\rangle=\EE[YX]$ and $\opP X=X-\EE[X]$.
The actions $\opW_t,\opW_t^*$ represent the compatible forward and transpose computations in \cref{it3:one-block-nonlinear-limits}.
With $\tau=rc_N$, the local-time equations are
\[
    \begin{aligned}
        \partial_\tau X_t&=\opP\phi(\opW_tX_t), \\
        -\partial_\tau R_t&=\opW_t^*[\phi'(\opW_tX_t)\opP R_t], \\
        0&\le\tau\le a\accN.
    \end{aligned}
\]
The interval grows with $N$; gradients retain their exact sum form and need not remain bounded. For $\alpha_N\le0$ the same discrete covariance necessity applies, with no Euler step-size limit required. In the linear case $\phi(z)=b_\phi z$, the solution of \cref{eqn:nonlinear-initialization-covariance-recursion} is
\[
    C_{r,s}=q_0\sum_{j=0}^{\min(r,s)}
    \binom{r}{j}\binom{s}{j}(b_\phi c_N)^{2j},
\]
which recovers the initialization calculation for a single-matrix residual branch in \cref{sec:one-block-desiderata}.

\begin{proof}[Justification of \cref{clm:one-block-initialization-stability}]
    We use \cref{it1:one-block-nonlinear-limits,it2:one-block-nonlinear-limits}. If $\alpha_N\ge1$, then $\accN\le1$ and $\normop{\mB_{0,N}}\le\normop{\mW_0}$.
    The forward norm estimate in the proof of \cref{lem:nonlinear-effective-operator-and-upper-bounds} bounds all initialized states uniformly over loops. Since $\vh_0^r=\mW_0\vx_0^{r-1}$, it also bounds the preactivations. Permutation symmetry gives coordinate tightness. Independence of the readout gives
    \[
        \left.f_0(\zeta)\given\vx_0^N(\zeta)\right.
        \sim\fN\left(0,\frac{\normrms{\vx_0^N(\zeta)}^2}{n}\right),
    \]
    so the output is bounded in the ordered limit.

    Conversely, we choose a fixed nonzero input, so $q_0>0$. By \cref{eqn:nonlinear-initialization-variance-lower-bound}, $\Var(H^N)=C_{N-1,N-1}$ diverges when $\alpha_N<1$. The Gaussian preactivation law excludes initialization stability.
\end{proof}

\subsection{Justification of training stability and nontriviality}\label{app:nonlinear-branches}

The initialization argument in \cref{app:nonlinear-initialization} excludes $\alpha_N<1$, so we now work with $\alpha_N\ge1$.
We analyze the critical path and the vanishing-residual regime using the regularity and joint closure in \cref{it1:one-block-nonlinear-limits,it3:one-block-nonlinear-limits}. For response lower bounds, we use the local Jacobian condition and integrated secant alignment in \cref{it4:one-block-nonlinear-limits,it6:one-block-nonlinear-limits}.

\paragraph{Critical residual path.}

Under \cref{asm:one-block-nonlinear-limits}~\cref{it3:one-block-nonlinear-limits}, when the loop Riemann limits exist, we use the same real $L^2$ coordinate-law conventions and define the shared limiting gradient kernel by
\[
    \begin{aligned}
        \opK_t&=a\sum_{b=1}^{|\Xi_t|}\int_0^1
        b_{t,b}(s)\otimes X_t(s;\xi_{t,b})\,\dif s, \\
        b_{t,b}(s)&=\phi'(H_t(s;\xi_{t,b}))\opP R_{t,b}(s),
        \qquad R_{t,b}(1)=v\chi_{t,b}.
    \end{aligned}
\]
Here $U,v,X,H,R,b$ are the coordinate-limit variables of $\mU,\vv=n\mV,\vx,\vh,\vr,\vb$, respectively.
We use $(b\otimes X)Y=b\EE[XY]$ for the gated backward variable $b$.
For a singleton batch, we omit the batch subscript.
The shared update before the final branch limit acts on the trajectory as
\[
    \opW_tX
    =\opW_0X-\eta a\updN\sum_{j<t}\sum_{b=1}^{|\Xi_j|}\int_0^1
    b_{j,b}(s)\langle X_j(s;\xi_{j,b}),X\rangle\,\dif s,
\]
and its adjoint exchanges the two vectors in each rank-one map. Consequently the covariance kernel $C_{t,j}(\lambda,\mu)=\EE[X_t(\lambda)X_j(\mu)]$ depends on two loop positions. We write $\opB_t=\accN\opW_t$ for the effective action. If $F_t(\lambda)=\opP\psi_{\accN}(\opB_tX_t(\lambda))$, then
\[
    \begin{aligned}
        \partial_\lambda C_{t,j}(\lambda,\mu) &=a\EE[F_t(\lambda)X_j(\mu)], \\
        \partial_\lambda\partial_\mu C_{t,j}(\lambda,\mu) &=a^2\EE[F_t(\lambda)F_j(\mu)].
    \end{aligned}
\]
The expectations use the full trained joint law, including the shared initial operator.

\begin{lemma}[Critical branch dynamics]\label{lem:nonlinear-critical-branch}
    Under \cref{it1:one-block-nonlinear-limits,it3:one-block-nonlinear-limits}, at $(\alpha_N,\gamma_N)=(1,0)$ the ordered-limit SGD dynamics on the declared finite input batches satisfy
    \begin{equation}\label{eqn:nonlinear-critical-branch-dynamics}
        \begin{aligned}
            \partial_\lambda X_t&=a\opP\phi(H_t),
            \qquad H_t=\opW_tX_t, \\
            X_t(0;\zeta)&=U\zeta, \\
            -\partial_\lambda R_t&=a\opW_t^*[\phi'(H_t)\opP R_t], \qquad R_t(1)=v\chi_t, \\
            \opW_{t+1}&=\opW_t-\eta\opK_t,
            \qquad f_t(\zeta)=\langle v,X_t(1;\zeta)\rangle.
        \end{aligned}
    \end{equation}
    For a training batch, the backward equation applies separately to each input and the update uses the summed limiting gradient kernel $\opK_t$. At $\alpha_N=1,\gamma_N>0$, state, preactivation, and output responses are $\bigO(\updN)$ and vanish over a fixed training horizon.
\end{lemma}
\begin{proof}
    At $\alpha_N=1$, $\accN=1$. Sum the exact forward and backward recursions in \cref{lem:nonlinear-exact-shared-sgd} over loops. By \cref{it3:one-block-nonlinear-limits}, their joint limits satisfy
    \[
        \begin{aligned}
            X_t(\lambda;\zeta)&=U\zeta+
            a\int_0^\lambda\opP\phi(H_t(s;\zeta))\,\dif s,\\
            R_t(\lambda)&=v\chi_t+
            a\int_\lambda^1\opW_t^*[
            \phi'(H_t(s))\opP R_t(s)]\,\dif s.
        \end{aligned}
    \]
    We apply the joint convergence in \cref{it3:one-block-nonlinear-limits} to the update kernel and readout contraction. For $\gamma_N=0$, the exact SGD update and differentiation of the integral equations give \cref{eqn:nonlinear-critical-branch-dynamics}, including its boundary values.

    For $\gamma_N>0$, we have $\respN=\updN\to0$. By \cref{lem:nonlinear-effective-operator-and-upper-bounds}, state, preactivation, and output responses are $\bigO(\updN)$ per step and vanish over the fixed horizon.
\end{proof}

\paragraph{Vanishing residual path.}
Under \cref{asm:one-block-nonlinear-limits}~\cref{it1:one-block-nonlinear-limits,it3:one-block-nonlinear-limits}, the bounded-weight branch for a singleton training input and $\alpha_N>1$, $\accN\to0$, has the following $L^2$ expansions:
\[
    \begin{aligned}
        X_t(\lambda;\zeta)&=U\zeta+\bigO_{L^2}(\accN), \\
        H_t(\lambda;\zeta)&=\opW_tU\zeta+\bigO_{L^2}(\accN), \\
        R_t(\lambda)&=v\chi_t+\littleo_{L^2}(1).
    \end{aligned}
\]
Thus
\[
    \opK_t
    =a\left(\phi'(\opW_tU\xi_t)\opP(v\chi_t)\right)
    \otimes(U\xi_t)+\littleo_{\mathrm{HS}}(1).
\]
The joint $L^2$ convergence and fourth-moment bounds in \cref{it1:one-block-nonlinear-limits,it3:one-block-nonlinear-limits}, together with boundedness and Lipschitz continuity of $\phi'$, give $L^2$ convergence of the gated backward variables. The rank-one estimate
\[
    \normHS{b\otimes X-b'\otimes X'}
    \le\norm{b-b'}_{L^2}\norm{X}_{L^2}
    +\norm{b'}_{L^2}\norm{X-X'}_{L^2}
\]
then gives the Hilbert--Schmidt remainder after integration over the loop fraction.
The matrix can still change at order one when $\gamma_N=0$, but the hidden-state path freezes and the output response vanishes. If $\gamma_N<0$ and $\respN\le1$, use the effective variables instead:
\[
    \begin{aligned}
        \partial_\lambda X_t&=a\opP\psi_{\accN}(\opB_tX_t), \\
        -\partial_\lambda R_t &=a\opB_t^*[ \phi'(\opB_tX_t/\accN)\opP R_t], \\
        \opB_{t+1}&=\opB_t-\eta\respN\opK_t.
    \end{aligned}
\]
For $1-\alpha_N<\gamma_N<0$, \cref{eqn:nonlinear-effective-operator-response-bounds} still gives output triviality, while \cref{eqn:nonlinear-effective-operator-direct-preactivation} and \cref{it4:one-block-nonlinear-limits} show that the preactivation changes are not $\bigO(1)$. At $\gamma_N=1-\alpha_N$, the effective update is finite and is described through the recession function $\phi_\infty(z)=\lim_{\beta\downarrow0}\beta\phi(z/\beta)$ when it exists. Gate convergence in the backward equation is supplied by \cref{it3:one-block-nonlinear-limits}, including any mass at a zero effective field. The first update uses the original gate $\phi'(\opW_0U\xi_0)$ even though $\opB_0\to0$.

\paragraph{Output response.}
Bounds on dynamics do not by themselves show that training changes the output, the residual state, or the preactivation. We therefore express these changes as parameter responses: a local derivative controls sufficiently small learning rates, whereas a finite update requires an integrated secant.

\begin{lemma}[Local variation and integrated response]\label{lem:nonlinear-state-and-training-response}
    At finite width and a fixed input $\zeta$, let $\vx_{\mW}$ solve
    \[
        \partial_\lambda\vx_{\mW}=a\mP_n\phi(\mW\vx_{\mW}),
        \qquad \vx_{\mW}(0)=\mU\zeta.
    \]
    We write $\vx=\vx_{\mW}$, $\vh=\mW\vx$, and $\mD(\lambda)=\diag(\phi'(\vh(\lambda)))$.
    For a matrix direction $\mA$, set $\vy_{\mA}=\partial_\theta\vx_{\mW+\theta\mA}|_{\theta=0}$.
    Let $\mT_{\mW}(\lambda,s)$ be the propagator of $\partial_\lambda\vy=a\mP_n\mD(\lambda)\mW\vy$. Then
    \begin{equation}\label{eqn:nonlinear-response-local-variation}
        \begin{aligned}
            \vy_{\mA}(\lambda)&=a\int_0^\lambda \mT_{\mW}(\lambda,s)\mP_n\mD(s)\mA \vx(s)\,\dif s, \\
            D_{\mW}\vh(\lambda)[\mA]&=\mA \vx(\lambda)+\mW \vy_{\mA}(\lambda).
        \end{aligned}
    \end{equation}
    For the finite-$N$ scalar-output network evaluated at its training input, write $\chi_0=\ell_0'(f_N(\mW_0))$, $\mJ_N(\mW)=\nabla_{\mW}f_N(\mW)$, and $\bar{\mJ}_N=\accN^{-1}\mJ_N$. The first SGD step follows $\mW_\theta=\mW_0-\theta\eta\chi_0\updN\bar{\mJ}_N(\mW_0)$, so its exact output response is
    \begin{equation}\label{eqn:nonlinear-response-integrated-secant}
        f_N(\mW_1)-f_N(\mW_0)
        =-\eta\chi_0\respN\int_0^1
        \langle\bar{\mJ}_N(\mW_\theta),\bar{\mJ}_N(\mW_0)\rangle_{\mathrm{F}}\,\dif\theta.
    \end{equation}
\end{lemma}
\begin{proof}
    We differentiate the flow while keeping its initial input fixed:
    \[
        \partial_\lambda\vy_{\mA}
        =a\mP_n\mD(\lambda)(\mW\vy_{\mA}+\mA\vx),
        \qquad \vy_{\mA}(0)=0.
    \]
    Variation of constants gives the first identity in \cref{eqn:nonlinear-response-local-variation}. Differentiating $\vh=\mW\vx$ gives its second identity, including the direct parameter action and state response. To obtain \cref{eqn:nonlinear-response-integrated-secant}, we integrate
    \[
        \partial_\theta f_N(\mW_\theta)
        =-\eta\chi_0\updN\langle \mJ_N(\mW_\theta),\bar{\mJ}_N(\mW_0)\rangle_{\mathrm{F}}
    \]
    and use $\respN=\accN\updN$. In particular the derivative at $\eta=0$ is $-\chi_0\respN\normfro{\bar{\mJ}_N(\mW_0)}^2$.
\end{proof}

Under \cref{it3:one-block-nonlinear-limits,it4:one-block-nonlinear-limits}, we pass \cref{eqn:nonlinear-response-local-variation} to the shared limiting variables. We write $\opT_0(\lambda,s)$ for the limiting state propagator corresponding to $\mT_{\mW_0}(\lambda,s)$. At $(\alpha_N,\gamma_N)=(1,0)$, the first SGD update has direction $\opA=-\opK_0$. We denote its residual-state and full preactivation derivatives with respect to $\eta$ at $\eta=0$ by $\mathscr{X}_\lambda$ and $\mathscr{H}_\lambda$:
\[
    \begin{aligned}
        \mathscr{X}_\lambda
        &=-a\int_0^\lambda \opT_0(\lambda,s)\opP\left(
        \phi'(H_0(s))\opK_0X_0(s)\right)\,\dif s,\\
        \mathscr{H}_\lambda&=-\opK_0X_0(\lambda)+\opW_0\mathscr{X}_\lambda.
    \end{aligned}
\]
After the ordered network limits, we expand the same update as $\eta\downarrow0$. For each fixed $\lambda>0$,
\[
    X_1(\lambda)-X_0(\lambda)=\eta\mathscr{X}_\lambda+\littleo_{L^2}(\eta),
    \qquad
    H_1(\lambda)-H_0(\lambda)=\eta\mathscr{H}_\lambda+\littleo_{L^2}(\eta).
\]
The joint lower-response conditions and remainder control needed to choose one fixed learning rate for all prescribed positive loop fractions are assumed in \cref{it4:one-block-nonlinear-limits}. At general exponents, \cref{eqn:nonlinear-response-integrated-secant} retains the entire finite update. With $\langle\vz,\vw\rangle_n=\vz^\transpose\vw/n$, its integrand has the exact two-position expansion
\[
    \langle\bar{\mJ}_N(\mW_\theta),\bar{\mJ}_N(\mW_0)\rangle_{\mathrm{F}}
    =\frac{a^2}{N^2}\sum_{r,q}
    \langle\vb_\theta^r,\vb_0^q\rangle_n
    \langle\vx_\theta^{r-1},\vx_0^{q-1}\rangle_n,
\]
where the backward terminal value is $\vv$. The finite-step lower bound uses \cref{it6:one-block-nonlinear-limits}.

To illustrate the local response, we take $\chi_\star=1$ and expand in the small branch constant $a$. Under \cref{it2:one-block-nonlinear-limits,it3:one-block-nonlinear-limits,it4:one-block-nonlinear-limits}, with the same moment estimates uniform for sufficiently small $a$, set $u=U\xi_0$, $q_0=\EE[u^2]>0$, $G=\opW_0u$, and $D=\phi'(G)$. Then
\[
    X_0(\lambda)=u+\bigO_{L^4}(a\accN),\qquad
    R_0(\lambda)=v+\bigO_{L^4}(a\accN),
\]
\begin{equation}\label{eqn:nonlinear-initial-kernel-small-a}
    \begin{aligned}
        \opK_0&=a(Dv)\otimes u+\bigO_{\mathrm{HS}}(a^2\accN), \\
        \normHS*{\opK_0}^2 &=a^2q_0\EE[D^2]+\bigO(a^3\accN).
    \end{aligned}
\end{equation}
Independence of $v$ from the initialization forward variables gives the norm identity in \cref{eqn:nonlinear-initial-kernel-small-a}. We also obtain $\opK_0u=aq_0Dv+\bigO_{L^2}(a^2\accN)$ and, at $\alpha_N=1$, $\mathscr{H}_\lambda=-aq_0Dv+\bigO_{L^2}(a^2)$ uniformly in $\lambda$.
The covariance lower bound in \cref{it2:one-block-nonlinear-limits} makes $\phi$ nonconstant.
Since $G$ is nondegenerate Gaussian, $\EE[D^2]>0$, so these same-input leading terms are nonzero.
The joint residual-state and full preactivation lower responses, including those on other declared inputs, remain conditional on \cref{it4:one-block-nonlinear-limits}.

We can also compute a finite first step on the boundary $\alpha_N>1,\gamma_N=1-\alpha_N$, where $\accN\to0$ and $\respN=1$. 
Suppose $\sup_{z\in\RR}|\phi(z)-z_+|<\infty$ and $\EE[\phi'(Z)]=1/2$ for $Z\sim\fN(0,1)$ (e.g., SiLU and GELU), where $z_+=\max\set{z,0}$.
Then $\phi_\infty(z)=z_+$ and, for $\beta>0$,
\[
    \sup_z|\beta\phi(z/\beta)-z_+|\le C\beta.
\]
Under \cref{asm:one-block-nonlinear-limits}~\cref{it2:one-block-nonlinear-limits,it3:one-block-nonlinear-limits,it6:one-block-nonlinear-limits}, for unit input $\xi_0=1$ and unit initial loss signal $\chi_\star=1$, the limiting variables are independent $U,v,G\sim\fN(0,1)$ and
\[
    \opK_0=a(\phi'(G)v)\otimes U,\qquad
    \opB_1=-\eta a(\phi'(G)v)\otimes U.
\]
Writing $D=\phi'(G)$, the rank-one update gives $\opB_1X_1(\lambda)=-\eta aDv\,\EE[UX_1(\lambda)]$.
Since $U$ is centered and independent of $(D,v)$, the forward equation gives $\partial_\lambda\EE[UX_1(\lambda)]=0$. Its initial value is $\EE[U^2]=1$, so
\begin{equation}\label{eqn:nonlinear-recession-output-response}
    \begin{aligned}
        X_1(\lambda)&=U+a\lambda\opP(-\eta aDv)_+, \\
        f_1-f_0&=a\EE[v(-\eta aDv)_+]=-\frac{\eta a^2}{4}.
    \end{aligned}
\end{equation}
We condition on $D$ and use Gaussian symmetry to obtain $\EE\left[v(-\eta aDv)_+\given D\right]=-\eta aD/2$.
Since $\EE[D]=1/2$, we obtain the output response in \cref{eqn:nonlinear-recession-output-response}.
Thus the first-step response is finite at $\respN=1$, while the preactivation retains the factor $\accN^{-1}$. For the general large-update region, we use the integrated alignment in \cref{it6:one-block-nonlinear-limits}.

\begin{proof}[Justification of \cref{clm:one-block-training-and-nontriviality}]
    We work in the initialization-stable region $\alpha_N\ge1$ and use \cref{it1:one-block-nonlinear-limits,it3:one-block-nonlinear-limits,it4:one-block-nonlinear-limits,it6:one-block-nonlinear-limits}. If $\gamma_N\ge0$, then $\updN\le1$ and $\respN=\accN\updN\le1$. The bounds in \cref{lem:nonlinear-effective-operator-and-upper-bounds} give stable training.

    For the singleton Jacobian witness, $\mK_{0,N}=\chi_0\bar{\mJ}_N(\mW_0)$. Hence \cref{eqn:nonlinear-effective-operator-direct-preactivation} becomes
    \[
        \Delta\vh_0^1=-\eta\chi_0\updN\bar{\mJ}_N(\mW_0)\vx^0.
    \]
    By \cref{it4:one-block-nonlinear-limits}, its normalized coordinate response has a nonzero lower scale. Thus $\gamma_N<0$ violates training stability.

    If $\alpha_N+\gamma_N>1$, then $\respN\to0$. By \cref{eqn:nonlinear-effective-operator-response-bounds}, every fixed routine has $\Delta f_t=\bigO(\respN)$ at each step, so its output remains unchanged in the limit. This also covers negative $\gamma_N$ with bounded effective operators.

    At $(1,0)$, \cref{it4:one-block-nonlinear-limits} gives a nonzero output response for one fixed small learning rate. At the remaining points with $\alpha_N+\gamma_N\le1$, we have $\respN\ge1$. The positive integrated alignment in \cref{it6:one-block-nonlinear-limits}, substituted into \cref{eqn:nonlinear-response-integrated-secant}, gives $|\Delta f_0|=\bigOmega(\respN)$ for its declared singleton routine. These points are therefore nontrivial.
\end{proof}

\subsection{Justification of feature properties and architectural extensions}\label{app:nonlinear-heuristic-justifications}

The preceding response calculation identifies the residual-state and full preactivation changes along the same SGD update. We now use their assumed joint nondegeneracy and the critical path equation to describe the feature properties.
We state the requirements for entrywise optimizers in \cref{rmk:entrywise-optimizer-extension}.

\begin{proof}[Justification of \cref{clm:one-block-loop-mup-and-diversity}]
    The initialization and training justifications in \cref{app:nonlinear-initialization,app:nonlinear-branches} select $(\alpha_N,\gamma_N)=(1,0)$ as the unique stable and nontrivial point.

    At this point, we use the joint closure, response, and remainder conditions in \cref{it3:one-block-nonlinear-limits,it4:one-block-nonlinear-limits}. For the same first singleton SGD update on the declared nondegenerate evaluation batch, they give one fixed sufficiently small $\eta>0$ such that both $\Delta\mX_0^{\floor{\lambda N}}$ and $\Delta\mH_0^{\floor{\lambda N}}$ are $\bigTheta(1)$ for every fixed $\lambda\in(0,1]$, establishing \cref{def:one-block-feature-learning}.
    The preactivation nondegeneracy in \cref{it5:one-block-nonlinear-limits}, together with the upper bounds, gives faithfulness for every routine and step.

    By \cref{it3:one-block-nonlinear-limits,lem:nonlinear-critical-branch},
    \begin{equation}\label{eqn:nonlinear-feature-increment-expansion}
        X_t(\lambda+\varepsilon)-X_t(\lambda)
        =a\varepsilon\opP\phi(H_t(\lambda))+\littleo_{L^2}(\varepsilon).
    \end{equation}
    By \cref{it5:one-block-nonlinear-limits}, the leading coefficient in \cref{eqn:nonlinear-feature-increment-expansion} has $L^2$ norm in $[a m_{t,\zeta},a M_{t,\zeta}]$, so the increment has size $\bigTheta(\varepsilon)$.
    We transfer this estimate to normalized finite-network increments using \cref{it3:one-block-nonlinear-limits}. The ordered network limits precede $\varepsilon\downarrow0$, so the floors disappear.
    Thus $\kappa_N=0$ is admissible in \cref{def:one-block-feature-diversity}. For every $\kappa_N>0$, the increment divided by $\varepsilon^{1-\kappa_N}$ is $\bigO(\varepsilon^{\kappa_N})$ and tends to zero, excluding that exponent.
    At the critical point, the finite-network bounds give
    \[
        \sum_{r=1}^N\normrms{\vx_t^r-\vx_t^{r-1}}=\bigO(1),
        \qquad
        \sum_{r=1}^N\normrms{\vx_t^r-\vx_t^{r-1}}^2=\bigO(N^{-1}).
    \]
    For the limiting trajectory in \cref{eqn:nonlinear-critical-branch-dynamics},
    \[
        a\EE\left[\int_0^1
        |\opP\phi(H_t(\lambda))|\,\dif\lambda\right]
        \le aM_{t,\zeta}.
    \]
    Thus the integral equation admits almost surely absolutely continuous loop trajectories with finite total variation on $[0,1]$.
\end{proof}

\paragraph{Fixed block depth.}
For the fixed-depth model in \cref{eqn:fixed-k-model,eqn:fixed-k-optimizer-update}, we update all $k$ shared matrices simultaneously. Since $k$ is fixed, additional matrix products change constants but introduce no new powers of $N$. The scales remain $\accN$, $\updN$, and $\respN$.

We require the regularity in \cref{it1:one-block-nonlinear-limits} for each $\phi_j$ and matrix-indexed versions of \cref{it2:one-block-nonlinear-limits,it3:one-block-nonlinear-limits,it4:one-block-nonlinear-limits,it5:one-block-nonlinear-limits}. Joint closure applies at every internal position. The response hypothesis concerns the residual state and all internal preactivations along the same simultaneous update, including every prefix of the branch. For internal features,
\[
    \normrms{\vz_t^{r,j}}\le\norm{\phi_j'}_\infty\normrms{\vh_t^{r,j}},
    \qquad
    \normrms{\Delta\vz_t^{r,j}}\le\norm{\phi_j'}_\infty\normrms{\Delta\vh_t^{r,j}},
\]
because $\phi_j(0)=0$ and $\normop{\mP_n}\le1$.

Under these hypotheses, the preceding initialization and response arguments give the general-case heuristic for \cref{clm:fixed-k-training-and-nontriviality,clm:fixed-k-loop-mup-and-feature-behavior}. Nontriviality is classified within the training-stable region. At $(\alpha_N,\gamma_N)=(1,0)$, the matrix-indexed hypotheses give faithfulness at every internal preactivation, feature learning in the residual state and every internal preactivation, and $\kappa_N=0$.

\paragraph{Loop-depth scaling.}
For the repeated-stack model, we reuse $L$ centered nonlinear residual blocks over $N$ loops. We require \cref{it1:one-block-nonlinear-limits} uniformly over the activations $\phi_l$ and block-indexed versions of \cref{it2:one-block-nonlinear-limits,it3:one-block-nonlinear-limits,it4:one-block-nonlinear-limits,it5:one-block-nonlinear-limits}, uniform in $L$ and including jointly nondegenerate residual-state and preactivation responses at positive loop fractions. We take $n\to\infty$, then $L\to\infty$, then $N\to\infty$.

With branch multiplier $c=aN^{-\alpha_N}L^{-\alpha_L}$, the accumulation and aggregate response scales satisfy
\[
    N\sqrt{L}\,c=a\accNL,
    \qquad NLc\,\updNL=a\respNL.
\]
The first block in the first loop isolates the direct update scale $\updNL$. The scale $\accNL$ reflects independent block fluctuations and reuse across loops, while $\respNL$ collects the aligned contributions of all $NL$ block applications.

A depthwise closure that retains the independent block fluctuations and remains uniformly nondegenerate supplies $\kappa_L=1/2$. The subsequent loop continuum supplies $\kappa_N=0$. These are the heuristic mechanisms behind \cref{clm:loop-depth-initialization-stability,clm:loop-depth-training-and-nontriviality,clm:loop-depth-distinguished-point}. The additional depth limit requires its own uniform closure for both increments in \cref{def:loop-depth-feature-increments}.


\begin{thebibliography}{40}
\providecommand{\natexlab}[1]{#1}
\providecommand{\url}[1]{\texttt{#1}}
\expandafter\ifx\csname urlstyle\endcsname\relax
  \providecommand{\doi}[1]{doi: #1}\else
  \providecommand{\doi}{doi: \begingroup \urlstyle{rm}\Url}\fi

\bibitem[Bachlechner et~al.(2021)Bachlechner, Majumder, Mao, Cottrell, and
  McAuley]{bachlechner2020rezero}
Thomas Bachlechner, Bodhisattwa~Prasad Majumder, Huanru~Henry Mao, Gary
  Cottrell, and Julian~J. McAuley.
\newblock {ReZero} is all you need: fast convergence at large depth.
\newblock In Cassio~P. de~Campos, Marloes~H. Maathuis, and Erik Quaeghebeur,
  editors, \emph{Proceedings of the Thirty-Seventh Conference on Uncertainty in
  Artificial Intelligence, {UAI} 2021, Virtual Event, 27-30 July 2021},
  Proceedings of Machine Learning Research, pages 1352--1361. {AUAI} Press,
  2021.
\newblock URL \url{https://proceedings.mlr.press/v161/bachlechner21a.html}.

\bibitem[Bordelon et~al.(2024)Bordelon, Noci, Li, Hanin, and
  Pehlevan]{bordelon2024depthwise}
Blake Bordelon, Lorenzo Noci, Mufan~Bill Li, Boris Hanin, and Cengiz Pehlevan.
\newblock Depthwise hyperparameter transfer in residual networks: Dynamics and
  scaling limit.
\newblock In \emph{The Twelfth International Conference on Learning
  Representations, {ICLR} 2024, Vienna, Austria, May 7-11, 2024}.
  OpenReview.net, 2024.
\newblock URL \url{https://openreview.net/forum?id=KZJehvRKGD}.

\bibitem[Chen et~al.(2026)Chen, Li, Liang, Lao, and Liu]{chen2026training}
Lizhang Chen, Jonathan Li, Chen Liang, Ni~Lao, and Qiang Liu.
\newblock Training-free looped transformers.
\newblock \emph{arXiv preprint arXiv:2605.23872}, 2026.
\newblock \doi{10.48550/arXiv.2605.23872}.

\bibitem[Chen et~al.(2018)Chen, Rubanova, Bettencourt, and
  Duvenaud]{chen2018neural}
Tian~Qi Chen, Yulia Rubanova, Jesse Bettencourt, and David Duvenaud.
\newblock Neural ordinary differential equations.
\newblock In Samy Bengio, Hanna~M. Wallach, Hugo Larochelle, Kristen Grauman,
  Nicol{\`{o}} Cesa{-}Bianchi, and Roman Garnett, editors, \emph{Advances in
  Neural Information Processing Systems 31: Annual Conference on Neural
  Information Processing Systems 2018, NeurIPS 2018, December 3-8, 2018,
  Montr{\'{e}}al, Canada}, pages 6572--6583, 2018.
\newblock URL
  \url{https://proceedings.neurips.cc/paper/2018/hash/69386f6bb1dfed68692a24c8686939b9-Abstract.html}.

\bibitem[Dehghani et~al.(2019)Dehghani, Gouws, Vinyals, Uszkoreit, and
  Kaiser]{dehghani2018universal}
Mostafa Dehghani, Stephan Gouws, Oriol Vinyals, Jakob Uszkoreit, and Lukasz
  Kaiser.
\newblock Universal transformers.
\newblock In \emph{7th International Conference on Learning Representations,
  {ICLR} 2019, New Orleans, LA, USA, May 6-9, 2019}. OpenReview.net, 2019.
\newblock URL \url{https://openreview.net/forum?id=HyzdRiR9Y7}.

\bibitem[Dey et~al.(2025)Dey, Zhang, Noci, Li, Bordelon, Bergsma, Pehlevan,
  Hanin, and Hestness]{dey2025completep}
Nolan Dey, Bin~Claire Zhang, Lorenzo Noci, Mufan Li, Blake Bordelon, Shane
  Bergsma, Cengiz Pehlevan, Boris Hanin, and Joel Hestness.
\newblock Don't be lazy: {CompleteP} enables compute-efficient deep
  transformers.
\newblock In \emph{Advances in Neural Information Processing Systems},
  volume~38, 2025.

\bibitem[Diaconis and Freedman(1987)]{diaconis1987dozen}
Persi Diaconis and David Freedman.
\newblock A dozen de {Finetti}-style results in search of a theory.
\newblock \emph{Annales de l'I.H.P. Probabilit{\'e}s et statistiques},
  23\penalty0 (S2):\penalty0 397--423, 1987.

\bibitem[Fan et~al.(2025)Fan, Du, Ramchandran, and Lee]{fan2024looped}
Ying Fan, Yilun Du, Kannan Ramchandran, and Kangwook Lee.
\newblock Looped transformers for length generalization.
\newblock In \emph{The Thirteenth International Conference on Learning
  Representations, {ICLR} 2025, Singapore, April 24-28, 2025}. OpenReview.net,
  2025.
\newblock URL \url{https://openreview.net/forum?id=2edigk8yoU}.

\bibitem[Gatmiry et~al.(2024)Gatmiry, Saunshi, Reddi, Jegelka, and
  Kumar]{gatmiry2024multistep}
Khashayar Gatmiry, Nikunj Saunshi, Sashank~J. Reddi, Stefanie Jegelka, and
  Sanjiv Kumar.
\newblock Can looped transformers learn to implement multi-step gradient
  descent for in-context learning?
\newblock In Ruslan Salakhutdinov, Zico Kolter, Katherine~A. Heller, Adrian
  Weller, Nuria Oliver, Jonathan Scarlett, and Felix Berkenkamp, editors,
  \emph{Forty-first International Conference on Machine Learning, {ICML} 2024,
  Vienna, Austria, July 21-27, 2024}, Proceedings of Machine Learning Research,
  pages 15130--15152. {PMLR} / OpenReview.net, 2024.
\newblock URL \url{https://proceedings.mlr.press/v235/gatmiry24b.html}.

\bibitem[Geiping et~al.(2025)Geiping, McLeish, Jain, Kirchenbauer, Singh,
  Bartoldson, Kailkhura, Bhatele, and Goldstein]{geiping2025scaling}
Jonas Geiping, Sean McLeish, Neel Jain, John Kirchenbauer, Siddharth Singh,
  Brian~R. Bartoldson, Bhavya Kailkhura, Abhinav Bhatele, and Tom Goldstein.
\newblock Scaling up test-time compute with latent reasoning: A recurrent depth
  approach.
\newblock In \emph{Advances in Neural Information Processing Systems 38: Annual
  Conference on Neural Information Processing Systems 2025, NeurIPS 2025},
  2025.
\newblock \doi{10.48550/arXiv.2502.05171}.

\bibitem[Giannou et~al.(2023)Giannou, Rajput, Sohn, Lee, Lee, and
  Papailiopoulos]{giannou2023looped}
Angeliki Giannou, Shashank Rajput, Jy{-}yong Sohn, Kangwook Lee, Jason~D. Lee,
  and Dimitris Papailiopoulos.
\newblock Looped transformers as programmable computers.
\newblock In \emph{International Conference on Machine Learning, {ICML} 2023},
  2023.

\bibitem[H{\"a}gele et~al.(2026)H{\"a}gele, Hern{\'a}ndez-Cano, Kosson, and
  Jaggi]{hagele2026decoupling}
Alexander H{\"a}gele, Alejandro Hern{\'a}ndez-Cano, Atli Kosson, and Martin
  Jaggi.
\newblock Improving neural network training by decoupling the magnitude and
  direction of weight vectors.
\newblock \emph{arXiv preprint arXiv:2606.25971}, 2026.
\newblock \doi{10.48550/arXiv.2606.25971}.

\bibitem[Hayou and Yang(2023)]{hayou2023commute}
Soufiane Hayou and Greg Yang.
\newblock Width and depth limits commute in residual networks.
\newblock In Andreas Krause, Emma Brunskill, Kyunghyun Cho, Barbara Engelhardt,
  Sivan Sabato, and Jonathan Scarlett, editors, \emph{International Conference
  on Machine Learning, {ICML} 2023, 23-29 July 2023, Honolulu, Hawaii, {USA}},
  Proceedings of Machine Learning Research, pages 12700--12723. {PMLR}, 2023.
\newblock URL \url{https://proceedings.mlr.press/v202/hayou23a.html}.

\bibitem[Isserlis(1918)]{isserlis1918formula}
Leon Isserlis.
\newblock On a formula for the product-moment coefficient of any order of a
  normal frequency distribution in any number of variables.
\newblock \emph{Biometrika}, 12\penalty0 (1/2):\penalty0 134--139, 1918.

\bibitem[Lan et~al.(2020)Lan, Chen, Goodman, Gimpel, Sharma, and
  Soricut]{lan2019albert}
Zhenzhong Lan, Mingda Chen, Sebastian Goodman, Kevin Gimpel, Piyush Sharma, and
  Radu Soricut.
\newblock {ALBERT:} {A} lite {BERT} for self-supervised learning of language
  representations.
\newblock In \emph{8th International Conference on Learning Representations,
  {ICLR} 2020, Addis Ababa, Ethiopia, April 26-30, 2020}. OpenReview.net, 2020.
\newblock URL \url{https://openreview.net/forum?id=H1eA7AEtvS}.

\bibitem[Li et~al.(2026)Li, Zhang, Guo, Gu, and Wang]{li2026deeploop}
Shuzhen Li, Yifan Zhang, Jiacheng Guo, Quanquan Gu, and Mengdi Wang.
\newblock {DeepLoop}: Depth scaling for looped transformers.
\newblock \emph{arXiv preprint arXiv:2607.13491}, 2026.
\newblock \doi{10.48550/arXiv.2607.13491}.

\bibitem[{Llama Team}(2024)]{dubey2024llama}
{Llama Team}.
\newblock The llama 3 herd of models.
\newblock \emph{CoRR}, abs/2407.21783, 2024.
\newblock \doi{10.48550/ARXIV.2407.21783}.
\newblock URL \url{https://doi.org/10.48550/arXiv.2407.21783}.

\bibitem[Loshchilov and Hutter(2019)]{loshchilov2019decoupled}
Ilya Loshchilov and Frank Hutter.
\newblock Decoupled weight decay regularization.
\newblock In \emph{7th International Conference on Learning Representations,
  {ICLR} 2019, New Orleans, LA, USA, May 6-9, 2019}. OpenReview.net, 2019.
\newblock URL \url{https://openreview.net/forum?id=Bkg6RiCqY7}.

\bibitem[Marion et~al.(2025)Marion, Fermanian, Biau, and
  Vert]{marion2022scalingresnets}
Pierre Marion, Adeline Fermanian, G{\'{e}}rard Biau, and Jean{-}Philippe Vert.
\newblock Scaling resnets in the large-depth regime.
\newblock \emph{J. Mach. Learn. Res.}, 26:\penalty0 56:1--56:48, 2025.
\newblock URL \url{https://jmlr.org/papers/v26/22-0664.html}.

\bibitem[Mlodozeniec et~al.(2025)Mlodozeniec, Ablin, Béthune, Busbridge,
  Klein, Ramapuram, and Cuturi]{mlodozeniec2025completed}
Bruno Mlodozeniec, Pierre Ablin, Louis Béthune, Dan Busbridge, Michal Klein,
  Jason Ramapuram, and Marco Cuturi.
\newblock Completed hyperparameter transfer across modules, width, depth, batch
  and duration.
\newblock 2025.
\newblock URL \url{https://arxiv.org/abs/2512.22382}.

\bibitem[Movahedi et~al.(2026)Movahedi, Milovanovi{\'c}, Feigin, Theus,
  Hofmann, Boeva, Rusch, and Orvieto]{movahedi2026fixed}
Sajad Movahedi, Vera Milovanovi{\'c}, Shlomo~Libo Feigin, Alexander Theus,
  Thomas Hofmann, Valentina Boeva, T.~Konstantin Rusch, and Antonio Orvieto.
\newblock Fixed-point reasoners: Stable and adaptive deep looped transformers.
\newblock \emph{arXiv preprint arXiv:2606.18206}, 2026.
\newblock \doi{10.48550/arXiv.2606.18206}.

\bibitem[Ng and Wang(2024)]{ng2024loopnn}
Kei-Sing Ng and Qingchen Wang.
\newblock Loop neural networks for parameter sharing.
\newblock 2024.
\newblock URL \url{https://arxiv.org/abs/2409.14199}.

\bibitem[Penedo et~al.(2024)Penedo, Kydl{\'{\i}}cek, Allal, Lozhkov, Mitchell,
  Raffel, von Werra, and Wolf]{penedo2024fineweb}
Guilherme Penedo, Hynek Kydl{\'{\i}}cek, Loubna~Ben Allal, Anton Lozhkov,
  Margaret Mitchell, Colin~A. Raffel, Leandro von Werra, and Thomas Wolf.
\newblock The {FineWeb} datasets: Decanting the web for the finest text data at
  scale.
\newblock In \emph{Advances in Neural Information Processing Systems 38: Annual
  Conference on Neural Information Processing Systems 2024, NeurIPS 2024,
  Vancouver, BC, Canada, December 10 - 15, 2024}, 2024.
\newblock URL
  \url{http://papers.nips.cc/paper\_files/paper/2024/hash/370df50ccfdf8bde18f8f9c2d9151bda-Abstract-Datasets\_and\_Benchmarks\_Track.html}.

\bibitem[Prairie et~al.(2026)Prairie, Novack, Berg-Kirkpatrick, and
  Fu]{prairie2026parcae}
Hayden Prairie, Zachary Novack, Taylor Berg-Kirkpatrick, and Daniel~Y. Fu.
\newblock Parcae: Scaling laws for stable looped language models.
\newblock \emph{arXiv preprint arXiv:2604.12946}, 2026.
\newblock \doi{10.48550/arXiv.2604.12946}.

\bibitem[Press and Wolf(2017)]{press2017using}
Ofir Press and Lior Wolf.
\newblock Using the output embedding to improve language models.
\newblock In Mirella Lapata, Phil Blunsom, and Alexander Koller, editors,
  \emph{Proceedings of the 15th Conference of the European Chapter of the
  Association for Computational Linguistics, {EACL} 2017, Valencia, Spain,
  April 3-7, 2017, Volume 2: Short Papers}, pages 157--163. Association for
  Computational Linguistics, 2017.
\newblock \doi{10.18653/V1/E17-2025}.
\newblock URL \url{https://doi.org/10.18653/v1/e17-2025}.

\bibitem[Qiu et~al.(2025)Qiu, Chen, Phan, Lei, and
  Wilson]{qiu2025hyperparameter}
Shikai Qiu, Zixi Chen, Hoang Phan, Qi~Lei, and Andrew~Gordon Wilson.
\newblock Hyperparameter transfer enables consistent gains of
  matrix-preconditioned optimizers across scales.
\newblock In \emph{Advances in Neural Information Processing Systems 38: Annual
  Conference on Neural Information Processing Systems 2025, NeurIPS 2025},
  2025.
\newblock \doi{10.48550/arXiv.2512.05620}.

\bibitem[Saunshi et~al.(2025)Saunshi, Dikkala, Li, Kumar, and
  Reddi]{saunshi2025latent}
Nikunj Saunshi, Nishanth Dikkala, Zhiyuan Li, Sanjiv Kumar, and Sashank~J.
  Reddi.
\newblock Reasoning with latent thoughts: On the power of looped transformers.
\newblock In \emph{The Thirteenth International Conference on Learning
  Representations, {ICLR} 2025, Singapore, April 24-28, 2025}. OpenReview.net,
  2025.
\newblock URL \url{https://openreview.net/forum?id=din0lGfZFd}.

\bibitem[Touvron et~al.(2023)Touvron, Lavril, Izacard, Martinet, Lachaux,
  Lacroix, Rozi{\`{e}}re, Goyal, Hambro, Azhar, Rodriguez, Joulin, Grave, and
  Lample]{touvron2023llama}
Hugo Touvron, Thibaut Lavril, Gautier Izacard, Xavier Martinet, Marie{-}Anne
  Lachaux, Timoth{\'{e}}e Lacroix, Baptiste Rozi{\`{e}}re, Naman Goyal, Eric
  Hambro, Faisal Azhar, Aur{\'{e}}lien Rodriguez, Armand Joulin, Edouard Grave,
  and Guillaume Lample.
\newblock {LLaMA}: Open and efficient foundation language models.
\newblock \emph{CoRR}, abs/2302.13971, 2023.
\newblock \doi{10.48550/ARXIV.2302.13971}.
\newblock URL \url{https://doi.org/10.48550/arXiv.2302.13971}.

\bibitem[Vaswani et~al.(2017)Vaswani, Shazeer, Parmar, Uszkoreit, Jones, Gomez,
  Kaiser, and Polosukhin]{vaswani2017attention}
Ashish Vaswani, Noam Shazeer, Niki Parmar, Jakob Uszkoreit, Llion Jones,
  Aidan~N. Gomez, Lukasz Kaiser, and Illia Polosukhin.
\newblock Attention is all you need.
\newblock In Isabelle Guyon, Ulrike von Luxburg, Samy Bengio, Hanna~M. Wallach,
  Rob Fergus, S.~V.~N. Vishwanathan, and Roman Garnett, editors, \emph{Advances
  in Neural Information Processing Systems 30: Annual Conference on Neural
  Information Processing Systems 2017, December 4-9, 2017, Long Beach, CA,
  {USA}}, pages 5998--6008, 2017.
\newblock URL
  \url{https://proceedings.neurips.cc/paper/2017/hash/3f5ee243547dee91fbd053c1c4a845aa-Abstract.html}.

\bibitem[Wang et~al.(2024)Wang, Ma, Dong, Huang, Zhang, and
  Wei]{wang2022deepnet}
Hongyu Wang, Shuming Ma, Li~Dong, Shaohan Huang, Dongdong Zhang, and Furu Wei.
\newblock {DeepNet}: Scaling transformers to 1,000 layers.
\newblock \emph{{IEEE} Trans. Pattern Anal. Mach. Intell.}, 46\penalty0
  (10):\penalty0 6761--6774, 2024.
\newblock \doi{10.1109/TPAMI.2024.3386927}.
\newblock URL \url{https://doi.org/10.1109/TPAMI.2024.3386927}.

\bibitem[Wen et~al.(2024)Wen, Li, Wang, Hall, Liang, and
  Ma]{wen2024understanding}
Kaiyue Wen, Zhiyuan Li, Jason~S. Wang, David Hall, Percy Liang, and Tengyu Ma.
\newblock Understanding warmup-stable-decay learning rates: {A} river valley
  loss landscape perspective.
\newblock \emph{CoRR}, abs/2410.05192, 2024.
\newblock \doi{10.48550/ARXIV.2410.05192}.
\newblock URL \url{https://doi.org/10.48550/arXiv.2410.05192}.

\bibitem[Xiong et~al.(2020)Xiong, Yang, He, Zheng, Zheng, Xing, Zhang, Lan,
  Wang, and Liu]{xiong2020layer}
Ruibin Xiong, Yunchang Yang, Di~He, Kai Zheng, Shuxin Zheng, Chen Xing,
  Huishuai Zhang, Yanyan Lan, Liwei Wang, and Tie{-}Yan Liu.
\newblock On layer normalization in the transformer architecture.
\newblock In \emph{Proceedings of the 37th International Conference on Machine
  Learning, {ICML} 2020, 13-18 July 2020, Virtual Event}, Proceedings of
  Machine Learning Research, pages 10524--10533. {PMLR}, 2020.
\newblock URL \url{http://proceedings.mlr.press/v119/xiong20b.html}.

\bibitem[Yang and Hu(2021)]{yang2021tensor}
Greg Yang and Edward~J. Hu.
\newblock Tensor programs {IV}: Feature learning in infinite-width neural
  networks.
\newblock In \emph{International Conference on Machine Learning}, volume 139,
  pages 11727--11737, 2021.

\bibitem[Yang and Littwin(2023)]{yang2023tensor}
Greg Yang and Etai Littwin.
\newblock Tensor programs {IVb}: Adaptive optimization in the infinite-width
  limit, 2023.
\newblock URL \url{https://arxiv.org/abs/2308.01814}.

\bibitem[Yang et~al.(2022)Yang, Hu, Babuschkin, Sidor, Liu, Farhi, Ryder,
  Pachocki, Chen, and Gao]{yang2022tensorprogramsV}
Greg Yang, Edward~J. Hu, Igor Babuschkin, Szymon Sidor, Xiaodong Liu, David
  Farhi, Nick Ryder, Jakub Pachocki, Weizhu Chen, and Jianfeng Gao.
\newblock Tensor programs v: Tuning large neural networks via zero-shot
  hyperparameter transfer.
\newblock 2022.
\newblock \doi{10.48550/ARXIV.2203.03466}.
\newblock URL \url{https://arxiv.org/abs/2203.03466}.

\bibitem[Yang et~al.(2024{\natexlab{a}})Yang, Yu, Zhu, and
  Hayou]{yang2024tensor}
Greg Yang, Dingli Yu, Chen Zhu, and Soufiane Hayou.
\newblock Tensor programs {VI}: Feature learning in infinite-depth neural
  networks.
\newblock In \emph{International Conference on Learning Representations},
  2024{\natexlab{a}}.

\bibitem[Yang et~al.(2024{\natexlab{b}})Yang, Lee, Nowak, and
  Papailiopoulos]{yang2023looped}
Liu Yang, Kangwook Lee, Robert~D. Nowak, and Dimitris Papailiopoulos.
\newblock Looped transformers are better at learning learning algorithms.
\newblock In \emph{The Twelfth International Conference on Learning
  Representations, {ICLR} 2024, Vienna, Austria, May 7-11, 2024}.
  OpenReview.net, 2024{\natexlab{b}}.
\newblock URL \url{https://openreview.net/forum?id=HHbRxoDTxE}.

\bibitem[Zhang and Sennrich(2019)]{zhang2019rmsnorm}
Biao Zhang and Rico Sennrich.
\newblock Root mean square layer normalization.
\newblock In Hanna~M. Wallach, Hugo Larochelle, Alina Beygelzimer, Florence
  d'Alch{\'{e}}{-}Buc, Emily~B. Fox, and Roman Garnett, editors, \emph{Advances
  in Neural Information Processing Systems 32: Annual Conference on Neural
  Information Processing Systems 2019, NeurIPS 2019, December 8-14, 2019,
  Vancouver, BC, Canada}, pages 12360--12371, 2019.
\newblock URL
  \url{https://proceedings.neurips.cc/paper/2019/hash/1e8a19426224ca89e83cef47f1e7f53b-Abstract.html}.

\bibitem[Zhang et~al.(2019)Zhang, Dauphin, and Ma]{zhang2019fixup}
Hongyi Zhang, Yann~N. Dauphin, and Tengyu Ma.
\newblock Fixup initialization: Residual learning without normalization.
\newblock In \emph{7th International Conference on Learning Representations,
  {ICLR} 2019, New Orleans, LA, USA, May 6-9, 2019}. OpenReview.net, 2019.
\newblock URL \url{https://openreview.net/forum?id=H1gsz30cKX}.

\bibitem[Zhu et~al.(2025)Zhu, Wang, Hua, Zhang, Li, Que, Wei, Wen, Yin, Xing,
  Li, Shi, Ma, Li, Kergan, Smith, Qu, Hui, Wu, Min, Huang, Zhou, Ye, Liu, Yang,
  Shi, Lin, Zhao, Cai, Zhang, Huang, Bengio, and Eshraghian]{zhu2025ouro}
Rui-Jie Zhu, Zixuan Wang, Kai Hua, Tianyu Zhang, Ziniu Li, Haoran Que, Boyi
  Wei, Zixin Wen, Fan Yin, He~Xing, Lu~Li, Jiajun Shi, Kaijing Ma, Shanda Li,
  Taylor Kergan, Andrew Smith, Xingwei Qu, Mude Hui, Bohong Wu, Qiyang Min,
  Hongzhi Huang, Xun Zhou, Wei Ye, Jiaheng Liu, Jian Yang, Yunfeng Shi,
  Chenghua Lin, Enduo Zhao, Tianle Cai, Ge~Zhang, Wenhao Huang, Yoshua Bengio,
  and Jason Eshraghian.
\newblock Scaling latent reasoning via looped language models.
\newblock 2025.
\newblock \doi{10.48550/ARXIV.2510.25741}.
\newblock URL \url{https://arxiv.org/abs/2510.25741}.

\end{thebibliography}
\end{document}